\pdfoutput=1

\documentclass{article}

\usepackage[preprint]{tmlr}

\usepackage{amsmath,amssymb,amsthm,mathtools,bm}
\usepackage{microtype}
\usepackage{graphicx}
\usepackage{booktabs,longtable,tabularx,array,threeparttable}
\usepackage{etoolbox,pdflscape}
\usepackage{enumitem}
\usepackage{caption}
\usepackage{url}

\usepackage[hidelinks]{hyperref}
\usepackage[nameinlink,noabbrev]{cleveref}
\crefname{figure}{Figure}{Figures}
\Crefname{figure}{Figure}{Figures}

\newcommand{\tablefont}{\fontsize{10}{12}\selectfont}
\AtBeginEnvironment{table}{\tablefont}
\AtBeginEnvironment{longtable}{\tablefont}

\newtheorem{theorem}{Theorem}[section]

\newtheorem{proposition}[theorem]{Proposition}
\newtheorem{corollary}[theorem]{Corollary}
\newtheorem{assumption}[theorem]{Assumption}

\theoremstyle{definition}
\newtheorem{definition}[theorem]{Definition}
\newtheorem{algorithmdef}[theorem]{Algorithm}

\theoremstyle{remark}

\newcommand{\R}{\mathbb R}
\newcommand{\E}{\mathbb E}
\newcommand{\Prob}{\mathbb P}
\newcommand{\Hh}{\mathcal H}
\newcommand{\Ff}{\mathcal F}
\newcommand{\Aa}{\mathcal A}
\newcommand{\Cc}{\mathcal C}
\newcommand{\Kk}{\mathcal K}
\newcommand{\norm}[1]{\left\lVert #1\right\rVert}

\newcommand{\argmin}{\operatorname*{arg\,min}}

\newcommand{\Null}{\operatorname{null}}
\newcommand{\Range}{\operatorname{range}}

\newcommand{\tr}{\operatorname{tr}}

\newcommand{\kap}{\kappa}
\newcommand{\rhoP}{\rho_{\mathrm{persistent}}}
\newcommand{\DeltaP}{\Delta_{\mathrm{persistent}}}
\newcommand{\DeltaZero}{\Delta_{0}}
\newcommand{\lhat}{\widehat{\lambda}}

\newcommand{\DeltaPers}{\Delta_{\mathrm{persistent}}}
\newcommand{\rhoPers}{\rho}
\newcommand{\lambdahat}{\widehat{\lambda}}
\newcommand{\CI}{\mathrm{CI}_{95\%}}

\newcommand{\kappaF}{\kappa}

\title{Functional compatibility as a determinant of persistent neural learning}

\author{
\name Hossein Javidnia
\email hossein.javidnia@dcu.ie \\
\addr School of Computing\\
Dublin City University\\
Dublin, Ireland
}

\begin{document}

\maketitle

\begin{abstract}
\noindent Neural networks can acquire new capabilities while damaging existing ones, but what determines whether new learning persists remains unclear. We identify functional compatibility, the extent to which incoming learning can coexist with behavior that must be preserved, as an experimentally manipulable causal determinant of persistence. From identical neural states, we vary compatibility while matching unrestricted learning opportunity and imposing a common retention requirement. Persistent learning increases with compatibility across separately constructed directions, convolutional and transformer architectures, vision and text, and a ten-seed extension. Learning rules and retention constraints determine how much compatible opportunity is retained, whereas nonlinear geometry limits the matched intervention at larger update norms. Adaptive Functional Metaplasticity couples a retention-constrained parameter update to finite prediction restoration. Its analysis separates persistent loss reduction from exact interpolation on protected evidence, connecting local compatibility to an executable continual learner.
\end{abstract}

\section{Introduction}
Continual learning requires a neural network to acquire information while preserving behavior that remains useful. Regularization, replay and constrained-gradient methods address this tension by penalizing parameter changes, revisiting observations or restricting updates to selected directions \citep{kirkpatrick2017ewc,rolnick2019replay,lopezpaz2017gem,farajtabar2020ogd,saha2021gpm}. Function-space methods protect predictions directly \citep{benjamin2019functionspace}, while plasticity-preservation methods address the progressive loss of a network's ability to learn \citep{dohare2024plasticity}. These approaches motivate a complementary question: at the same neural state, do different incoming learning signals offer different opportunities for acquiring information without damaging protected function?

We call the relevant geometric property \emph{functional compatibility}: the fraction of an incoming gradient's available energy retained after projection onto directions compatible with the declared protection geometry. Compatibility is measured before the update. It therefore distinguishes learning opportunities at a fixed state rather than retrospectively measuring forgetting. Persistent learning means the loss decrease stored in the ordinary trainable parameters; a separate finite correction to deployed predictions is accounted for explicitly.

We test compatibility by constructing several current learning signals from an identical neural state, matching the finite loss decrease available without protection, and imposing the same retention allowance. This design separates compatibility from unrestricted learning opportunity and from the strength of protection. Across convolutional and transformer architectures, vision and text, independently constructed directions and a ten-seed extension, increasing compatibility increases the learning retained by compatible updates. The response depends on the learning rule and retention allowance. Experiments at larger update norms identify the finite nonlinear boundary of the matched intervention.

Adaptive Functional Metaplasticity (AFM) connects this intervention to a constructive protected learner. It separates a retention-constrained parameter update from a bounded residual that matches the unrestricted predictions on the current minibatch and restores the outputs on protected observations. Under the scalar-gradient, smoothness and projection conditions formalized below, the persistent fraction of the unrestricted loss decrease is at least one third of compatibility times the accepted fraction of the projected update. The guarantee identifies a local compatibility cost; the finite residual does not count as persistent assimilation.

Gradient alignment, protected subspaces and interference geometry have been studied in continual learning \citep{riemer2019mer,farajtabar2020ogd,wang2021adamnscl,benjamin2019functionspace,stork2026interference,shu2026zeroth,gong2026drift,kitkana2026subliminal,abbes2026gradient}. Our contribution is a matched-state, multi-level intervention on a quantified compatibility variable, with unrestricted opportunity and retention allowance controlled separately, together with a framework that distinguishes persistent adaptation from finite endpoint restoration. To our knowledge, this controlled causal design has not previously been established. Chronological benchmarks then test whether the resulting AFM transaction is useful under natural stream evolution.

\section{AFM framework and theoretical foundations}

\subsection{Problem formulation and design principle}
\label{sec:problem-setting}

Let $\mathcal X$ and $\mathcal Y$ be the input and label spaces. The learner observes a nonanticipating supervised stream in $\mathcal X\times\mathcal Y$, without task, session, episode, bucket, segment, intervention or boundary identifiers. A representation $z:\mathcal X\to\R^{d_Z}$ supplies observable addresses and is frozen before protected decisions begin. The trainable base map $G_\theta:\R^{d_Z}\to\R^m$ has parameters $\theta\in\R^d$, where $m$ is the output-logit dimension. At update round $t$, the residual function (or shield) $S_t:\R^{d_Z}\to\R^m$ is represented by finitely many bounded coefficients and compactly supported basis functions, with capacities fixed before the stream. Structural state comprises these coefficients and basis functions. The deployed logit predictor $F_t:\mathcal X\to\R^m$ is
\begin{equation}
F_t(x)=G_{\theta_t}(z(x))+S_t(z(x)),
\label{eq:deployed-predictor}
\end{equation}
The parameter and structural states have distinct roles: $\theta_t$ stores persistent adaptation, while $S_t$ imposes finite output constraints. The pathwise results permit an arbitrary loss sequence; each statistical result states its additional sampling or conditional-moment assumptions. Fully trainable causal probes use a separate fixed input-address encoder, as specified in Section~\ref{sec:causal-proposals}.

\emph{Persistent adaptation} is stored in ordinary parameters and can affect predictions beyond the finite intervention set. \emph{Finite deployed protection} constrains outputs on declared observations. \emph{Population retention} concerns unseen inputs and requires additional assumptions or certificates; it does not follow from exact finite protection.

\subsubsection{Same-state no-protection comparator}
\label{subsec:counterfactual-comparator}

For the current ordered minibatch $\mathcal B_t$ and its observed labels, let $\mathcal L_t:\R^{|\mathcal B_t|\times m}\to\R$ be the declared loss of its logit array; $\mathcal L_t(F)$ abbreviates evaluation on $(F(x):x\in\mathcal B_t)$. Holding the residual fixed defines
\[
\ell_t^S(\theta)=\mathcal L_t\bigl(x\mapsto G_\theta(z(x))+S(z(x))\bigr).
\]
The persistent decrease evaluates $\ell_t^{S_t}$ at the accepted base endpoint; the deployed decrease evaluates the complete new predictor, including its new residual. At a protected update, AFM first evaluates the update that the corresponding no-protection learner would actually accept from the same complete pre-step state. The comparator uses the same parameters, structural state, current minibatch, active coordinates, learning-rate and backtracking rules, mutable buffers, and declared private randomness, but no retention constraint. If the resulting endpoint is $\theta_t^0$ with deployed current-minibatch logits $U_t$, its accepted prediction-space loss decrease is
\begin{equation}
\Delta_t^0=\mathcal L_t(F_t)-\mathcal L_t(U_t)>0.
\label{eq:counterfactual-decrease-main}
\end{equation}
The unrestricted endpoint is a reference, never the protected transaction's persistent base.

The comparator measures the persistent plasticity cost of protection relative to the update the same learner would actually accept.

\subsubsection{Persistent compatible assimilation}
\label{subsec:persistent-assimilation-main}

Let $g_t=\nabla\ell_t^{S_t}(\theta_t)$ be the full parameter gradient. The orthogonal availability projector $A_t:\R^d\to\R^d$ selects trainable coordinates. The orthogonal projector $Q_t$ selects protected directions within $\Range(A_t)$, and $\Pi_t=A_t-Q_t$ selects their available orthogonal complement; hence $\Pi_tA_t=\Pi_t$ and $\Range(\Pi_t)\subseteq\Range(A_t)$. Compatibility is relative to this selected geometry and does not assert zero leakage outside its certified bound. In the scalar-gradient specialization,
\begin{equation}
d_t^0=-\alpha_tA_tg_t,
\qquad
v_t=\Pi_td_t^0=-\alpha_t\Pi_tg_t,
\label{eq:projected-comparator-main}
\end{equation}
where $\alpha_t>0$ is the accepted comparator step size. Let $\Aa_t$ index the active protected records, each fixing evidence and a behavior constructor $\mathcal B_j$ from logit functions on $\R^{d_Z}$ to a Hilbert space $\Hh_j$. Write $\Phi_j^{[S]}(\theta)=\mathcal B_j(G_\theta+S)$ and $\Phi_t^{[S]}=\bigoplus_{j\in\Aa_t}\Phi_j^{[S]}$, with the direct-sum Hilbert norm. All derivatives hold $S$ fixed. Norms are Euclidean for parameter vectors and Hilbert norms for behaviors. Let $\varepsilon_t$ bound empirical derivative leakage on unit compatible directions and $e_t^{\rm pop}$ bound any additional population discrepancy; empirical protection sets $e_t^{\rm pop}=0$. Write $E_t=\varepsilon_t+e_t^{\rm pop}$, and let $\overline H_t\ge0$ bound the behavior Taylor remainder on the declared trust region. Their certificates are constructed in Section~\ref{app:whole-behavior-memory}. Write $s_t^0=\|v_t\|$. For step length $s\ge0$, the retention charge is
\begin{equation}
C_t(s)=E_ts+\frac{\overline H_t}{2}s^2,
\label{eq:charge-main}
\end{equation}
A predeclared assimilation coordinate $\eta_t\in[0,1]$ allocates
\begin{equation}
b_t^{\rm norm}=\eta_t C_t(s_t^0),
\qquad
\lambda_t=\max\{\lambda\in[0,1]:C_t(\lambda s_t^0)\le b_t^{\rm norm}\}.
\label{eq:eta-main}
\end{equation}
When $\|A_tg_t\|>0$, the compatible-gradient fraction is
\begin{equation}
\kappa_t=\frac{\|\Pi_tg_t\|^2}{\|A_tg_t\|^2}\in[0,1].
\label{eq:kappa-main}
\end{equation}
Thus $\eta_t$ allocates a fraction of the projected comparator charge and $\kappa_t$ measures the retained fraction of available gradient energy. For an accepted fraction $\widehat\lambda_t\in[\eta_t,\lambda_t]$, put $\theta_t^{\rm base}=\theta_t+\widehat\lambda_t v_t$ and $\Delta_{\mathrm{persistent}}=\Delta_t^{\rm base}=\ell_t^{S_t}(\theta_t)-\ell_t^{S_t}(\theta_t^{\rm base})$. Theorem~\ref{thm:normalized-assimilation} gives
\begin{equation}\label{eq:main-bound}
\frac{\Delta_{\mathrm{persistent}}}{\Delta_t^0}
\ge\frac{\widehat\lambda_t\kappa_t}{3}\ge\frac{\eta_t\kappa_t}{3}.
\end{equation}
This ratio is denoted $\rho_{\mathrm{persistent}}$, or $\rho$ in the experiments. Degenerate directions and failed comparator or certificate checks produce the explicit obstructions defined below.

\subsubsection{Finite endpoint completion}
\label{subsec:finite-completion-main}

A persistent compatible step generally cannot reproduce the full current endpoint of the unrestricted learner. AFM therefore performs a second, explicitly separate operation in function space. After the safe persistent base move, localized basis functions interpolate prescribed logits at finitely many frozen feature locations. Each basis function is one at its own location and zero at the other constrained locations; their finite union of compact supports bounds where the residual can act. Label-free guards outside the replay envelopes of constrained nodes additionally require zero residual. On the current minibatch it requests the comparator logits $U_t$; on active protected evidence and unselected certified candidates it requests their pre-step deployed outputs; on a selected candidate it requests the frozen transfer target or the declared contraction toward it. If these finite requirements are mutually consistent and the support, capacity, and numerical checks pass, the residual satisfies them exactly.

The residual is not treated as persistent assimilation. Outside its finite union of compact supports, the deployed predictor follows the protected base endpoint. If finite requests conflict at the same observable address, compatible motion is unavailable, capacity is exhausted, or a required certificate fails, the transaction is rejected atomically. The complete specification, finite-support construction, and rollback conditions are given in Section~\ref{app:full-specification}.

\subsection{Adaptive Functional Metaplasticity}
\label{sec:method}

A bounded task-free record lifecycle connects sensitivity memory, adaptive protection, consolidation, reopening and structural renewal to the protected transaction (Section~\ref{app:full-specification}).

\subsubsection{Whole-behavior sensitivity memory}
\label{subsec:whole-behavior-main}

A committed record protects a behavior map rather than a parameter vector. For $n_j\ge1$ fixed evidence inputs $x_{j,1},\ldots,x_{j,n_j}$, record $j$ uses
\begin{equation}
\widehat\Phi_j^{[S]}(\theta)=n_j^{-1/2}
\bigl(\psi_j(F_\theta^{[S]}(x_{j,1})),\ldots,\psi_j(F_\theta^{[S]}(x_{j,n_j}))\bigr),
\label{eq:empirical-behavior-main}
\end{equation}
Here $F_\theta^{[S]}(x)=G_\theta(z(x))+S(z(x))$ and $\psi_j:\R^m\to\R^{m_j}$ is the declared differentiable output transform, such as logits or probabilities. All Jacobian rows at the frozen parameter and shield anchor are streamed into a Frequent Directions sketch \citep{ghashami2016}. The sketches define a bounded approximation to the protected sensitivity covariance. The selected rank-$r$ protected subspace suppresses the most consequential directions, while the unprotected complement remains available for adaptation. The exact sketch certificate, its population extension under declared sampling assumptions, and the rank-$r$ min-max frontier are given in Section~\ref{app:whole-behavior-memory}.

\subsubsection{Metaplastic allocation}
\label{subsec:metaplastic-main}

AFM does not assign one permanent importance coefficient to every protected record. It maintains a finite bank of exponentially spaced traces and uses them to predict which protected sensitivities remain relevant at the current time. This construction is motivated by the ability of multiple timescales to represent long temporal ranges efficiently \citep{benna2016metaplasticity}. In AFM the traces have a specific optimization role: a finite policy family chooses both temporal weights and a spectral rank, and the policy loss charges residual protected leakage together with current gradient energy blocked by protection. Section~\ref{app:metaplastic-allocation} proves that a logarithmic trace bank approximates scale-free relevance profiles with constant distortion over an exponentially long horizon, whereas a single exponential has polynomial worst-case distortion; it also gives the allocation-transfer and online policy-regret bounds.

\subsubsection{Task-free consolidation, reopening, and route refinement}
\label{subsec:routing-main}

Candidate consolidation is separated from both fitting and deployment. A private candidate is fitted on a finite routed block and then frozen. Later outcomes form a distinct validation sequence, so the candidate prediction precedes each validation outcome. An anytime-valid upper confidence sequence controls the first certification crossing; subsequent evidence is handled by a separate staleness process. AFM uses time-uniform inference because ordinary fixed-time tests are not generally valid under repeated optional checking \citep{ramdas2023anytime,waudbysmith2024betting}.

Routing uses only observable signatures and never evaluator task identities. Outcome evidence can justify reopening or release, but a new observable route additionally requires separately controlled signature evidence. Sequential two-sample procedures provide one admissible route-refinement construction \citep{jang2022}. Indistinguishable semantic states produce a routing obstruction. Full calibration, consolidation, staleness, reopening, and route-refinement statements are provided in Section~\ref{app:routing-consolidation}.

\subsubsection{Function-preserving structural renewal}
\label{subsec:renewal-main}

The model contains a fixed pool of dormant zero-gated modules. Resetting internal parameters of a module whose functional gate is exactly zero leaves the deployed predictor and active protected behaviors unchanged. After a reset, AFM recomputes the complete gradient, protected projector, and leakage certificates; a slot is activated only if an accepted protected update produces nonzero motion in that slot. Renewal therefore expands usable structure without granting an exception to the retention transaction. Section~\ref{app:structural-renewal} gives the exact reset result and shows that renewal success is controlled by the realized compatible richness of the trial distribution.

Figure~\ref{fig:afm-transaction} summarizes how these components
interact within one protected AFM transaction.

\begin{figure}[!htbp]
\centering

\includegraphics[width=0.96\textwidth]{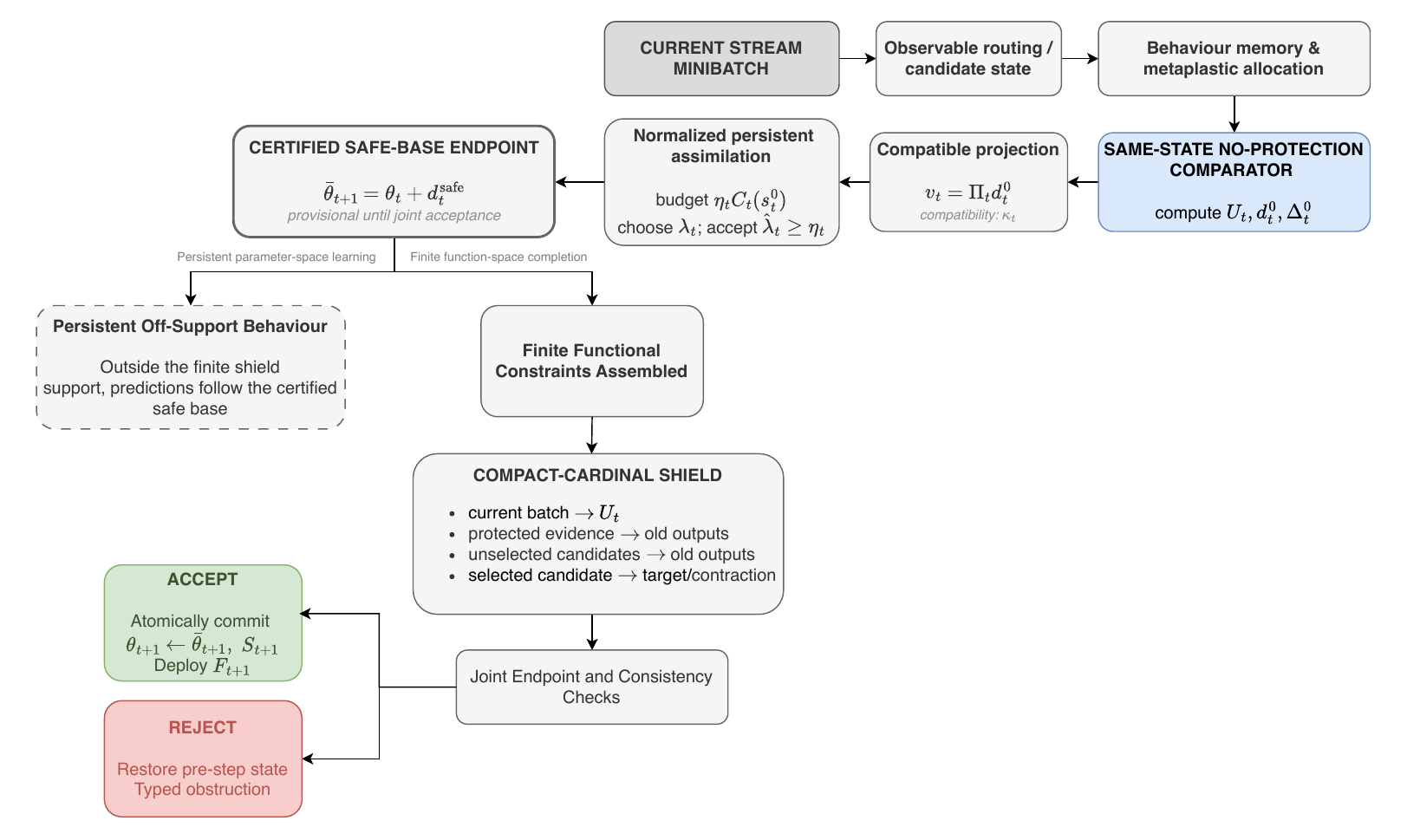}

\caption{\textbf{Full AFM protected-update transaction.}
The transaction combines same-state comparator construction,
compatibility projection, normalized persistent assimilation, finite functional constraints,
compact-cardinal shield, endpoint checks and accept/reject logic.}
\label{fig:afm-transaction}

\end{figure}

\subsubsection{Protected update}
\label{subsec:update-main}

A protected round can be summarized as follows:
\begin{enumerate}[label=\arabic*.]
\item route the current observations using only declared observable signatures;
\item update candidate fitting, certification, and staleness state without copying a private candidate into deployment;
\item prepare any zero-gated renewal trial triggered by prior history, then select protection using bounded sketches and metaplastic traces from previous rounds;
\item compute the genuine same-state no-protection endpoint and compatible projected reference, then update the current importance signals, policy losses and traces;
\item accept only a persistent base point that satisfies the retention, descent, and normalized-assimilation checks;
\item construct the finite residual that completes the current endpoint and restores the declared finite protected outputs;
\item commit or roll back provisional renewal with the transaction, and update reopening and route refinement under their evidence and capacity rules.
\end{enumerate}
No failed certificate authorizes installation of the unrestricted base endpoint or reduction of the declared persistent-assimilation requirement.

\subsection{Theoretical guarantees}
\label{sec:theory}

The central guarantee bounds persistent loss decrease and establishes finite deployed output equality separately. Here
$d_t^0=-\alpha_tA_tg_t$ is the accepted same-state no-protection
displacement, $v_t=\Pi_t d_t^0$ is its compatible projection, and
\[
\kappa_t=\frac{\|\Pi_tg_t\|^2}{\|A_tg_t\|^2}
\]
is the compatible-gradient energy fraction. Let $\overline L_t\ge0$ bound the absolute loss Taylor remainder by $\overline L_t\|d\|^2/2$ on both reference segments, as in~\eqref{eq:directional-smooth}.

\begin{theorem}[Counterfactual-normalized persistent assimilation]
\label{thm:normalized-assimilation}
Assume the accepted no-protection comparator has the scalar-gradient form $d_t^0=-\alpha_tA_tg_t$ with $\alpha_t>0$ and $0\le \alpha_t\overline L_t\le1$, $A_t$ and $\Pi_t$ are orthogonal projectors satisfying $\Pi_tA_t=\Pi_t$, the comparator decrease $\Delta_t^0$ is positive, and the projected reference $v_t=\Pi_td_t^0$ is nonzero. Assume the two-sided loss Taylor-remainder certificate~\eqref{eq:directional-smooth} holds on the complete comparator and projected-reference line segments, and that the stated retention-charge certificate holds on the complete projected-reference segment. Then the normalized construction in Equation~\eqref{eq:eta-main} satisfies
\begin{equation}
\lambda_t\ge\eta_t.
\label{eq:lambda-main}
\end{equation}
For any accepted realized fraction $\widehat\lambda_t\in[\eta_t,\lambda_t]\subseteq[0,1]$,
\begin{equation}
\frac{\Delta_t^{\rm base}}{\Delta_t^0}
\ge
\widehat\lambda_t\kappa_t
\frac{1-\tfrac12\alpha_t\overline L_t\widehat\lambda_t}
     {1+\tfrac12\alpha_t\overline L_t}
\ge \frac{\widehat\lambda_t\kappa_t}{3}
\ge \frac{\eta_t\kappa_t}{3}.
\label{eq:persistent-ratio-main}
\end{equation}
For $\eta_t=1$, the complete projected comparator is feasible.
\end{theorem}

The proof, including the exact role of convexity of the retention charge and the lower side of the two-sided Taylor-remainder certificate used for the denominator, is given in Section~\ref{app:full-specification}. The guarantee is not an $\eta_t$ fraction of unrestricted loss decrease in general: the compatibility factor $\kappa_t$ is unavoidable in the selected protected geometry.

\begin{theorem}[Exact finite counterfactual completion]
\label{thm:priority-oracle}
In exact arithmetic, suppose the no-protection comparator and certified persistent base endpoint are both available, the finite requested logits are consistent at duplicate frozen addresses, the declared node and guard capacities are sufficient, support separation is certified, and the numerical endpoint checks pass. Then the joint base-plus-residual transaction satisfies: (i) the persistent base remains within its certified behavior budget and achieves protected descent; (ii) the deployed predictor equals the no-protection endpoint on every current-minibatch observation; (iii) every active finite protected output is unchanged; (iv) every unselected certified candidate is unchanged and the selected candidate reaches its declared finite target or contraction; and (v) the residual is exactly zero outside its finite support union and on nonmatching guard addresses, namely guards outside every constraint node's declared replay envelope.
\end{theorem}

Because the current deployed loss factors through the current-minibatch logits, Theorem~\ref{thm:priority-oracle} gives
\begin{equation}
\mathcal L_t(F_t)-\mathcal L_t(F_{t+1})=\Delta_t^0,
\qquad
\frac{\mathcal L_t(F_t)-\mathcal L_t(F_{t+1})}{\Delta_t^0}=1.
\label{eq:deployed-ratio-main}
\end{equation}
The complete finite construction, support bound, transfer result, and finite-precision realization are developed in Section~\ref{app:full-specification}.

\subsubsection{Retention, stationarity, and local optimality}

Combining the persistent and finite legs gives a one-step guarantee. Let $b_t\ge0$ be the permitted fixed-shield behavior change and $\bar\theta_{t+1}=\theta_t+d_t^{\rm safe}$ the accepted persistent endpoint, with $s_t^{\rm safe}=\|d_t^{\rm safe}\|$. Then
\begin{align}
\|\Phi_t^{[S_t]}(\theta_t+d_t^{\rm safe})-\Phi_t^{[S_t]}(\theta_t)\|
&\le (\varepsilon_t+e_t^{\rm pop})s_t^{\rm safe}
+\frac{\overline H_t}{2}(s_t^{\rm safe})^2\le b_t,
\label{eq:retention-main}\\
\ell_t^{S_t}(\theta_t)-\ell_t^{S_t}(\theta_t+d_t^{\rm safe})
&\ge \frac12s_t^{\rm safe}\|\Pi_tg_t\|.
\label{eq:descent-main}
\end{align}
On the declared finite protected evidence, the stronger deployed statement is exact equality across accepted protected updates. On broader behavior maps, the theorem retains explicit structural-shield and routing-mismatch charges rather than asserting population preservation from finite evidence. Section~\ref{app:retention-adaptation} gives the active-interval bounds, release decomposition, arbitrary-stream projected-stationarity bound, and lifelong compatible-plasticity corollary.

The local protection geometry is also sharp at first order. For $0\le r<d$ and a behavior Jacobian $J$ with singular values $\sigma_1\ge\cdots\ge\sigma_d$, the smallest worst-case first-order leakage achievable by a $(d-r)$-dimensional plastic subspace is $\sigma_{r+1}(J)$. Thus finite-rank protection has an explicit spectral price; the result and the streamed-memory certificate are proved in Section~\ref{app:whole-behavior-memory}.

\subsubsection{Integrated frontier and unavoidable obstructions}

The full AFM theorem combines pathwise retention, endpoint emulation, compatible stationarity, task-free routing, operational consolidation, metaplastic allocation, reopening, structural renewal, bounded structural resources, and a convex dynamic-regret specialization. Its deterministic conclusions hold on every realized trajectory whenever the corresponding certificates accept. Statistical conclusions, such as candidate-risk, route-refinement, or renewal-probability statements, additionally require only the assumptions stated for those conclusions and share a summable failure budget. The complete theorem and dependency proof are given in Section~\ref{app:integrated-theorem}.

Each guarantee has explicit obstructions when observable information, compatible motion or capacity is unavailable. Exact semantic routing cannot be guaranteed when observable laws are identical; future risk cannot be certified distribution-free from an arbitrary finite past; finite-rank protection cannot uniformly eliminate leakage in every plastic direction; direct output contradictions cannot be simultaneously satisfied; reopening requires observational information; renewal requires compatible richness; and fixed finite memory cannot encode an unlimited number of independent facts. Section~\ref{app:obstructions} states and proves these lower bounds. These lower bounds define the scope of the frontier.

The nonlinear and convex optimization statements have different scope.
For smooth nonlinear models, the corresponding global-in-time optimization
statement is the pathwise projected-stationarity bound of
Theorem~\ref{thm:stationarity}, under its stated certificate conditions.
The dynamic-regret guarantee in Section~\ref{app:integrated-theorem}
requires the separate convex affine specialization and is not inherited by
the neural experiments below. The experiments instead evaluate the nonlinear
AFM transaction and its theorem-aligned quantities empirically.

\subsection{From local geometry to a continual learner}
\label{sec:conclusion}
The local assimilation and interpolation guarantees require a mechanism for maintaining protected evidence as the stream evolves. AFM assigns each record a frozen behavior constructor, anchor and bounded sensitivity sketch; allocates compatible directions before each update; and admits, transfers or releases records using separately controlled evidence. The resulting lifecycle determines when each local certificate applies and how its costs compose over time. Componentwise maximality refers to the matching obstruction for each stated information, geometric or resource condition, as formalized in Corollary~\ref{cor:maximality}; it does not assert optimality over every continual-learning algorithm.

\subsection{Full AFM specification and finite endpoint construction}
\label{app:full-specification}

Composition requires one predictive state, protected evidence set and numerical certificate system for each transaction. The admissible specification fixes them before learning.

\subsubsection{Admissible specification}
\begin{definition}[Admissible AFM specification]\label{def:spec}
Before the protected horizon begins, an AFM specification fixes the rules, capacities and admissible information for every protected decision. These comprise:
\begin{enumerate}[label=(\roman*)]
\item a finite family of observable context-signature maps, their mixture weights, calibration prefix, requested coverage, calibration ceiling, routing thresholds, registry size, and deterministic tie rules;
\item a bounded outcome loss, a deterministic same-state no-protection update operator, the normalized coordinate $\eta_t$, the protected safe-base operator, the activation-gap threshold, and all trust, smoothness, curvature, and numerical-certification procedures;
\item a finite candidate-fitting procedure, validation horizon, first-crossing certification rule, post-certification staleness test, candidate-service priority, and service horizon;
\item the frozen-representation address map, finite shield support rule, duplicate-consistency convention, current-minibatch bound, shield-node capacity, guard capacity, selected-transfer fraction, and atomic rollback rule;
\item the behavior-map constructors, protected-record and sketch capacities, finite rank-policy family, metaplastic trace bank, and all associated thresholds;
\item the shadow challenger, outcome-reopening and observable-signature route-refinement procedures, minimum effect and distinct-block requirements, diagnostic grace window, and capacity policy; and
\item a fixed pool of functionally zero-gated renewal modules, the renewal trial distribution, the certified numerical precision or finite-horizon precision policy, and all remaining resource budgets.
\end{enumerate}
An optional finite unprotected initialization prefix may produce the initial representation but creates no protected record. After that prefix, the representation is frozen and any bounded prefix references used for signature calibration are replayed through the final representation before protected decisions begin. Every component must be measurable with respect to the declared learner history. Failure of a component to satisfy its own conditions yields the corresponding calibration, routing, statistical, functional, capacity, numerical, retention, compatibility, or renewal obstruction; the declared algorithm is unchanged.
\end{definition}

\subsubsection{Losses and retained behaviors}

Global time $t$ indexes parameter updates; local count $n$ indexes resolved outcomes for a given record $j$, separately from its fixed evidence size $n_j$. Write $t_j(n)$ for the global round at which its $n$th outcome resolves, and use $\Ff_{j,n-1}$ for the history available before that outcome. A candidate's local first-crossing count maps to global certification time $\tau_j=t_j(\tau_j^{\rm cert})$. Record subscripts on filtrations are suppressed when only one record is under discussion. Within a round, predictions precede labels and each policy or renewal draw follows the declared nonanticipating information order. At round $t$, the learner has parameters $\theta_t\in\R^d$, observes an arbitrary data item, and incurs a differentiable loss $\ell_t:\R^d\to\R$ with gradient $g_t=\nabla\ell_t(\theta_t)$. For protected updates, $\ell_t=\ell_t^{S_t}$ holds the current shield fixed. A data item may be a declared finite ordered minibatch. In that case every member is routed in the fixed within-batch order, every consolidation and outcome-reopening statistic is updated at outcome resolution, every route-refinement statistic is updated once per distinct observable signature block, and $\ell_t$ is the declared differentiable minibatch loss used for the single safe parameter update. Each fixed signature map may assign one image-only signature to every member of a predeclared observable microblock, for example a normalized mean of frozen features, because the whole microblock is observed before its outcomes; labels, task identities, and evaluator metadata may not enter that signature. Repeating one block signature across several labeled members counts as one signature observation and several outcome observations. A semantic boundary inside a block contributes to the routing-mismatch terms. This is not an iid assumption: the ordered minibatch itself is one arbitrary stream element, and its size is bounded by the fixed specification. No stochastic or stationarity assumption is imposed on the loss stream. For randomized allocation, let $\mathcal P$ be the fixed policy family of size $P$. The current data and loss function are fixed before the controller draws $p_t\in\mathcal P$; they may otherwise depend arbitrarily on the past. Thus the allocation losses form a nonanticipating adaptive sequence rather than an adversary reacting to the current private coin.

The active protected records are indexed by $j\in\Aa_t$. Their constructors and evidence are fixed. In one-round formulas, $\Phi_j$ abbreviates the Fr\'echet-differentiable fixed-current-shield map $\Phi_j^{[S_t]}$
\[
\Phi_j:\R^d\to\Hh_j,
\]
where $\Hh_j$ is a Hilbert space. It may represent a predictor in $L^2(P)$, an output-derived policy or value function, or empirical behavior on a frozen routed segment. Derivative bounds also apply to auxiliary parameter maps, including internal representations, when their own certificates are supplied; finite logit restoration alone does not preserve such maps exactly. The empirical instantiation uses $\widehat\Phi_j^{[S_t]}$ on the fixed commit block. A population map requires a separate certificate. Because a nonlinear output transform can change its parameter derivative when the shield changes, the certificate below explicitly transports the stored anchor derivative to the current shield. The stacked active map is
\[
\Phi_t=\bigoplus_{j\in\Aa_t}\Phi_j,\qquad
J_t=D\Phi_t(\theta_t).
\]

Before any nonzero update, AFM must possess a trust radius $R_t^{\rm cap}\ge0$ and finite certificates $\overline L_t,\overline H_t$. On every protected displacement $d\in\Range(\Pi_t)$ with $\norm d\le R_t^{\rm cap}$, the current loss must satisfy the two-sided Taylor-remainder certificate
\begin{equation}
\left|\ell_t(\theta_t+d)-\ell_t(\theta_t)-g_t^Td\right|
\le \frac{\overline L_t}{2}\norm d^2,
\label{eq:directional-smooth}
\end{equation}
and the protected behavior must satisfy
\begin{equation}
\norm{\Phi_t(\theta_t+d)-\Phi_t(\theta_t)-J_td}
\le \frac{\overline H_t}{2}\norm d^2.
\label{eq:behavior-curvature}
\end{equation}
In counterfactual-normalized scalar-gradient mode, the same loss remainder certificate is additionally required on the complete no-protection comparator segment $\{s d_t^0:0\le s\le1\}$, which need not lie in $\Range(\Pi_t)$, and on the complete projected-reference segment $\{s v_t:0\le s\le1\}$. These are direct quadratic remainder/upper-model certificates; they do not require a globally existing Hessian. If no valid certificate is available, the algorithm sets $R_t^{\rm cap}=0$ and makes no update. Hence the proof never defines smoothness or curvature after choosing the step. Losses are bounded below by $\ell_{\inf}$ on the realized iterates.

\begin{proposition}[Constructible trust certificates for finite computation graphs]\label{prop:cert-construct}
Suppose the current loss and every active empirical behavior map are represented by finite computation graphs whose primitive operations have computable interval enclosures for their first and second derivatives on a compact parameter box $\mathcal K_t^{\rm par}$. Then interval automatic differentiation and the chain, product, and composition rules produce finite certified values $\overline L_t$, $\overline H_t$, and a parameter-Jacobian Lipschitz bound $L_{J,j,t}$ on every Euclidean ball contained in $\mathcal K_t^{\rm par}$. For piecewise-affine primitives, any ball certified not to cross a switching surface has zero second-order remainder for that primitive. If interval propagation is unbounded or inconclusive, choosing $R_t^{\rm cap}=0$ remains valid.
\end{proposition}

\begin{proof}
Proceed in topological order through the finite graph. Interval arithmetic encloses every intermediate value. The derivative of each primitive is enclosed by hypothesis; applying the chain and product rules recursively encloses the graph Jacobian and all directional second derivatives on the box. Taking operator-norm upper bounds gives the displayed certificates. A piecewise-affine primitive is affine on a switch-stable box, so its second derivative there is zero.
\end{proof}

\subsubsection{Task-free routing and semantic mismatch}

The algorithm receives a context signature $Z_t$ but no task label. It uses the predictable bounded-registry router defined in Algorithm~\ref{alg:afm}. The internal theorem protects the behaviors assigned by this router. To compare with any external semantic partition, embed the internal and evaluator behaviors in a common Hilbert space. For the joint base-shield state, let $\widetilde\Psi_t(\theta,S)$ be the evaluator's desired behavior map, let $\widetilde\Phi_t(\theta,S)$ be the router-selected map, and define the exact mismatch
\[
\widetilde\Delta_t(\theta,S)=\widetilde\Psi_t(\theta,S)-\widetilde\Phi_t(\theta,S).
\]
For the safe-base leg with the pre-step shield fixed, suppose the mismatch is twice continuously differentiable on a neighborhood of the update segment, and define
\begin{equation}\label{eq:routing-terms}
r_t=\norm{D_\theta\widetilde\Delta_t(\theta_t,S_t)}_{\rm op},\qquad
k_t=\sup_{0\le s\le1}
\norm{D_\theta^2\widetilde\Delta_t(\theta_t+s d_t^{\rm safe},S_t)}_{\rm op}.
\end{equation}
For the structural replacement define the separate exact mismatch charge
\begin{equation}\label{eq:routing-structural-charge}
\omega_t^{\Delta}:=
\norm{\widetilde\Delta_t(\bar\theta_{t+1},S_{t+1})-
\widetilde\Delta_t(\bar\theta_{t+1},S_t)}.
\end{equation}
These charges vanish when the two behavior maps coincide throughout the relevant states. A finite external bound requires finite $r_t$ and $k_t$; the internal retention result does not require an external mismatch certificate.

\subsubsection{Streamed whole-behavior sensitivity memory}

A protected segment $j$ has joint anchor $(\bar\theta_j,S_j^{\rm anc})$ and a frozen commit evidence block
\[
\mathcal D_j=(x_{j,1},\ldots,x_{j,n_j}).
\]
Its empirical whole-behavior map is
\[
\widehat\Phi_j^{[S]}(\theta)=n_j^{-1/2}
\bigl(\psi_j(F_\theta^{[S]}(x_{j,1})),\ldots,\psi_j(F_\theta^{[S]}(x_{j,n_j}))\bigr),
\]
The map includes every observation in the frozen commit block. During the atomic commit, all anchor Jacobian blocks
\[
A_j(x_{j,i})=D_\theta\psi_j\!\left(F_\theta^{[S_j^{\rm anc}]}(x_{j,i})\right)\big|_{\theta=\bar\theta_j}
\]
are streamed into a Frequent Directions sketch. At atomic commit, the sketch and its shrinkage certificate are rescaled by $n_j^{-1/2}$ and $n_j^{-1}$ respectively, producing $B_j\in\R^{\ell_j\times d}$ with predeclared row capacity $\ell_j$. The conceptual scaled row stack is
\[
C_j=n_j^{-1/2}[A_j(x_{j,1});\cdots;A_j(x_{j,n_j})].
\]
It is then fixed and need not be stored by the optimizer. In a certified population-map instantiation, raw observations may be discarded after the sketch and commit certificate are formed. In the empirical endpoint-check instantiation, the specification instead declares a fixed evidence capacity $n_{\max}$ and retains at most $n_{\max}$ fixed-size observation references (or the corresponding bounded tensors), transforms, labels, and anchor outputs per active segment so that the complete commit-block map can be reevaluated exactly. Later observations may create an append-only segment but do not replace this segment's evidence or behavior constructor. Explicit release deletes that segment's evidence and ends its future guarantee.
For a segment arising from a certified candidate, let $F_j^{\rm snap}$ denote the complete frozen validated candidate predictor on its commit evidence block, including the structural state frozen with the candidate snapshot; in the base-shield notation introduced below this state is $(\bar\theta_j,S_j^{\rm snap})$. Let $\mathcal S_j$ denote the candidate's frozen transfer inputs. At activation, $\mathcal D_j=\mathcal S_j$ and $n_j=|\mathcal S_j|$. Because validation may take time while the global learner continues moving, let $c_j$ be the atomic activation round and define the exactly observed \emph{joint deployed activation-transfer gap}
\begin{equation}\label{eq:activation-gap}
a_j^{\rm act}=n_j^{-1/2}\left(\sum_{i=1}^{n_j}
\norm{F_{c_j}(x_{j,i})-F_j^{\rm snap}(x_{j,i})}^2\right)^{1/2}.
\end{equation}
This is a raw-logit RMS gap. For the empirical behavior constructor $\psi_j$, define the gap in its own output norm by
\[
a_{j,\Phi}^{\rm act}=n_j^{-1/2}\left(\sum_{i=1}^{n_j}
\norm{\psi_j(F_{c_j}(x_{j,i}))-\psi_j(F_j^{\rm snap}(x_{j,i}))}^2\right)^{1/2}.
\]
For raw logits, $a_{j,\Phi}^{\rm act}=a_j^{\rm act}$; if $\psi_j$ is $L_{\psi,j}$-Lipschitz on the relevant outputs, then $a_{j,\Phi}^{\rm act}\le L_{\psi,j}a_j^{\rm act}$. Both compare complete deployed predictors, including their shields. The validation certificate belongs to the frozen snapshot $(\bar\theta_j,S_j^{\rm snap})$; the active-interval budget begins at the actually deployed state $(\theta_{c_j},S_{c_j})$. Transfer admission uses $a_j^{\rm act}$; snapshot-retention bounds use $a_{j,\Phi}^{\rm act}$ and subsequent certified motion. The gap is zero precisely when the deployed and snapshot logits agree on the frozen transfer evidence. The transfer rule below therefore attempts to reduce the gap through ordinary protected updates and commits only after the measured gap lies below a threshold fixed before the stream. If no compatible transfer direction is found, the corresponding transfer obstruction is returned and the deployed predictor is unchanged.

\subsubsection{Guarded compact-cardinal shielding of certified candidates}

A candidate becomes a transfer target only at the first anytime-valid certification crossing defined in Section~\ref{sec:consolidation}. At that stopping time, freeze the complete bounded commit evidence block $\mathcal S_j$, the final-frozen representation $z(x)$ of every member, the deployed logits $F_{\tau_j}(x)$, the frozen candidate structural state $S_j^{\rm snap}$, and the complete candidate-snapshot logits $Y_j(x)=F_j^{\rm snap}(x)=G_{\bar\theta_j}(z(x))+S_j^{\rm snap}(z(x))$. The fixed empirical transfer objective remains
\begin{equation}\label{eq:transfer-objective}
D_j(F)=\frac{1}{2|\mathcal S_j|}
\sum_{x\in\mathcal S_j}
\norm{F(x)-Y_j(x)}^2.
\end{equation}
The private parameter vector is never copied into deployment. Let $\mathcal C_t^{\rm cand}$ be the certified-candidate set. A predictable bounded least-recently-served rule selects $j_t^\star\in\mathcal C_t^{\rm cand}$ for service; every other certified candidate remains a safeguard. While certified, each candidate reserves its bounded route slot and frozen evidence. It leaves the certified set by atomic commitment, a separately controlled staleness crossing, exhaustion of a predeclared service-attempt limit, or declared experiment termination. The latter two exits record service or termination obstructions and carry no completion claim. The retained transfer evidence consists of the first validation observations up to the declared commit-evidence capacity; fitting outcomes are excluded.

The deployed logits are decomposed as
\begin{equation}\label{eq:shield-decomposition}
F_t(x)=G_{\theta_t}(z(x))+S_t(z(x)),
\end{equation}
where $G_{\theta}$ is the ordinary trainable base map and $S_t$ is bounded structural state over the final frozen representation $z$.
For any behavior constructor $\mathcal B_j$, write
\[
\Phi_j^{[S]}(\theta)=\mathcal B_j(G_\theta+S),
\qquad
\widetilde\Phi_j(\theta,S)=\mathcal B_j(G_\theta+S).
\]
The bracket in $\Phi_j^{[S]}$ means that the structural shield is held fixed while the persistent parameters move. The streamed Jacobian memory, compatible projector, safe radius, and endpoint check govern $\Phi_j^{[S_t]}$ during the safe-base proposal; the compact-cardinal transaction then replaces $S_t$ by $S_{t+1}$. The joint theorem accounts for both legs rather than treating the shield replacement as invisible structural state. For the current minibatch define analogously $\ell_t^{S}(\theta)$ as the declared loss of $G_\theta+S$ with $S$ held fixed, and set $g_t=\nabla\ell_t^{S_t}(\theta_t)$. In the joint exact-endpoint-emulation mode, this current-batch loss is required to factor through the finite deployed minibatch logits: there is a declared loss functional $\mathcal L_t$ such that
\begin{equation}\label{eq:prediction-loss-factorization}
\ell_t^{S}(\theta)=\mathcal L_t\!\left((G_\theta(z(x))+S(z(x)))_{x\in\mathcal B_t}\right).
\end{equation}
For any deployed predictor $F$, write $\mathcal L_t(F):=\mathcal L_t((F(x))_{x\in\mathcal B_t})$. Hence equality of all current-minibatch logits implies equality of the deployed current loss. A direct parameter regularizer, if used in some separate optimization objective, is not part of the exact endpoint-emulation ratio unless it too factors through these declared minibatch outputs.

\paragraph{Genuine ordinary counterfactual.}
Let $\mathsf U_t^0$ be the complete deterministic update operator used by the declared no-protection learner: it receives the same pre-step deployed state, current minibatch $\mathcal B_t$, active-coordinate mask, learning-rate and backtracking rules, renewal trial, and private random choices, but no protected basis, behavior budget, candidate constraint, or retention radius. Applied to the current state, it either abstains or returns base parameters $\theta_t^0$ and endpoint logits
\begin{equation}\label{eq:exact-counterfactual-logits}
U_t(x)=G_{\theta_t^0}(z(x))+S_t(z(x)),\qquad x\in\mathcal B_t.
\end{equation}
Its exact endpoint decrease is
\begin{equation}\label{eq:exact-counterfactual-decrease}
\Delta_t^0=\mathcal L_t(F_t)-\mathcal L_t(U_t)
=\ell_t^{S_t}(\theta_t)-\ell_t^{S_t}(\theta_t^0).
\end{equation}
This is the comparator used by the counterfactual-normalized mode. It is not projected into the protected nullspace and is not restricted by the retention-safe radius. If $\mathsf U_t^0$ has no accepted nonzero endpoint, AFM reports that the exact counterfactual endpoint is unavailable; it does not substitute a weaker denominator.

\paragraph{Persistent metaplastic safe endpoint.}
Using the selected protected-subspace projector $Q_t$, compatible projector $\Pi_t$, behavior budget $b_t$, leakage certificate $E_t$, curvature certificate $\overline H_t$, trust cap, and the same declared backtracking rule, let $R_t$ be the maximum radius within the trust cap satisfying $C_t(R_t)\le b_t$ (Equation~\eqref{eq:safe-radius}). The safe operator $\mathsf U_t^{\rm safe}$ returns either an accepted base endpoint
\[
\bar\theta_{t+1}=\theta_t+d_t^{\rm safe},
\qquad d_t^{\rm safe}\in\Range(\Pi_t),
\qquad \norm{d_t^{\rm safe}}\le R_t,
\]
or reports that a certified safe-base endpoint is unavailable. This is the only persistent base endpoint in the joint mode. The unrestricted endpoint $\theta_t^0$ is never installed. Every backtracking candidate for both operators is evaluated from the identical pre-step predictive transaction state: parameters, mutable buffers, current shield, module modes, and declared private randomness. Endpoint probing restores that state before every candidate factor and after each operator; only the complete accepted predictive state of the safe-base operator is eligible for commitment.

\paragraph{Counterfactual-normalized persistent assimilation.}
The general endpoint-emulation oracle below remains valid for any deterministic comparator. The stronger persistent-assimilation mode considered here is an auditable scalar-gradient specialization. Let $A_t$ be the availability projector with $\Range(\Pi_t)\subseteq\Range(A_t)$, put
\[
a_t=A_tg_t,\qquad q_t=\Pi_tg_t,
\]
and suppose the accepted ordinary comparator displacement is
\begin{equation}\label{eq:ordinary-scalar-comparator}
d_t^0=-\alpha_t a_t,\qquad \alpha_t>0,\quad 0\le \alpha_t\overline L_t\le1.
\end{equation}
Its feasible persistent reference is
\begin{equation}\label{eq:projected-counterfactual-reference}
v_t=\Pi_td_t^0=-\alpha_tq_t,\qquad s_t^0=\norm{v_t}.
\end{equation}
Define the certified retention charge
\[
C_t(s)=E_ts+\frac{\overline H_t}{2}s^2.
\]
For a predeclared charge coordinate $\eta_t\in[0,1]$, the normalized round budget is
\begin{equation}\label{eq:normalized-retention-budget}
b_t^{\rm norm}=\eta_t C_t(s_t^0),
\end{equation}
and the maximal certified reference-path fraction is
\begin{equation}\label{eq:normalized-path-fraction}
\lambda_t=\max\{\lambda\in[0,1]:C_t(\lambda s_t^0)\le b_t^{\rm norm}\}.
\end{equation}
If there is no active protected behavior, set $\lambda_t=1$. An endpoint backtrack may reduce this to $\widehat\lambda_t$; the normalized mode accepts only when $\widehat\lambda_t\ge\eta_t$. Comparator collinearity, projection idempotence, trust-cap membership, persistent descent, and the realized charge are checked explicitly. Failure returns a typed obstruction, with the declared path and denominator unchanged.

\begin{theorem}[Counterfactual-normalized persistent assimilation]\label{thm:normalized-assimilation-app}
Assume~\eqref{eq:ordinary-scalar-comparator}, $A_t$ and $\Pi_t$ are orthogonal projectors with $\Pi_tA_t=\Pi_t$, the accepted comparator has $\Delta_t^0>0$, and $s_t^0>0$. Assume further that the same fixed-shield current loss admits the two-sided $\overline L_t$ Taylor-remainder certificate~\eqref{eq:directional-smooth} on both complete line segments $\{\theta_t+s d_t^0:0\le s\le1\}$ and $\{\theta_t+s v_t:0\le s\le1\}$, and that the retention charge $C_t(s)$ is certified for every projected-reference displacement $s v_t/s_t^0$ with $0\le s\le s_t^0$. In particular, the full projected reference and every accepted realized backtrack lie inside their respective certified domains. Then the normalized construction satisfies
\begin{align}
\lambda_t&\ge\eta_t,\label{eq:normalized-fraction-guarantee}\\
C_t(\lambda_ts_t^0)&\le b_t^{\rm norm},\label{eq:normalized-charge-guarantee}
\end{align}
and $\eta_t=1$ gives $\lambda_t=1$, hence the complete projected comparator $v_t$. For any accepted realized fraction $\widehat\lambda_t\in[\eta_t,\lambda_t]\subseteq[0,1]$, define
\[
\kappa_t=\frac{\norm{q_t}^2}{\norm{a_t}^2}\in[0,1].
\]
Writing
\[
\Delta_t^{\rm base}=\ell_t^{S_t}(\theta_t)-
\ell_t^{S_t}(\theta_t+\widehat\lambda_tv_t),
\]
one has
\begin{equation}\label{eq:persistent-ratio}
\frac{\Delta_t^{\rm base}}{\Delta_t^0}
\ge
\widehat\lambda_t\kappa_t
\frac{1-\tfrac12\alpha_t\overline L_t\widehat\lambda_t}
     {1+\tfrac12\alpha_t\overline L_t}
\ge \frac{\widehat\lambda_t\kappa_t}{3}
\ge \frac{\eta_t\kappa_t}{3}.
\end{equation}
After the compact-cardinal residual is installed, the deployed prediction-loss current-batch ratio remains exactly one by~\eqref{eq:genuine-progress-ratio}. Thus $\eta_t$ governs persistent assimilation, $\kappa_t$ displays the unavoidable local compatibility price, and the shield supplies only the residual endpoint correction. If $s_t^0=0$, $\kappa_t=0$, the comparator is not scalar-gradient aligned, or any required certificate or endpoint inequality fails, AFM returns the corresponding compatibility, alignment, retention, descent, or numerical obstruction atomically.
\end{theorem}

\begin{proof}
The function $C_t$ is nonnegative, convex, nondecreasing on $[0,\infty)$, and satisfies $C_t(0)=0$. Hence
$C_t(\eta_ts_t^0)\le\eta_tC_t(s_t^0)=b_t^{\rm norm}$, proving~\eqref{eq:normalized-fraction-guarantee}; maximality gives~\eqref{eq:normalized-charge-guarantee}, and $\eta_t=1$ makes the endpoint feasible with equality. Smoothness, $\Pi_tA_t=\Pi_t$, and~\eqref{eq:ordinary-scalar-comparator} give
\[
\Delta_t^{\rm base}
\ge \widehat\lambda_t\alpha_t\norm{q_t}^2
-\frac{\overline L_t}{2}\widehat\lambda_t^2\alpha_t^2\norm{q_t}^2.
\]
The lower side of the two-sided Taylor-remainder certificate gives
\[
\Delta_t^0\le \alpha_t\norm{a_t}^2
+\frac{\overline L_t}{2}\alpha_t^2\norm{a_t}^2.
\]
Dividing yields the first quotient in~\eqref{eq:persistent-ratio}; since
$0\le\alpha_t\overline L_t\le1$ and $0\le\widehat\lambda_t\le1$, its scalar factor is at least $1/3$. The compact-cardinal theorem then restores the exact comparator logits on the current finite minibatch without changing the committed base displacement.
\end{proof}

Provisionally move the complete accepted safe-base predictive state to $\bar\theta_{t+1}$. Form one finite desired-output multiset. On the current minibatch require the exact counterfactual logits $U_t$; on every active protected evidence block and every unselected certified candidate block require the pre-update deployed logits; and on the selected candidate block require its frozen target or the predeclared contraction toward it:
\begin{equation}\label{eq:priority-transfer-program}
T_t(x)=
\begin{cases}
U_t(x),&x\in\mathcal B_t,\\
F_t(x),&x\in\mathcal P_t,\\
F_t(x),&x\in\mathcal S_j,\ j\in\mathcal C_t^{\rm cand}\setminus\{j_t^\star\},\\
(1-\xi_t)F_t(x)+\xi_tY_{j_t^\star}(x),&x\in\mathcal S_{j_t^\star},
\end{cases}
\qquad \xi_t\in(0,1].
\end{equation}
Here $\mathcal P_t$ is the union of the complete bounded active empirical evidence blocks; the evaluated construction uses $\xi_t=1$. The same restoration system is executed on every round with an active protected record, whether or not a candidate is currently served.

To remove arbitrary feature scale, use the fixed injective address map
\begin{equation}\label{eq:shield-address}
\iota(z)=\frac{z}{\sqrt{1+\norm z^2}},
\qquad \iota:\R^{d_Z}\longrightarrow\{u:\norm u<1\}.
\end{equation}
If one addressed feature occurs in several constraint rows, those rows are merged only when their requested logits agree under the declared exact or outward-certified equality convention. Distinct requests at the same address are a functional inconsistency: no deterministic predictor depending only on the frozen representation can satisfy them simultaneously.

The specification declares a bounded label-free guard bank $\mathcal Q_t$ of at most $Q_{\max}$ frozen-backbone addresses. Guards use no labels, evaluator data, task identity, session identity, or boundary metadata. Let $z_1,\ldots,z_N$ be the merged distinct raw frozen-representation constraint nodes, let
\[
u_i=\iota(z_i)
\]
be their injective compact addresses, and let $q_i$ be the corresponding requested logits from~\eqref{eq:priority-transfer-program}. Define
\[
R_i=q_i-G_{\bar\theta_{t+1}}(z_i)
\]
to be the required residual logits after the certified metaplastic safe-base update. The base map is always evaluated at the raw frozen representation $z_i$; $u_i$ is used only as the compact support address. A predeclared replay envelope $\varepsilon_a>0$ and support multiplier $c_{\rm sup}>1$ determine $r_i=c_{\rm sup}\varepsilon_a$; the evaluated construction uses $c_{\rm sup}=4$. Distinct constraint centers and nonmatching guards must be farther than $2r_i$ from center $i$. A guard within a constraint node's replay envelope is assigned to that node and inherits its requested residual; it is not also required to have zero residual. Guard silence refers only to the remaining nonmatching guards. Define the $C^1$ plateau-cardinal bump
\[
\beta_{r,\varepsilon_a}(u;c)=
\begin{cases}
1,&\norm{u-c}\le\varepsilon_a,\\
\left(1-\left(\dfrac{\norm{u-c}-\varepsilon_a}{r-\varepsilon_a}\right)^2\right)^2,
&\varepsilon_a<\norm{u-c}<r,\\
0,&\norm{u-c}\ge r,
\end{cases}
\]
and atomically replace the shield by
\begin{equation}\label{eq:shield-interpolant}
S_{t+1}(z)=\sum_{i=1}^N \beta_{r_i,\varepsilon_a}(\iota(z);u_i)R_i.
\end{equation}
There is no fitted interpolation coefficient or linear solve. Let $B_{\max}$, $J_{\max}$ and $C_{\max}$ bound minibatch size, active records and candidate/context slots, respectively. With at most $n_{\max}$ evidence inputs per record or candidate, the node count obeys
\begin{equation}\label{eq:shield-capacity}
N\le B_{\max}+J_{\max}n_{\max}+C_{\max}n_{\max}.
\end{equation}
The projected safe-radius update is the persistent base update in
counterfactual-normalized mode. If either the safe-base or shield transaction
fails, the update is rejected; the unrestricted endpoint is not substituted
for the protected update.

\begin{theorem}[Joint safe-base counterfactual-emulation oracle and sharp obstruction]\label{thm:priority-oracle-app}
Work in exact arithmetic and fix the deployed-progress acceptance threshold $\rho_{\rm tr}=0.25$. Assume that the genuine comparator $\mathsf U_t^0$ returns an accepted endpoint with $\Delta_t^0>0$, the metaplastic safe-base operator returns $\bar\theta_{t+1}=\theta_t+d_t^{\rm safe}$ with
\[
d_t^{\rm safe}\in\Range(\Pi_t),\qquad
\norm{d_t^{\rm safe}}\le R_t,
\]
and its complete bounded base-behavior and loss endpoint checks pass. Assume also that the requested logits in~\eqref{eq:priority-transfer-program} are consistent on duplicate frozen addresses, the node capacity is sufficient, pairwise support and guard separation are outward-certified, and endpoint arithmetic is certified. Then the atomic safe-base-plus-shield update exists and satisfies:
\begin{enumerate}[label=(\roman*)]
\item certified persistent base retention and descent,
\begin{align}
\norm{\Phi_t^{[S_t]}(\bar\theta_{t+1})-\Phi_t^{[S_t]}(\theta_t)}&\le b_t,\label{eq:joint-base-retention}\\
\ell_t^{S_t}(\bar\theta_{t+1})&\le \ell_t^{S_t}(\theta_t)-\frac12s_t^{\rm safe}\norm{\Pi_tg_t};\label{eq:joint-base-descent}
\end{align}
\item exact current counterfactual equality, $F_{t+1}(x)=U_t(x)$ for every $x\in\mathcal B_t$;
\item exact finite protected retention, $F_{t+1}(x)=F_t(x)$ for every $x\in\mathcal P_t$;
\item exact unselected-candidate safeguards and the selected contraction in~\eqref{eq:priority-transfer-program};
\item pairwise-disjoint support, exact silence on nonmatching guards, and exact zero residual outside the finite support union;
\item the global pointwise bound $\norm{S_{t+1}(z)}\le\max_i\norm{R_i}$.
\end{enumerate}
Here $s_t^{\rm safe}=\norm{d_t^{\rm safe}}$ and $\ell_t^{S_t}(\theta)$ is the current loss with the pre-step shield held fixed. Consequently
\begin{equation}\label{eq:genuine-progress-ratio}
\mathcal L_t(F_t)-\mathcal L_t(F_{t+1})=\Delta_t^0,
\qquad
\frac{\mathcal L_t(F_t)-\mathcal L_t(F_{t+1})}{\Delta_t^0}=1\ge\rho_{\rm tr}.
\end{equation}
For any broader deployed behavior map, define the exact structural charge
\begin{equation}\label{eq:shield-structural-charge}
\omega_{t}^{S}:=
\norm{\widetilde\Phi_t(\bar\theta_{t+1},S_{t+1})-
\widetilde\Phi_t(\bar\theta_{t+1},S_t)}.
\end{equation}
Then
\begin{equation}\label{eq:joint-deployed-retention}
\norm{\widetilde\Phi_t(\bar\theta_{t+1},S_{t+1})-
\widetilde\Phi_t(\theta_t,S_t)}\le b_t+\omega_t^S.
\end{equation}
For the declared finite protected-evidence map, item (iii) gives the stronger exact value zero rather than this generic upper bound. If any comparator, safe-base, consistency, capacity, address-resolution, or numerical assumption fails, AFM restores every predictive component touched by the transaction, including parameters, mutable model buffers, shield state, module modes and declared private randomness, and emits the corresponding typed obstruction. Structural route and renewal proposals are committed only after this transaction accepts. Monotone audit and service counters retain the failed attempt. A rejected protected transaction does not install the unrestricted endpoint.
\end{theorem}

\begin{proof}
The safe-base endpoint is an accepted projected step in $\Range(\Pi_t)$ with norm at most the retention-safe radius. The streamed-memory certificate and the declared curvature upper model therefore prove~\eqref{eq:joint-base-retention}; the safe-base current-loss endpoint check proves~\eqref{eq:joint-base-descent}. At each merged raw node $z_i$, its compact address is $u_i=\iota(z_i)$, so compact cardinality and support separation give $\beta_i(\iota(z_i))=1$ and $\beta_h(\iota(z_i))=0$ for $h\ne i$. Hence $S_{t+1}(z_i)=R_i$, and because the residual coefficient is computed relative to the base evaluated at the same raw node,
\[
G_{\bar\theta_{t+1}}(z_i)+S_{t+1}(z_i)
=G_{\bar\theta_{t+1}}(z_i)+R_i=q_i.
\]
The four choices of $q_i$ prove current equality, protected equality, candidate safeguards, and selected transfer. By~\eqref{eq:prediction-loss-factorization}, current-minibatch logit equality with $U_t$ proves~\eqref{eq:genuine-progress-ratio}. Compact support, guard exclusion, and the pointwise bound follow because at most one bump is active at an address. Equation~\eqref{eq:joint-deployed-retention} is the triangle inequality applied first to the persistent base move with $S_t$ fixed and then to the structural shield replacement. Duplicate contradiction is logically impossible for a deterministic address map; every other failed certificate triggers the declared full-state rollback.
\end{proof}

\begin{proposition}[Conservative finite-precision realization]\label{prop:envelope-verification}
A theorem-certified finite-precision implementation deploys exact-counterfactual restoration only when outward-certified arithmetic proves: availability and acceptance of the counterfactual endpoint; duplicate consistency; node and guard capacities; a certified replay envelope; unique replay-address ownership; pairwise center and guard separation; bounded cardinal residual arithmetic; and the current, protected, candidate, selected, and guard endpoint inequalities after the declared two-sided numerical envelope is added. All endpoint equalities and zero-damage statements below are exact-arithmetic identities. If finite arithmetic certifies only a discrepancy, let $\epsilon_{j,t}^{\rm num}\ge0$ bound the protected-map difference between the accepted post-step and pre-step states, and let $\epsilon_t^{\rm loss}$ bound the current prediction-loss discrepancy from the comparator. The corresponding deployed decrease is then at least $\Delta_t^0-\epsilon_t^{\rm loss}$. Such allowances are recorded explicitly; exact represented-value equality sets them to zero. Empirical endpoint mode instead reevaluates all bounded evidence and guards in floating point and reports the observed errors and configured tolerances. These observations are not outward-rounded certificates. Any failed check restores both the base parameters and shield.
\end{proposition}

\begin{proposition}[Full genuine ordinary progress and zero candidate damage]\label{prop:safe-transfer}
On every accepted joint safe-base counterfactual-emulation step, with the same comparator and declared threshold $\rho_{\rm tr}$,
\begin{align}
\mathcal L_t(F_t)-\mathcal L_t(F_{t+1})&=\Delta_t^0\ge\rho_{\rm tr}\Delta_t^0,\\
D_j(F_t)-D_j(F_{t+1})&=0\quad\forall j\in\mathcal C_t^{\rm cand}\setminus\{j_t^\star\},\\
D_{j_t^\star}(F_{t+1})&=(1-\xi_t)^2D_{j_t^\star}(F_t),
\end{align}
the persistent base endpoint satisfies its declared behavior budget, and every declared active finite deployed protected behavior is unchanged exactly. With $\xi_t=1$, the selected finite transfer objective is zero after one accepted service. The exact construction attains ratio one and therefore satisfies this declared progress threshold.
\end{proposition}

\begin{corollary}[Budget-controlled persistence off support]\label{cor:base-equivalence}
At every address outside the new compact support union,
\[
F_{t+1}(x)=G_{\bar\theta_{t+1}}(z(x)).
\]
Hence future off-support behavior follows the certified metaplastic base endpoint, not the unrestricted no-protection endpoint. Different rank or behavior-budget policies may therefore induce different persistent trajectories even though all accepted variants reproduce the same genuine no-protection logits on the finite current minibatch. If $\bar\theta_{t+1}=\theta_t^0$, the two base predictors coincide. Parameter equality is sufficient but is not necessary for functional equality.
\end{corollary}

For every certified candidate, define the exact deployed ledger increment
\[
\Delta^D_{j,t}=D_j(F_t)-D_j(F_{t+1}).
\]
In exact arithmetic, accepted compact-cardinal and certified-candidate-safe updates satisfy $\Delta^D_{j,t}\ge0$. Finite-precision nonincrease requires an outward bound certifying this signed inequality; an allowed positive objective error must instead be entered explicitly in the ledger. Hence
\begin{equation}\label{eq:transfer-ledger}
D_j(F_T)=D_j(F_{\tau_j})-
\sum_{t=\tau_j}^{T-1}\Delta^D_{j,t},
\qquad \Delta^D_{j,t}\ge0.
\end{equation}
An implementation must snapshot and roll back the base parameters, deployed shield, support radii, guard state, and frozen candidate and record shield states atomically. For the predeclared activation-gap threshold $\tau_j^{\rm act}\ge0$, define $D_j^{\rm act}=(\tau_j^{\rm act})^2/2$. On the same evidence block, $a_j^{\rm act}\le\tau_j^{\rm act}$ is equivalent to $D_j(F)\le D_j^{\rm act}$.

\begin{proposition}[Predictable non-starving transfer service]\label{prop:fair-service}
Suppose at most $C_{\max}$ certified candidates coexist. Initialize each candidate's service time to its certification round and update it at every actual service round. If one candidate is selected at each available transfer round by the least-recently-served rule, every continuously eligible candidate is selected at least once in every $C_{\max}$ such rounds, even when later candidates arrive. The rule is predictable and uses no task, route, episode, boundary, or evaluator metadata.
\end{proposition}

\begin{proof}
Fix a continuously eligible candidate $j$ and its current timestamp. A later arrival receives a later initial timestamp. Each competing candidate selected ahead of $j$ receives the current, hence larger, timestamp. At most $C_{\max}-1$ competitors coexist, so deterministic tie breaking selects $j$ after at most that many other selections.
\end{proof}

\begin{theorem}[Finite compact-cardinal transfer completion without a compatibility-rate assumption]\label{thm:transfer-completion}
An available service round has an accepted genuine comparator and an admissible persistent base endpoint, in addition to candidate eligibility. Let candidate $j$ remain certified and non-stale until its next such round. If the finite constraints for that service are functionally consistent, lie within the declared node and guard capacities, and the address and endpoint arithmetic are certified, then the choice $\xi=1$ gives
\begin{equation}\label{eq:transfer-contraction}
D_j(F_{t+1})=0
\end{equation}
and therefore passes every nonnegative predeclared activation-gap threshold on the same frozen evidence block after that one accepted service. Under Proposition~\ref{prop:fair-service}, this occurs within at most $C_{\max}$ available service rounds. For a predeclared $\xi\in(0,1)$, after $N$ accepted services,
\begin{equation}\label{eq:transfer-delay}
D_{j,N}\le(1-\xi)^{2N}D_{j,0},
\qquad
N\ge
\left\lceil
\frac{\log(D_j^{\rm act}/D_{j,0})}{2\log(1-\xi)}
\right\rceil
\end{equation}
suffices whenever $0<D_j^{\rm act}<D_{j,0}$. For $D_j^{\rm act}=0$ and $\xi<1$, geometric contraction alone gives no finite-service guarantee. There is no projected-PL, positive-cosine, inter-service-damage, or interpolation-conditioning assumption. Failure can occur only through a declared statistical, exact functional-consistency, capacity, address-certification, endpoint-certification, or service-availability obstruction.
\end{theorem}

\begin{proof}
The selected identity in Proposition~\ref{prop:safe-transfer} gives one-step completion for $\xi=1$ and geometric contraction otherwise. Proposition~\ref{prop:fair-service} supplies the service bound. The ledger~\eqref{eq:transfer-ledger} prevents intervening accepted updates from undoing the progress.
\end{proof}

This theorem concerns finite empirical function-space constraints; population retention requires the additional assumptions stated separately below. In contrast to a globally supported residual construction, the compact-cardinal oracle proves exactly where the structural residual can act, proves silence on a bounded label-free guard set, and prevents accumulation from multiple centers. Behavior inside an unguarded compact support but away from its center remains governed by the displayed pointwise residual bound and the separate routing, Jacobian, curvature, and evaluator statements below. If a current, protected, or candidate requirement requests contradictory logits at exactly the same observable address, no deterministic continuation of that representation can satisfy all of them; AFM must abstain or invoke a separately proved representation-renewal mechanism; no accepted transaction is formed from inconsistent requests.

At operational commitment AFM additionally requires the observed gap~\eqref{eq:activation-gap} to satisfy $a_j^{\rm act}\le\tau_j^{\rm act}$. This requirement does not replace the validation certificate or enlarge a retention budget. The private vector is not copied into deployment.

Frequent Directions maintains an observable shrinkage certificate $\Delta_j^{\rm FD}$ such that, deterministically for every $v\in\R^d$,
\begin{equation}\label{eq:fd}
0\le\norm{C_jv}^2-\norm{B_jv}^2
\le \Delta_j^{\rm FD}\norm v^2.
\end{equation}
Moreover, for every $k<\ell_j$,
\begin{equation}\label{eq:fd-tail}
\Delta_j^{\rm FD}\le
\frac{\norm{C_j-(C_j)_k}_F^2}{\ell_j-k},
\end{equation}
where $(C_j)_k$ is a best rank-$k$ approximation and $\|\cdot\|_F$ is the Frobenius norm~\citep{ghashami2016}.

At the start of round $t$, before the current gradient is revealed, the admissible specification also fixes a predictable orthogonal \emph{structural-availability projector} $A_t$. Its range contains exactly the parameter directions available to the optimizer in that round: ordinary active coordinates and, during a renewal trial, the named reset slot; dormant nontrial coordinates are excluded. Define
\begin{equation}\label{eq:availability}
g_t^A=A_tg_t,\qquad \widetilde B_{j,t}=B_jA_t.
\end{equation}
Leakage calculations and policy comparisons use only the directions available for updates within the fixed architecture.

The importance signal, evaluated only after the current gradient is observed, is
\begin{equation}\label{eq:importance}
\chi_{j,t}=\min\left\{1,\frac{\norm{\widetilde B_{j,t}g_t^A}^2}{1+\norm{g_t^A}^2}\right\},
\qquad w_{j,t}=1+\chi_{j,t}\in[1,2],
\end{equation}

For $K$ predeclared timescales $k=0,\ldots,K-1$, the exponential moving averages (EMAs) obey
\begin{equation}\label{eq:trace}
z_{j,t+1}^{(k)}=(1-\rho_k)z_{j,t}^{(k)}+\rho_k\chi_{j,t},
\qquad \rho_k=2^{-(k+1)}.
\end{equation}
New traces are initialized at zero.

Each allocation policy is a fixed pair $p=(\beta^p,r^p)$ with $\beta_k^p\ge0$, $\sum_k\beta_k^p\le1$, and $0\le r^p\le r_{\max}$. The specification also supplies a predictable frontier price $\zeta_t\in[0,\zeta_{\max}]$ that values current plasticity relative to retention leakage. The declared family may include the normalized scale-free choices
\[
\beta_k^{(\alpha)}=\frac{\rho_k^{\alpha-1}}{\sum_{h=0}^{K-1}\rho_h^{\alpha-1}},
\qquad 0<\alpha\le1,
\]
for a fixed finite grid of $\alpha$ values and ranks. Before the current importance signal is revealed a policy predicts
\[
\widehat u_{j,t}^{(p)}=1+\sum_{k=0}^{K-1}\beta_k^p z_{j,t}^{(k)}\in[1,2]
\]
and forms
\begin{equation}\label{eq:predicted-sketch}
\widehat S_{p,t}=\sum_{j\in\Aa_t}\widehat u_{j,t}^{(p)}\widetilde B_{j,t}^T\widetilde B_{j,t}.
\end{equation}
It chooses $Q_{p,t}$ as the orthogonal projector onto the top-$r^p$ nonzero eigenspace within $\Range(A_t)$, with the fixed deterministic eigenvector/tie convention from the admissible specification; if the available covariance has rank below $r^p$, all of its nonzero eigendirections are used. The realized weighted covariance is
\begin{equation}\label{eq:actual-sketch}
S_t^*=\sum_{j\in\Aa_t}w_{j,t}\widetilde B_{j,t}^T\widetilde B_{j,t}.
\end{equation}
For every counterfactual policy define its current blocked-gradient fraction
\begin{equation}\label{eq:blocked}
c_{p,t}^{\rm block}=
\frac{\norm{Q_{p,t}g_t^A}^2}{1+\norm{g_t^A}^2}\in[0,1].
\end{equation}
The controller has already selected $p_t$ and uses $Q_t=Q_{p_t,t}$. Because $\Range(Q_t)\subseteq\Range(A_t)$, the effective compatible projector
\begin{equation}\label{eq:compatible-projector}
\Pi_t=A_t-Q_t
\end{equation}
is the orthogonal projector onto the structurally available directions orthogonal to the selected protected subspace. The observable spectral-allocation residual used by the optimizer is exactly
\begin{equation}\label{eq:alloc-residual}
a_t^{\rm spec}=\tr\bigl(\Pi_tS_t^*\bigr)=\tr\bigl((I-Q_t)S_t^*\bigr).
\end{equation}
The pair $(a_t^{\rm spec},c_{p_t,t}^{\rm block})$ is the realized local stability-plasticity frontier: the first measures protected sensitivity left in plastic directions, while the second measures current descent energy removed by protection.
The anchor-drift term is
\begin{equation}\label{eq:anchor-drift}
d_t^{\rm anc}=\left(\sum_{j\in\Aa_t}w_{j,t}L_{J,j,t}^2
\norm{\theta_t-\bar\theta_j}^2\right)^{1/2},
\end{equation}
where $L_{J,j,t}$ bounds parameter-Jacobian variation for the current fixed-shield map $\widehat\Phi_j^{[S_t]}$. Define a certified anchor shield-transport bound
\[
h_{j,t}^{S}\ge\left\|D_\theta\widehat\Phi_j^{[S_t]}(\bar\theta_j)
-D_\theta\widehat\Phi_j^{[S_j^{\rm anc}]}(\bar\theta_j)\right\|_{\rm op},
\qquad e_t^{S}=\left(\sum_jw_{j,t}(h_{j,t}^{S})^2\right)^{1/2}.
\]
For the finite map, stored anchor logits and the parameter anchor suffice to evaluate this derivative difference. Frozen snapshot evaluation on later observations additionally retains the corresponding shield. For raw-logit constructors this term is zero because the additive shield is parameter independent. A nonlinear constructor requires the displayed certificate; finite-graph derivative enclosures provide one when applicable. Failure to certify it prevents a theorem-certified nonzero step. The empirical leakage certificate is
\begin{equation}\label{eq:epsilon}
\varepsilon_t=
\sqrt{a_t^{\rm spec}+\sum_{j\in\Aa_t}w_{j,t}\Delta_j^{\rm FD}}
+d_t^{\rm anc}+e_t^S.
\end{equation}
If a population map $\Phi_j$ is desired instead of $\widehat\Phi_j$, let $e_t^{\rm pop}\ge0$ be a certified additional allowance such that the population derivative obeys $\norm{J_tv}\le\varepsilon_t+e_t^{\rm pop}$ for every unit $v\in\Range(\Pi_t)$. This compares derivative norms; it does not subtract operators with different output spaces. The empirical finite-evidence instantiation is empirical protection, for which $e_t^{\rm pop}=0$; a population guarantee requires a separately justified certificate, for example from matrix concentration under declared sampling conditions.

Given $E_t=\varepsilon_t+e_t^{\rm pop}$, behavior budget $b_t\ge0$, certified curvature $\overline H_t\ge0$, and trust cap $R_t^{\rm cap}$, define the retention-safe radius
\begin{equation}\label{eq:safe-radius}
\widetilde R_t=\begin{cases}
+\infty,&E_t=\overline H_t=0,\\
0,&b_t=0,\ E_t+\overline H_t>0,\\
\dfrac{2b_t}{E_t+\sqrt{E_t^2+2\overline H_tb_t}},
&b_t>0,\ \overline H_t>0,\\[1.2ex]
\dfrac{b_t}{E_t},&b_t>0,\ \overline H_t=0,\ E_t>0,
\end{cases}
\qquad
R_t=\min\{R_t^{\rm cap},\widetilde R_t\}.
\end{equation}
It is the largest radius certified by the quadratic upper model $E_tR+\overline H_tR^2/2\le b_t$, before the independent trust cap is applied.

\subsubsection{Conflict, variation, evidence, and renewal richness}

The compatible gradient is
\begin{equation}\label{eq:q}
q_t=\Pi_tg_t.
\end{equation}
The discarded component $(I-\Pi_t)g_t$ measures local first-order conflict with the structurally available protected complement, including both protection and declared structural unavailability.

For the non-convex pathwise theorem, define realized one-sided variation
\begin{equation}\label{eq:variation}
\nu_t=[\ell_{t+1}(\theta_{t+1})-\ell_t(\theta_{t+1})]_+,
\qquad V_T=\sum_{t=1}^{T-1}\nu_t.
\end{equation}

For reopening record $j$, let the fixed predictable challenger procedure output a shadow prediction before the $n$th routed outcome. Let $Y^{\rm rec}_{j,n},Y^{\rm chal}_{j,n}\in[0,1]$ be the subsequently observed losses of the committed snapshot and challenger, and fix a hysteresis margin $\kappa_j\in[0,1]$. AFM uses the concrete observable increment
\begin{equation}\label{eq:score}
X_{j,n}=Y^{\rm rec}_{j,n}-Y^{\rm chal}_{j,n}-\kappa_j.
\end{equation}
The operational validity null is
\begin{equation}\label{eq:null}
\E[X_{j,n}\mid\Ff_{n-1}]\le0.
\end{equation}
Because $X_{j,n}$ lies in an interval of length two, its centered noise is conditionally $1$-sub-Gaussian by Hoeffding's lemma; a smaller certified parameter may be used. Under change, define the actual observable challenger advantage
\begin{equation}\label{eq:drift-seq}
\mu_{j,n}=\E[X_{j,n}\mid\Ff_{n-1}].
\end{equation}
It may be zero, positive, negative, or time-varying. Thus reopening uses raw routed outcomes and a predictable comparator, not a semantic oracle. A real semantic change that no declared challenger can expose has zero observable information and is correctly covered by the impossibility term.

An e-process is a nonnegative adapted statistic with expectation at most one at every finite stopping time under its null. Suppress the record index in $X_i=X_{j,i}$ and fix a probability measure $\Pi$ on $(0,\infty)$ that assigns positive mass to every nonempty interval; a half-normal density is one concrete choice. Use any certified sub-Gaussian scale $\sigma$ for~\eqref{eq:score} (the universal bounded-loss choice is $\sigma=1$) and define the continuous-mixture e-process
\begin{equation}\label{eq:eprocess}
E_n=\int_0^\infty\exp\left(
\lambda\sum_{i=1}^nX_i-\frac{\lambda^2\sigma^2n}{2}
\right)\,\Pi(d\lambda).
\end{equation}
It requires only the sufficient state $S_n=\sum_{i=1}^nX_i$ and $n$. For example, with the half-normal prior of scale $\tau_{\rm mix}>0$, set $A_n=\sigma^2n+\tau_{\rm mix}^{-2}$; then
\begin{equation}\label{eq:halfnormal-mixture}
E_n=\frac{2}{\tau_{\rm mix}\sqrt{A_n}}
\exp\left(\frac{S_n^2}{2A_n}\right)
\Phi_{\rm N}\left(\frac{S_n}{\sqrt{A_n}}\right),
\end{equation}
where $\Phi_{\rm N}$ is the standard-normal distribution function. Thus adaptation to unknown positive drift has exact fixed-state evaluation and does not require an ever-growing grid.

Index stall episodes by $e$ and their renewal trials by $n$; $\Ff_{e,n-1}$ is the history before trial $n$. A renewal trial resets a zero-gated batch and recomputes the \emph{full} current loss gradient $g^{\rm reset}_{e,n}$ and the same global compatible direction used by the ordinary optimizer,
\[
q^{\rm reset}_{e,n}=\Pi_{e,n}g^{\rm reset}_{e,n}.
\]
Let $M_{e,n}$ be the predictable orthogonal coordinate projector onto the named trial slot, with $M_{e,n}A_{e,n}=M_{e,n}$. An eligible dormant slot satisfies the declared zero-gate invariant: while its functional gate is zero, the current loss and every active protected behavior have zero derivative with respect to its internal nongate coordinates, and every retained protected sketch has zero columns on those coordinates; a previously activated slot is returned to the renewal pool only by a separately certified exact-dormancy operation that re-establishes this invariant. Because the protected covariance has zero columns in those dormant internal coordinates and the current gradient has zero components there, the compatible projected gradient has zero dormant-internal component as well. Thus any nonzero component selected by $M_{e,n}$ contains nonzero gate-coordinate motion. A claim of nonzero renewed function additionally requires that the resulting gated module contribution is not identically zero; this condition is separate from the gradient-motion event used to define richness. The trial-slot component of the global compatible gradient is
\[
r_{e,n}=M_{e,n}q^{\rm reset}_{e,n}.
\]
For a declared threshold $\gamma>0$, a trial is $\gamma$-useful when $\norm{r_{e,n}}\ge\gamma$. Its exact conditional richness is
\begin{equation}\label{eq:rho}
\rho_{e,n}(\gamma)=\Prob\!\left(
\norm{M_{e,n}\Pi_{e,n}g^{\rm reset}_{e,n}}\ge\gamma
\mid\Ff_{e,n-1}\right).
\end{equation}
No positivity is assumed. The safe update is always computed from the full direction $q^{\rm reset}_{e,n}$; $r_{e,n}$ is the exact diagnostic for whether that accepted update moves the renewed slot. Reset-induced changes in the current functional Jacobian are charged through the recomputed anchor-drift and population-discrepancy terms in the same leakage certificate~\eqref{eq:epsilon}.

\subsubsection{Complete AFM algorithm}
\begin{algorithmdef}[Adaptive Functional Metaplasticity]\label{alg:afm}
Fix the finite structural capacities and summable statistical error allocations in Definition~\ref{def:spec}. If an unprotected initialization prefix is used, complete it, freeze the representation, replay the bounded calibration references through that representation, and either certify the requested routing calibration or record that no positive routing conclusion is available. For each protected round $t$:
\begin{enumerate}[label=\arabic*.]
\item \textbf{Route.} Compute only the predeclared observable signatures and assign the current observations to the bounded registry using the fixed threshold and tie rule. Capacity overflow follows the predeclared release or unprotected-overflow policy; no evaluator task identity is used.
\item \textbf{Fit, certify, and service candidates.} Update finite private fitting blocks. Once a candidate is frozen, validate it only on later outcomes with a fresh summable error allocation. At the first certification crossing, freeze its evidence, outputs, and transfer target; thereafter update a separately allocated staleness process. Certified candidates are serviced by the predictable non-starving rule, and private parameters are never copied directly into deployment.
\item \textbf{Maintain append-only protected records.} At commitment, freeze the record's evidence, anchor, behavior map, and bounded sensitivity sketch. Re-anchoring creates a new segment rather than rewriting an active one. When structural capacity is exhausted, only the predeclared explicit release or overflow rule may remove a guarantee.
\item \textbf{Prepare renewal and allocate protection.} If prior history triggers a renewal trial, reset an eligible zero-gated slot provisionally and recompute structural availability. Before the current importance signal is disclosed, select a policy from the finite family, form the predicted protected covariance, and choose its declared rank. After the current gradient is observed, update the realized importance signals, frontier losses, policy weights, and the metaplastic traces in Equation~\eqref{eq:trace}.
\item \textbf{Compute the same-state comparator and persistent base.} Evaluate the genuine no-protection endpoint without installing it. Construct the protected projected reference, the normalized charge budget, and the maximal certified path fraction. Accept a persistent base endpoint only if the declared path-fraction, retention, descent, projection, trust-region, and numerical checks all pass.
\item \textbf{Complete the finite deployed transaction.} From the accepted persistent base, form the finite requested-output set for the current minibatch, active protected evidence, certified candidates, selected transfer target, and guards. Construct the compact-cardinal residual and accept only if all finite endpoint, capacity, separation, and numerical checks pass. Otherwise restore the entire predictive state. The unrestricted base endpoint is never a fallback for a failed protected transaction.
\item \textbf{Update reopening and route refinement.} Outcome evidence may justify reopening or release. A new observable route additionally requires the separately controlled signature evidence and effect-size conditions. Lack of signature information cannot be replaced by semantic labels.
\item \textbf{Finalize renewal.} Commit the provisional reset only when the accepted protected transaction produces nonzero trial-slot motion; otherwise restore the pretrial slot. The comparator, gradient and protection certificates in steps 4--6 all use the post-reset state.
\end{enumerate}
\end{algorithmdef}

\subsection{Task-free routing, consolidation, and reopening}
\label{app:routing-consolidation}

The record lifecycle couples observable routing to outcome-based certification. Observable separation supports positive route identification; when it fails, the pathwise retention guarantee still applies to the evidence actually routed.

\begin{theorem}[Exact task-free routing under separated observable contexts]\label{thm:routing-positive}
Suppose that at most $C_{\max}$ contexts are active and context $c$ emits signatures only in the closed ball $B(\mu_c,r_Z)$. Assume
\[
\norm{\mu_c-\mu_{c'}} > 4r_Z\quad(c\ne c'),
\qquad h_Z=2r_Z,
\]
and no capacity release occurs. Initialize a new centroid at its first assigned signature. On an accepted assignment $Z$, update its centroid $m_c\leftarrow(1-r)m_c+rZ$ with declared $r\in[0,1]$; a frozen committed centroid uses $r=0$. Then every subsequent signature is routed to the unique correct context slot, without task labels.
\end{theorem}

\begin{proof}
A centroid initialized from context $c$ lies in $B(\mu_c,r_Z)$. Because the ball is convex, every EMA update using another signature from $c$ keeps the centroid in the same ball. Hence a signature from $c$ lies at distance at most $2r_Z=h_Z$ from its own centroid. Its distance from any centroid of $c'\ne c$ is greater than $4r_Z-2r_Z=h_Z$. Thus nearest-threshold routing selects the unique correct slot. Before context $c$ has a slot, its first signature lies beyond threshold from every existing other-context centroid and therefore opens an empty candidate slot. Induction over arrivals completes the proof.
\end{proof}

\subsubsection{Frozen-representation calibration and observable families}

A finite unprotected initialization prefix may train the backbone, but calibration features produced while that backbone is changing do not constitute one fixed signature map.  The admissible sequence is therefore: retain bounded learner-visible prefix references and transform seeds; finish bootstrap; freeze the backbone; replay the declared prefix through that final frozen representation; and only then fix all signature centers, scales, radii, temporal state, and routing thresholds.

\begin{proposition}[Honest frozen-signature calibration]\label{prop:signature-calibration}
Let $R_1,\ldots,R_m$ be the declared learner-visible prefix references and let $Z^\star(R_i)$ be their signatures under the final frozen representation.  Any deterministic calibration statistic computed from $(Z^\star(R_i))_{i\le m}$ is measurable before the protected horizon and may initialize the fixed router.  If the requested calibrated threshold $h_Z^{\rm req}$ exceeds a predeclared operational ceiling $h_Z^{\max}$, applying the ceiling is permitted only while withholding the positive routing conclusion, with no invocation of Theorem~\ref{thm:routing-positive} for that calibration.  The universal retention, consolidation, reopening, and obstruction statements remain valid for the actual routed sequence.
\end{proposition}

\begin{proof}
The replay uses only pre-horizon learner-visible information and the final frozen map, so all calibrated quantities are fixed before protected decisions.  A clipped threshold generally does not satisfy the coverage or separation premise that produced $h_Z^{\rm req}$; withholding the positive routing corollary is therefore necessary.  The arbitrary-stream results condition on the realized routing and already charge mismatch explicitly, so they do not require a positive calibration claim.
\end{proof}

The specification may predeclare a finite family of bounded task-free signatures $Z^{(1)},\ldots,Z^{(M)}$, such as named frozen-layer or deterministic-augmentation statistics.  No family member may be selected with evaluator contexts or test-set purity.  For route $r$, let $E^{(m)}_{r,b}$ be a valid e-process for the stability null of expert $m$ on distinct observable blocks. Fix $w_m>0$ with $\sum_mw_m=1$ and route-specific error allocation $\alpha_r^{\rm sig}\in(0,1)$.

\begin{proposition}[Error-controlled finite signature family]\label{prop:signature-mixture}
The mixture
\begin{equation}\label{eq:signature-mixture}
E^{\rm mix}_{r,b}=\sum_{m=1}^Mw_mE^{(m)}_{r,b}
\end{equation}
is an e-process under the intersection stability null and therefore
\[
\Prob\!\left(\sup_b E^{\rm mix}_{r,b}\ge1/\alpha_r^{\rm sig}\right)
\le\alpha_r^{\rm sig}.
\]
If one predeclared expert $m^\star$ has positive information, its crossing analysis incurs only the fixed evidence overhead $\log(1/w_{m^\star})$.  If no observable expert has positive information, AFM makes no finite route-identification claim.
\end{proposition}

\begin{proof}
A nonnegative fixed weighted sum of e-processes is an e-process.  Ville's inequality gives the first statement, and
$E^{\rm mix}_{r,b}\ge w_{m^\star}E^{(m^\star)}_{r,b}$ gives the fixed log-evidence overhead.  Sequential classifier two-sample processes are one possible construction, but the theorem requires only valid adapted e-processes~\citep{jang2022,ramdas2023anytime}.
\end{proof}

The evaluated implementation uses $M=1$. Its hybrid descriptor concatenates 27 image-style statistics (channel means and standard deviations, horizontal and vertical differences, $2\times2$ pooled pixels and three channel covariances) with unit-normalized frozen-backbone features. Frozen prefix means and scales standardize and clip these features. Each observable block contributes its coordinatewise mean and population standard deviation; a causal EMA of these block descriptors is unit-normalized to obtain the signature. The temporal rate, clipping level, block size and centroid-freeze option are predeclared in the released configuration. This construction supplies observable evidence but does not assume that semantic contexts are identifiable.

\subsubsection{Operational consolidation and evidence-based reopening}\label{sec:consolidation}

The bounded-state instantiation separates candidate fitting from certification. For one routing slot, collect a finite fitting block $\mathcal T=((x_i,y_i))_{i=1}^{m_{\rm fit}}$, apply the fixed deterministic candidate optimizer to a private vector initialized before that block, and freeze its output $\bar\theta$. The outcomes in $\mathcal T$ are training outcomes only. The validation sequence starts strictly afterwards, so the snapshot prediction for every validation outcome is fixed before that outcome is observed. Candidate fitting may succeed or fail arbitrarily; no retention or risk claim is attached to it until the separate validation test commits.

During consolidation the privately fitted candidate snapshot $\bar\theta$ is frozen, while the deployed learner may continue ordinary protected AFM updates and may attempt the safe transfer rule above. Let $Y_n\in[0,1]$ be the frozen candidate's observed validation loss on the $n$th routed outcome and let $\mu_n=\E[Y_n\mid\Ff_{n-1}]$. For error allocation $\alpha\in(0,1)$, write $\bar Y_n=n^{-1}\sum_{i=1}^nY_i$ and define
\begin{equation}\label{eq:commit-ucb}
U_n=\bar Y_n+
\sqrt{\frac{\log(\pi^2n^2/(6\alpha))}{2n}}.
\end{equation}

\begin{theorem}[Anytime-valid operational consolidation]\label{thm:commit}
With probability at least $1-\alpha$, simultaneously for every $n\ge1$,
\begin{equation}\label{eq:commit-cert}
\frac1n\sum_{i=1}^n\mu_i\le U_n.
\end{equation}
For a predeclared risk threshold $\tau_{\rm risk}\in[0,1]$, certification at $U_n\le\tau_{\rm risk}$ guarantees average conditional validation risk at most $\tau_{\rm risk}$ on that sequence. No stationarity or iid assumption is needed.
\end{theorem}

\begin{proof}
For fixed $n$, put $r_n=U_n-\bar Y_n$. Conditional Hoeffding-Azuma gives
\[
\Prob\left(\frac1n\sum_{i=1}^n(\mu_i-Y_i)>r_n\right)\le e^{-2nr_n^2}
=\frac{6\alpha}{\pi^2n^2}.
\]
Sum over $n$.
\end{proof}

A bounded candidate test also fixes a maximum validation count $N_{\rm val}$. Let $S_n=\sum_{i=1}^nY_i$. Because every future loss is nonnegative, the smallest UCB that can be attained at any remaining validation count, even if all remaining losses are zero, is
\begin{equation}\label{eq:best-case-commit}
U^{\rm best}_{n\to N_{\rm val}}=
\min_{m\in\{n,\ldots,N_{\rm val}\}}
\left[
\frac{S_n}{m}+
\sqrt{\frac{\log(\pi^2m^2/(6\alpha))}{2m}}
\right].
\end{equation}
If this quantity exceeds $\tau_{\rm risk}$, no future validation count can cross the certification threshold even under zero remaining loss, so rejection is algebraically forced and may occur immediately; otherwise the candidate is rejected without commitment at $N_{\rm val}$ if its UCB has never crossed the threshold. Rejection makes no statistical claim. Starting another candidate with a fresh summably allocated $\alpha^{\rm cert}$ preserves Theorem~\ref{thm:commit} and lifetime error control.

With minimum eligible count $n_{\min}\ge1$ and record-specific bound $U_{j,n}$, define the first certification crossing
\begin{equation}\label{eq:first-certification-crossing}
\tau_j^{\rm cert}=\inf\{n\in\mathcal N_j:\ n\ge n_{\min},\ U_{j,n}\le\tau_{\rm risk}\},
\end{equation}
Here $\mathcal N_j$ is the predeclared set of eligible checking counts, including validation-batch endpoints in the implementation. The anytime bound remains valid on this subset. When this stopping time is finite, AFM freezes the count, loss sum, UCB, bounded commit evidence, candidate outputs, streamed sketch, and transfer objective.  It stops appending outcomes to that certification sequence.  The transfer target is therefore immutable while the deployed learner approaches it.

For later routed outcomes, let $Y^{\rm stale}_{j,m}\in[0,1]$ be the frozen candidate's observed error and define
\[
X^{\rm stale}_{j,m}=Y^{\rm stale}_{j,m}-\tau_{\rm risk}.
\]
A fresh summably allocated half-normal-mixture e-process, with certified sub-Gaussian scale at least $1/2$, tests the post-certification null
\[
\E[X^{\rm stale}_{j,m}\mid\Ff_{m-1}]\le0.
\]

\begin{proposition}[Frozen first-crossing certificate and separately controlled staleness]\label{prop:frozen-certificate}
With probability at least $1-\alpha_j^{\rm cert}$, the certificate in Theorem~\ref{thm:commit} holds at the stopping time $\tau_j^{\rm cert}$.  Conditional on the separately allocated staleness null, the probability that the post-certification e-process ever falsely cancels the candidate is at most $\alpha_j^{\rm stale}$.  The two conclusions hold simultaneously over the learner's lifetime under the declared summable allocation.
\end{proposition}

\begin{proof}
Theorem~\ref{thm:commit} is simultaneous in $n$ and therefore valid at the first crossing.  The staleness process uses only subsequent outcomes with predictions fixed in advance and a fresh allocation; Ville's inequality controls its lifetime crossing.  A union bound over all instantiated processes gives the lifetime statement~\citep{waudbysmith2024betting,ramdas2023anytime}.
\end{proof}

A certified candidate whose activation gap exceeds $\tau^{\rm act}$ enters the predictable bounded transfer queue.  It is committed when the measured gap satisfies $a_j^{\rm act}\le\tau_j^{\rm act}$, canceled if the staleness process crosses, or rejected after its fixed number of service opportunities.  These outcomes respectively record successful safe transfer, observable post-certification invalidation, or a quantified transfer-service obstruction.  No alternative deployed predictor is protected in their place.

\subsubsection{Separately controlled reopening and observable route refinement}

Outcome contradiction and observable context displacement are different statistical questions.  For record $j$, let $E^{\rm out}_{j,n}$ denote the outcome e-process driven by~\eqref{eq:score}, where $n$ counts routed outcomes.  Separately, let $b$ count \emph{distinct observable signature blocks}.  A signature repeated for several outcome-resolved members contributes once to $b$ and once to the route-refinement statistic.

Let $c(j)$ identify record $j$'s source route. At commitment, freeze its centroid $m_{c(j)}$ and a predictable reference radius $r_j\in[0,2]$ computed from pre-crossing information by the declared specification.  Assume signatures are normalized so $\norm{Z_{j,b}}\le1$ and $\norm{m_{c(j)}}\le1$.  Define
\begin{equation}\label{eq:signature-score}
D_{j,b}=\norm{Z_{j,b}-m_{c(j)}},
\qquad
W_{j,b}=\frac{D_{j,b}-r_j}{2}.
\end{equation}
For fixed $r_j$, $W_{j,b}$ lies in an interval of length one, so its centered noise is conditionally $1/2$-sub-Gaussian.  The operational observable-stability null is
\begin{equation}\label{eq:signature-null}
\E[W_{j,b}\mid\mathcal G_{j,b-1}]\le0,
\end{equation}
where $\mathcal G_{j,b-1}$ contains the history before the $b$th distinct signature block.  This is an operational null, not an assumption that semantic contexts are identifiable.  If $r_j$ is estimated from a stochastic recurrence model, any claim that~\eqref{eq:signature-null} holds must be justified by that model and, when statistical, by a separately allocated calibration event; absent such a model, the universal theorem reports violation of the null as observable-shift evidence without treating recurrence as an assumed property.

Using a half-normal probability measure $\Pi_Z$ on $(0,\infty)$ with predeclared scale and certified scale $\sigma_Z=1/2$ (or any larger certified scale), define
\begin{equation}\label{eq:signature-eprocess}
E^{\rm sig}_{j,b}
=\int_0^\infty
\exp\!\left(
\lambda\sum_{i=1}^b W_{j,i}
-\frac{\lambda^2\sigma_Z^2b}{2}
\right)\Pi_Z(d\lambda).
\end{equation}
It has the same fixed-state half-normal evaluation as~\eqref{eq:halfnormal-mixture}, with the outcome sufficient statistics replaced by the signature-block sufficient statistics.

Let $\bar Z_{j,b}=b^{-1}\sum_{i=1}^b Z_{j,i}$.  For error allocations $\alpha_j^{\rm open},\alpha_j^{\rm split}\in(0,1)$, minimum block count $b_{\rm split}\ge1$ and centroid-displacement threshold $h_{\rm split}\ge0$, route refinement requires
\begin{equation}\label{eq:route-split}
E^{\rm out}_{j,n}\ge \frac1{\alpha_j^{\rm open}},
\qquad
E^{\rm sig}_{j,b}\ge \frac1{\alpha_j^{\rm split}},
\qquad
b\ge b_{\rm split},
\qquad
\norm{\bar Z_{j,b}-m_{c(j)}}\ge h_{\rm split}.
\end{equation}
Then the bounded route-refinement rule may preserve all active source segments and allocate one distinct centroid at $\bar Z_{j,b}$, subject to the declared context-capacity policy.  The split changes only future routing and alters no protected anchor, evidence block, sketch, active budget, or current parameter.  If only the outcome process crosses, the record remains protected during a fixed finite diagnostic grace window and is then reopened or released without allocating a route unless~\eqref{eq:route-split} becomes true.

\begin{proposition}[Separate lifetime control and retention invariance]\label{prop:route-split}
For a record satisfying the outcome null~\eqref{eq:null},
\[
\Prob\!\left(\sup_n E^{\rm out}_{j,n}\ge1/\alpha_j^{\rm open}\right)
\le\alpha_j^{\rm open}.
\]
For a record satisfying the observable-stability null~\eqref{eq:signature-null}, even when the outcome null is false,
\[
\Prob\!\left(\text{record $j$ ever creates a route split}\right)
\le
\Prob\!\left(\sup_b E^{\rm sig}_{j,b}\ge1/\alpha_j^{\rm split}\right)
\le\alpha_j^{\rm split}.
\]
If construction of the frozen reference set or radius is covered by a separate event of failure probability $\eta_j$, the unconditional false-split bound is $\eta_j+\alpha_j^{\rm split}$.  Conditional on any permitted split, every source record's existing retention guarantee is unchanged until explicit release.
\end{proposition}

\begin{proof}
Each outcome and signature exponential component is a nonnegative supermartingale under its own null, and fixed mixing preserves that property.  Ville's inequality gives the two bounds separately.  A split is a subset of the signature crossing event regardless of whether semantic or predictive contradiction makes the outcome process cross.  The split performs no parameter update and modifies none of the source records' protected state, so their retention proofs are unchanged.  The optional calibration term follows by a union bound.
\end{proof}

This refinement does not assert that an external semantic context has been identified.  Semantic-only change can justify reopening from outcomes but cannot justify a new route when the declared signature law is stable.  Observational indistinguishability therefore remains an explicit routing-mismatch obstruction rather than being converted into a noisy semantic oracle.

\begin{theorem}[Lifetime false-reopening control]\label{thm:falseopen}
Under~\eqref{eq:null}, $E_n$ is a nonnegative supermartingale with $E_0=1$. Therefore
\begin{equation}\label{eq:falseopen}
\Prob\left(\sup_{n\ge1}E_n\ge1/\alpha\right)\le\alpha.
\end{equation}
Summable allocation of $\alpha$ over records controls every false reopening over an unbounded lifetime.
\end{theorem}

\begin{proof}
Each exponential component is a supermartingale by conditional sub-Gaussianity. Tonelli's theorem shows that integration against the fixed probability measure $\Pi$ preserves this property. Apply Ville's inequality; see~\citet{ramdas2023anytime}.
\end{proof}

\begin{theorem}[Delay under actual observable contradiction]\label{thm:delay}
Consider the continuing process~\eqref{eq:eprocess} after $m$ observations, conditional on its history $\Ff_m$ and score $S_m$. Suppose subsequent increments have conditional mean at least $\Delta>0$ and centered conditional sub-Gaussian scale $\sigma>0$. Define
\[
I_\Delta=\left[\frac{\Delta}{2\sigma^2},\frac{3\Delta}{4\sigma^2}\right],\quad
\pi_\Delta=\Pi(I_\Delta)>0,\quad
\Omega_m(\Delta)=\frac{3\Delta(-S_m)_+}{4\sigma^2}
+\frac{9m\Delta^2}{32\sigma^2}.
\]
For $0<\beta<1$, crossing occurs within a number of additional observations bounded by
\begin{equation}\label{eq:delay}
N\le\left\lceil C\frac{\sigma^2}{\Delta^2}
\left[\log\frac1\alpha+\log\frac1\beta+\log\frac1{\pi_\Delta}
+\Omega_m(\Delta)+1\right]\right\rceil
\end{equation}
with conditional probability at least $1-\beta$; on the same event the wealth remains above threshold at every later observation while the drift condition continues. One may take $C=32$. The fresh-process case has $m=0$, $S_m=0$ and $\Omega_m=0$. A later change in a continuing detector retains its history cost; the algorithm need not observe the change time. No finite delay is asserted without positive observable drift.
\end{theorem}

\begin{proof}
Let $Z_n=\sum_{i=m+1}^{m+n}(X_i-\mu_i)$ and $\lambda_0=\Delta/(4\sigma^2)$. Conditional sub-Gaussianity makes
$\exp(-\lambda_0Z_n-\lambda_0^2\sigma^2n/2)$ a nonnegative supermartingale starting at one. If $Z_n<-n\Delta/4$, its value exceeds $\exp(n\Delta^2/(32\sigma^2))$. Ville's inequality therefore gives
\[
\Prob\!\left(\exists n\ge N:Z_n<-n\Delta/4\mid\Ff_m\right)
\le\exp\!\left(-\frac{N\Delta^2}{32\sigma^2}\right).
\]
On the complementary event, $S_{m+n}\ge S_m+3n\Delta/4$ for every $n\ge N$. Uniformly over $\lambda\in I_\Delta$, the history contribution is at least $-\Omega_m(\Delta)$ and
$3\lambda n\Delta/4-\lambda^2\sigma^2n/2\ge n\Delta^2/(4\sigma^2)$.
Hence
\[
E_{m+n}\ge\pi_\Delta\exp\!\left[-\Omega_m(\Delta)+\frac{n\Delta^2}{4\sigma^2}\right]
\qquad(n\ge N).
\]
Choosing $N$ as the ceiling in~\eqref{eq:delay} with $C=32$ makes the exceptional probability at most $\beta$ and the last display at least $1/\alpha$ simultaneously for every $n\ge N$.
\end{proof}

The history term is necessary for the continuing statistic. If $M$ increments equal $-a$ before a change and subsequent increments equal $\Delta$, then $S_{M+n}\le0$ for $n\le Ma/\Delta$ and every mixture component is at most one. A history-independent delay cannot hold uniformly over such histories.

\begin{theorem}[Joint delay under outcome and observable information]\label{thm:split-delay}
Condition on the histories of the outcome and signature processes at a change. Suppose their subsequent conditional drifts are at least $\Delta_{\rm out}>0$ and $\Delta_{\rm sig}>0$. Let $\Omega_{\rm out}$ and $\Omega_{\rm sig}$ be their history costs from Theorem~\ref{thm:delay}, with their respective pre-change counts and scores. Write $\omega_{\rm out}=\log(1/\pi_{\Delta_{\rm out}})$ and $\omega_{\rm sig}=\log(1/\pi_{\Delta_{\rm sig}})$. Here $\beta\in(0,1)$ is the conditional failure probability, $C=32$, $\alpha_{\rm out}=\alpha_j^{\rm open}$, $\alpha_{\rm sig}=\alpha_j^{\rm split}$, $\sigma_{\rm out}=\sigma$ and $\sigma_{\rm sig}=\sigma_Z$. The sufficient additional observation counts are bounded by $N_\star$, where
\begin{equation}\label{eq:split-delay}
\begin{aligned}
N_a&=\left\lceil C\frac{\sigma_a^2}{\Delta_a^2}
\left[\log\frac1{\alpha_a}+\log\frac2\beta+\omega_a+\Omega_a+1\right]\right\rceil,\\
N_\star&=\max\{N_{\rm out},N_{\rm sig}\},\qquad a\in\{\mathrm{out},\mathrm{sig}\}.
\end{aligned}
\end{equation}
With conditional probability at least $1-\beta$, both wealth thresholds hold at every evaluation for which both additional counts are at least $N_\star$. On that event, whenever capacity, minimum-block and vector-effect conditions hold and the source is still eligible, the rule permits a split. Observation counts are process-specific and do not imply an equal bound in global update rounds. A finite grace window may end eligibility before these conditions are met. If $\Delta_{\rm sig}=0$, no finite split-delay guarantee is asserted; ordinary reopening or release remains available.
\end{theorem}

\begin{proof}
Apply the uniform-in-count lower bound in Theorem~\ref{thm:delay} to the two continuing processes with failure probability $\beta/2$. A union bound gives both wealth inequalities simultaneously at all sufficiently late observation counts. The remaining hypotheses are exactly the deterministic admission and continued-eligibility requirements in~\eqref{eq:route-split}. Thus the stochastic crossing guarantee does not imply that capacity is available or that a released record remains eligible.
\end{proof}

\subsection{Whole-behavior memory and spectral protection}
\label{app:whole-behavior-memory}

The results below quantify the leakage represented by the bounded sensitivity sketches and characterize the exact first-order rank-plasticity frontier.

\subsubsection{Whole-behavior memory and the local frontier}

\begin{theorem}[Deterministic streamed-memory certificate]\label{thm:memory}
For every round and every unit vector $v\in\Range(\Pi_t)$,
\begin{equation}\label{eq:leak-cert}
\norm{J_tv}\le \varepsilon_t+e_t^{\rm pop}.
\end{equation}
For the routed empirical behavior, $e_t^{\rm pop}=0$ and the result is deterministic.
\end{theorem}

\begin{proof}
At the anchors, stack the realized weighted conceptual matrices $\sqrt{w_{j,t}}C_j$ and sketches $\sqrt{w_{j,t}}B_j$. By~\eqref{eq:fd},
\[
\sum_j w_{j,t}\norm{C_jv}^2
\le v^TS_t^*v+\left(\sum_j w_{j,t}\Delta_j^{\rm FD}\right)\norm v^2.
\]
Since $v\in\Range(\Pi_t)$ and $Q_t$ is an orthogonal projector,
\[
v^TS_t^*v\le\lambda_{\max}((I-Q_t)S_t^*(I-Q_t))
\le\tr((I-Q_t)S_t^*)=a_t^{\rm spec}.
\]
Transport first from the stored shield $S_j^{\rm anc}$ to $S_t$ at $\bar\theta_j$, then from $\bar\theta_j$ to $\theta_t$ with $S_t$ fixed. The weighted stacked error is at most $(e_t^S+d_t^{\rm anc})\norm v$ by the declared bounds and Minkowski's inequality. Since $w_{j,t}\ge1$, the unweighted stacked norm is no larger than the weighted norm. Add the valid population-empirical discrepancy certificate $e_t^{\rm pop}$.
\end{proof}

\begin{corollary}[Constructive population certificate]\label{cor:population}
Condition on a parameter/shield anchor and behavior constructor fixed independently of the commit-input sequence. Suppose those inputs are iid from $P_j$ and $\norm{A_j(x)}_{\rm op}\le R_j$ almost surely. A privately fitted candidate frozen before an independent validation sequence meets this sampling requirement; a data-adaptive anchor formed from the same inputs instead requires a uniform or otherwise valid adaptive concentration certificate. Let
\[
G_j=\E_{P_j}[A_j(x)^TA_j(x)],\qquad
\widehat G_{j,n}=C_{j,n}^TC_{j,n}
\]
for the scaled conceptual row stack formed from the first $n$ commit inputs. If $n_j$ is fixed before those inputs are observed, then with probability at least $1-\delta_j$,
\begin{equation}\label{eq:population-bernstein}
\norm{\widehat G_{j,n_j}-G_j}_{\rm op}
\le \xi_j:=R_j^2\left[
\sqrt{\frac{2\log(2d/\delta_j)}{n_j}}
+\frac{2\log(2d/\delta_j)}{3n_j}
\right].
\end{equation}
If instead the commitment count is selected adaptively from the same iid sequence, predeclare numbers $\delta_{j,n}>0$ with $\sum_{n\ge1}\delta_{j,n}\le\delta_j$ and define
\begin{equation}\label{eq:population-bernstein-anytime}
\xi_{j,n}:=R_j^2\left[
\sqrt{\frac{2\log(2d/\delta_{j,n})}{n}}
+\frac{2\log(2d/\delta_{j,n})}{3n}
\right].
\end{equation}
Then with probability at least $1-\delta_j$, simultaneously for every $n\ge1$,
$\norm{\widehat G_{j,n}-G_j}_{\rm op}\le\xi_{j,n}$; in particular the bound is valid at any data-dependent commit time. In the formulas below, $\xi_j$ denotes the applicable fixed-$n_j$ value or the simultaneous value $\xi_{j,n_j}$ at commitment.
Assume additionally that the pointwise Jacobian at the stored shield is $L_{J,j}$-Lipschitz in $\theta$. If a population behavior is evaluated with $S_t\ne S_j^{\rm anc}$, let $h_{j,t}^{S,\rm pop}$ certify its population derivative change at the current parameter, and put $e_t^{S,\rm pop}=(\sum_jw_{j,t}(h_{j,t}^{S,\rm pop})^2)^{1/2}$. On the intersection of these events, a valid choice of current-parameter population allowance in~\eqref{eq:leak-cert} is
\[
e_t^{\rm pop}:=\left(\sum_{j\in\Aa_t}w_{j,t}\xi_j\right)^{1/2}
+\left(\sum_{j\in\Aa_t}w_{j,t}L_{J,j}^2
\norm{\theta_t-\bar\theta_j}^2\right)^{1/2}+e_t^{S,\rm pop}.
\]
The second term transports the population derivative from the anchor to the current parameter at the stored shield, and the final term transports its structural state. The empirical parameter and shield transports are already included in $\varepsilon_t$; retaining both is conservative. Thus the population term is constructible under declared sampling conditions rather than assumed to vanish, including at adaptive commitment times when the simultaneous allocation is used.
\end{corollary}

\begin{proof}
For each fixed $n$, apply self-adjoint matrix Bernstein to $A_j(x_i)^TA_j(x_i)-G_j$ and divide by $n$~\citep{tropp2012}. This gives~\eqref{eq:population-bernstein} at a fixed $n_j$. For adaptive commitment, apply the same fixed-$n$ bound with level $\delta_{j,n}$ and take a union bound over $n$; the resulting event is simultaneous and hence remains valid at a stopping time. For any unit $v$, the population quadratic form is at most the empirical form plus the applicable $\xi_j$. Stack the weighted records and take square roots. Transport the population derivative in parameter space at the stored shield and then to the current shield; the two declared transport terms bound these changes. Adding the nonnegative empirical transport terms already included in $\varepsilon_t$ preserves the bound.
\end{proof}

\begin{theorem}[Exact rank-$r$ local frontier]\label{thm:spectral}
Let $J:\R^d\to\Hh$ be compact with singular values $\sigma_1\ge\cdots\ge\sigma_d\ge0$, and let $0\le r<d$. Among all plastic subspaces $S\subseteq\R^d$ of dimension $d-r$,
\begin{equation}\label{eq:minmax}
\inf_{\dim S=d-r}\ \sup_{v\in S,\ \norm v=1}\norm{Jv}
=\sigma_{r+1}(J).
\end{equation}
The optimum is the span of the right singular vectors associated with $\sigma_{r+1},\ldots,\sigma_d$.
\end{theorem}

\begin{proof}
This is the Courant-Fischer min-max characterization applied to $J^*J$.
\end{proof}

\begin{corollary}[Unavoidable first-order forgetting price]
No rank-$r$ linear protection rule can guarantee first-order leakage smaller than $\sigma_{r+1}(J)$ in every unit plastic direction. Thus the spectral tail in the AFM certificate is not an artifact of the proof.
\end{corollary}

\subsection{Retention and persistent adaptation}
\label{app:retention-adaptation}

The memory certificate bounds the derivative leakage of a compatible step. Combining it with the curvature certificate yields a one-step retention bound; composing accepted steps then connects finite protection to persistent movement and projected stationarity.

\subsubsection{Retention and compatible adaptation}

\begin{theorem}[One-step joint persistent retention and exact deployed progress]\label{thm:onestep}
For every accepted protected update, let $\bar\theta_{t+1}=\theta_t+d_t^{\rm safe}$ be the persistent safe-base endpoint and let $s_t^{\rm safe}=\norm{d_t^{\rm safe}}$. Then
\begin{align}
\norm{\Phi_t^{[S_t]}(\bar\theta_{t+1})-\Phi_t^{[S_t]}(\theta_t)}
&\le (\varepsilon_t+e_t^{\rm pop})s_t^{\rm safe}
+\frac{\overline H_t}{2}(s_t^{\rm safe})^2\le b_t,
\label{eq:retention-step}\\
\ell_t^{S_t}(\bar\theta_{t+1})
&\le \ell_t^{S_t}(\theta_t)-\frac12s_t^{\rm safe}
\norm{\Pi_t\nabla\ell_t^{S_t}(\theta_t)},
\label{eq:descent-step}\\
\mathcal L_t(F_{t+1})&=\mathcal L_t(U_t)=\mathcal L_t(F_t)-\Delta_t^0,
\label{eq:deployed-progress-step}\\
\frac{\ell_t^{S_t}(\theta_t)-\ell_t^{S_t}(\bar\theta_{t+1})}{\Delta_t^0}
&\ge \frac{\widehat\lambda_t\kappa_t}{3}
\quad\text{in counterfactual-normalized mode.}
\label{eq:onestep-persistent-ratio}
\end{align}
For the corresponding joint deployed behavior constructor with the stated fixed-shield certificate,
\begin{equation}\label{eq:semantic-step}
\norm{\widetilde\Phi_t(\bar\theta_{t+1},S_{t+1})-
\widetilde\Phi_t(\theta_t,S_t)}
\le b_t+\omega_t^S.
\end{equation}
For an external evaluator map $\widetilde\Psi_t=\widetilde\Phi_t+\widetilde\Delta_t$ defined on the joint state,
\begin{equation}\label{eq:external-semantic-step}
\norm{\widetilde\Psi_t(\bar\theta_{t+1},S_{t+1})-
\widetilde\Psi_t(\theta_t,S_t)}
\le b_t+\omega_t^S+\omega_t^{\Delta}+r_ts_t^{\rm safe}+\frac{k_t}{2}(s_t^{\rm safe})^2.
\end{equation}
On every declared finite protected-evidence block, the deployed drift in~\eqref{eq:semantic-step} is exactly zero by Theorem~\ref{thm:priority-oracle}. No routing correctness is assumed.
\end{theorem}

\begin{proof}
The safe-base displacement lies in $\Range(\Pi_t)$ and is clipped by the root of the certified quadratic retention upper model, proving~\eqref{eq:retention-step}. The loss upper model and the accepted safe-base endpoint check prove~\eqref{eq:descent-step}; Theorem~\ref{thm:normalized-assimilation} gives~\eqref{eq:onestep-persistent-ratio} in the normalized mode. The shield construction reproduces the genuine no-protection logits on the complete current minibatch, proving~\eqref{eq:deployed-progress-step}. Equation~\eqref{eq:semantic-step} is the joint-state bound~\eqref{eq:joint-deployed-retention}; applying the same fixed-$S_t$ Taylor bound to $\widetilde\Delta_t$ and then charging its exact shield-replacement change by~\eqref{eq:routing-structural-charge} gives~\eqref{eq:external-semantic-step}. Exact finite deployed retention is item (iii) of Theorem~\ref{thm:priority-oracle}.
\end{proof}

\begin{corollary}[Active-interval joint retention and exact finite deployed retention]\label{cor:interval}
For record $j$, use $\omega_{j,t}^S,\omega_{j,t}^{\Delta},r_{j,t},k_{j,t}$ for the corresponding structural and routing charges defined above. If it remains active on rounds $\tau,\ldots,T-1$, the corresponding external behavior map $\widetilde\Psi_j=\widetilde\Phi_j+\widetilde\Delta_j$, with the stated certificates, obeys
\begin{equation}\label{eq:semantic-interval}
\norm{\widetilde\Psi_j(\theta_T,S_T)-
\widetilde\Psi_j(\theta_\tau,S_\tau)}
\le \sum_{t=\tau}^{T-1}
\left(b_t+\omega_{j,t}^S+\omega_{j,t}^{\Delta}+r_{j,t}s_t^{\rm safe}
+\frac{k_{j,t}}2(s_t^{\rm safe})^2\right).
\end{equation}
For the internal map $\widetilde\Phi_j$, the same interval bound holds with the three routing-mismatch terms omitted. For the declared finite protected-evidence map, every accepted joint step restores the complete pre-step deployed value, so the stronger identity
\begin{equation}\label{eq:finite-deployed-interval}
\widetilde\Phi_j(\theta_T,S_T)=
\widetilde\Phi_j(\theta_\tau,S_\tau)
\end{equation}
holds in exact arithmetic. With certified per-step discrepancies $\epsilon_{j,t}^{\rm num}$, its finite-precision counterpart is
\[
\norm{\widetilde\Phi_j(\theta_T,S_T)-\widetilde\Phi_j(\theta_\tau,S_\tau)}
\le\mathcal E_{j,\tau:T},\qquad
\mathcal E_{j,\tau:T}=\sum_{t=\tau}^{T-1}\epsilon_{j,t}^{\rm num}.
\]
A smaller interval envelope requires an independently certified fixed-reference bound. For this finite evidence map, if the record was activated at $c_j$ from validated snapshot $(\bar\theta_j,S_j^{\rm snap})$, then
\begin{equation}\label{eq:snapshot-interval}
\norm{\widetilde\Phi_j(\theta_T,S_T)-
\widetilde\Phi_j(\bar\theta_j,S_j^{\rm snap})}
\le a_{j,\Phi}^{\rm act}+\mathcal E_{j,c_j:T},
\end{equation}
where $\mathcal E_{j,c_j:T}=0$ in the exact construction. Thus exact restoration has no cumulative active-interval drift; approximate restoration pays its certified accumulated error. The budget controls the persistent safe-base leg, the explicit $\omega^S$ ledger charges the structural replacement for broader maps, and the finite protected map is restored exactly.
\end{corollary}

\begin{proof}
Telescope the one-step joint bound~\eqref{eq:external-semantic-step} to obtain~\eqref{eq:semantic-interval}. Equation~\eqref{eq:finite-deployed-interval} follows by induction from exact pre-step restoration on the complete finite evidence block. For approximate endpoints, apply the triangle inequality to their certified discrepancies. Bridging once from the validated snapshot to activation then gives~\eqref{eq:snapshot-interval}.
\end{proof}

\begin{proposition}[Explicit activation and release decomposition]\label{prop:release}
For the declared finite protected-evidence map, let record $j$ be activated at $c_j$ from snapshot $(\bar\theta_j,S_j^{\rm snap})$ and released at $r_j\le T$, with the convention $r_j=T$ if it remains active. Then
\begin{equation}\label{eq:release-decomp}
\norm{\widetilde\Phi_j(\theta_T,S_T)-\widetilde\Phi_j(\bar\theta_j,S_j^{\rm snap})}
\le a_{j,\Phi}^{\rm act}+\mathcal E_{j,c_j:r_j}
+\norm{\widetilde\Phi_j(\theta_T,S_T)-\widetilde\Phi_j(\theta_{r_j},S_{r_j})}.
\end{equation}
The activation term is the observed transfer from the validated snapshot. The interval term is zero under exact restoration and otherwise accounts for certified arithmetic errors; the final term is post-release drift intentionally uncontrolled by a bounded registry. Thus delayed activation and capacity release are reported separately from protected forgetting.
\end{proposition}

\begin{proof}
Insert and subtract $\widetilde\Phi_j(\theta_{r_j},S_{r_j})$, apply the triangle inequality, and use the applicable exact or finite-precision bound from Corollary~\ref{cor:interval} on the active interval.
\end{proof}

\begin{theorem}[Arbitrary-stream safe-base projected-stationarity bound]\label{thm:stationarity}
Let
\[
\nu_t^{\rm joint}:=
\max\{\ell_{t+1}^{S_{t+1}}(\theta_{t+1})-
\ell_t^{S_t}(\theta_{t+1}),0\}
\]
be the realized variation of the pre-step-shield loss sequence. For every horizon $T$,
\begin{equation}\label{eq:stationarity}
\sum_{t=1}^{T-1}s_t^{\rm safe}
\norm{\Pi_t\nabla\ell_t^{S_t}(\theta_t)}
\le 2\bigl(\ell_1^{S_1}(\theta_1)-\ell_{\inf}+V_T^{\rm joint}\bigr),
\qquad V_T^{\rm joint}=\sum_{t=1}^{T-1}\nu_t^{\rm joint}.
\end{equation}
This holds pathwise for smooth non-convex losses. Independently, every accepted deployed update has exact current-batch ratio one by~\eqref{eq:genuine-progress-ratio}.
\end{theorem}

\begin{proof}
Apply~\eqref{eq:descent-step}, add the realized joint-loss variation, telescope, and use the lower bound on the final loss. The shield-emulated deployed progress is a separate exact ledger and is not substituted into the base stationarity proof.
\end{proof}

\begin{corollary}[Realized lifelong compatible base plasticity plus exact current emulation]\label{cor:plasticity}
On any infinite tail for which
\[
\sum_t\nu_t^{\rm joint}<\infty,
\qquad
\sum_t s_t^{\rm safe}=\infty,
\]
one has
\[
\liminf_{t\to\infty}
\norm{\Pi_t\nabla\ell_t^{S_t}(\theta_t)}=0.
\]
If additionally $\sum_tb_t<\infty$, the cumulative certified fixed-shield behavior charge of the safe-base legs is finite, while every declared finite deployed evidence block has exact active-interval retention. A broader joint behavior has finite total drift whenever its structural and routing charges are also summable. Thus indefinitely available compatible persistent movement, exact finite protection, and ratio-one current endpoint emulation coexist on the realized accepted tail.
\end{corollary}

\begin{proof}
If the liminf were bounded below by $\gamma>0$, the left side of~\eqref{eq:stationarity} would dominate $\gamma\sum_ts_t^{\rm safe}=\infty$. The safe-base behavior-charge and exact finite deployed-retention statements follow from Theorem~\ref{thm:onestep} and Corollary~\ref{cor:interval}.
\end{proof}

\begin{proposition}[A checkable sufficient safe-base movement condition]\label{prop:movement-mass}
Suppose an infinite tail satisfies $\sum_t\nu_t^{\rm joint}<\infty$, $\sum_tb_t<\infty$, and $\sum_t\sqrt{b_t}=\infty$. If there is a constant $c_s>0$ such that every accepted nonzero safe-base update on that tail has
\[
s_t^{\rm safe}\ge c_s\sqrt{b_t},
\]
and, writing $\mathcal I=\{t:\text{the safe-base update is accepted and nonzero}\}$, the accepted rounds retain divergent square-root budget mass,
\[
\sum_{t\in\mathcal I}\sqrt{b_t}=\infty,
\]
then Corollary~\ref{cor:plasticity} applies. A sufficient condition is that the omitted rounds have finite square-root budget mass $\sum_{t\notin\mathcal I}\sqrt{b_t}<\infty$. The lower bound is an observable certificate condition on the realized safe radius, projected gradient, and endpoint backtracking factor; it is not inferred merely from first-order compatibility or from the shield update.
\end{proposition}

\begin{proof}
The assumptions give
$\sum_ts_t^{\rm safe}\ge\sum_{t\in\mathcal I}s_t^{\rm safe}\ge c_s\sum_{t\in\mathcal I}\sqrt{b_t}=\infty$.
Apply Corollary~\ref{cor:plasticity}.
\end{proof}

\subsection{Multi-timescale metaplasticity and allocation}
\label{app:metaplastic-allocation}

The following results justify the logarithmic trace bank and the adaptive rank/timescale controller. The controller is evaluated on a frontier cost that contains both residual protected leakage and gradient energy blocked by protection.

For $\rho_k=2^{-(k+1)}$, the impulse response of trace $k$ at age $\tau\ge1$ is
\[
h_k(\tau)=\rho_k(1-\rho_k)^{\tau-1}.
\]
For $0<\alpha\le1$, define the unnormalized scale-free policy kernel
\begin{equation}\label{eq:scale-kernel}
M_{\alpha,K}(\tau)=\sum_{k=0}^{K-1}\rho_k^{\alpha-1}h_k(\tau)
=\sum_{k=0}^{K-1}\rho_k^\alpha(1-\rho_k)^{\tau-1}.
\end{equation}
A policy rescales this known finite kernel into its declared bounded weight range.

\begin{theorem}[Logarithmic EMA bank approximates scale-free relevance]\label{thm:multiscale}
For every $0<\alpha\le1$ and $K\ge2$, there are constants $0<c_\alpha\le C_\alpha<\infty$, independent of $K$ and $\tau$, such that
\begin{equation}\label{eq:powerlaw}
c_\alpha\tau^{-\alpha}
\le M_{\alpha,K}(\tau)
\le C_\alpha\tau^{-\alpha},
\qquad 1\le\tau\le2^{K-2}.
\end{equation}
Consequently, $K=O(\log H)$ actual EMA traces approximate a power-law relevance profile over ages $1,\ldots,H$ with constant multiplicative distortion.
\end{theorem}

\begin{proof}
Choose $k_*$ so that $\rho_{k_*}\tau\in[1/4,1/2]$, which exists in the stated range. Since $\rho_{k_*}\le1/2$,
\[
(1-\rho_{k_*})^{\tau-1}\ge \exp(-2\rho_{k_*}(\tau-1))\ge e^{-1}.
\]
The $k_*$ term is therefore at least $e^{-1}4^{-\alpha}\tau^{-\alpha}$.

For $\tau=1$, the upper bound is the finite geometric sum $\sum_k\rho_k^\alpha$. For $\tau\ge2$, split the sum into $\rho_k\tau<1$ and $\rho_k\tau\ge1$. The first part is a geometric sum bounded by a constant times $\tau^{-\alpha}$. For the second, $\tau-1\ge\tau/2$ and $(1-\rho_k)^{\tau-1}\le e^{-\rho_k(\tau-1)}$; grouping dyadic values of $\rho_k\tau$ shows that the sum is at most
$\tau^{-\alpha}\sum_{m\ge0}2^{\alpha(m+1)}e^{-2^{m-1}}$, which is finite.
\end{proof}

\begin{theorem}[Every single exponential has polynomial worst-case distortion]\label{thm:single}
Let $H\ge3$ be odd and $f(\tau)=\tau^{-\alpha}$. If $g(\tau)=a q^{\tau-1}$, $a>0$, $0<q\le1$, satisfies
\[
C^{-1}f(\tau)\le g(\tau)\le Cf(\tau),\qquad 1\le\tau\le H,
\]
then
\begin{equation}\label{eq:single-lower}
C\ge 2^{-\alpha/2}H^{\alpha/4}.
\end{equation}
\end{theorem}

\begin{proof}
Let $m=(H+1)/2$. Exponential sequences satisfy $g(m)^2=g(1)g(H)$. The approximation inequalities imply
\[
C^2f(m)^2\ge g(m)^2\ge C^{-2}f(1)f(H),
\]
so $C^4\ge f(1)f(H)/f(m)^2=m^{2\alpha}H^{-\alpha}\ge H^\alpha/4^\alpha$.
\end{proof}

For each allocation policy $p\in\mathcal P$, let $Q_{p,t}$ be the projector it would select before the current importance is revealed. After the current signal, define its realized frontier cost and bounded expert loss
\begin{align}
\mathcal V_t(p)&=\tr((I-Q_{p,t})S_t^*)+\zeta_t c_{p,t}^{\rm block},\label{eq:frontier-cost}\\
\ell_t^{\rm pol}(p)&=\Lambda^{-1}\mathcal V_t(p)\in[0,1].\label{eq:policy-loss}
\end{align}
The design uses declared bounds $\norm{B_j}_F^2\le M_j$ and
\[
\Lambda=2\sum_{j=1}^{J_{\max}}M_j+\zeta_{\max},
\]
so the loss is bounded for every registry state when $\Lambda>0$. If $\Lambda=0$, every frontier cost vanishes and we set $\ell_t^{\rm pol}(p)=0$. The first term penalizes retention leakage; the second penalizes plastic gradient energy blocked by the protected subspace. Thus larger rank is not free and cannot dominate merely by freezing more coordinates.

\begin{theorem}[Prediction error to functional-leakage transfer]\label{thm:allocation-transfer}
Fix a rank $r$. Let $Q_{p,t}$ be a top-$r$ projector of $\widehat S_{p,t}$ and let $Q_t^*$ be a top-$r$ projector of $S_t^*$. Define
\[
a_{p,t}=\tr((I-Q_{p,t})S_t^*),\qquad
a_t^*=\tr((I-Q_t^*)S_t^*)=\sum_{i>r}\lambda_i(S_t^*).
\]
Then
\begin{equation}\label{eq:allocation-transfer}
0\le a_{p,t}-a_t^*
\le 2r\norm{\widehat S_{p,t}-S_t^*}_{\rm op}
\le 2r\sum_{j\in\Aa_t}
|\widehat u_{j,t}^{(p)}-w_{j,t}|\,\norm{B_j}_{\rm op}^2.
\end{equation}
Thus a metaplastic policy that predicts the realized retention weights accurately attains correspondingly near-oracle functional leakage. If its prediction is exact, its rank-$r$ allocation is oracle optimal for that round.
\end{theorem}

\begin{proof}
By Ky Fan's maximum principle, $Q_t^*$ maximizes $\tr(QS_t^*)$ and $Q_{p,t}$ maximizes $\tr(Q\widehat S_{p,t})$ over rank-$r$ orthogonal projectors. Hence
\begin{align*}
a_{p,t}-a_t^*
&=\tr(Q_t^*S_t^*)-\tr(Q_{p,t}S_t^*)\\
&\le |\tr(Q_t^*(S_t^*-\widehat S_{p,t}))|
   +|\tr(Q_{p,t}(S_t^*-\widehat S_{p,t}))|\\
&\le 2r\norm{S_t^*-\widehat S_{p,t}}_{\rm op}.
\end{align*}
The final inequality follows from the covariance definitions and the triangle inequality.
\end{proof}

\begin{theorem}[High-probability metaplastic allocation regret]\label{thm:hedge}
For a declared horizon $T\ge1$ and $P=|\mathcal P|\ge2$, initialize the policy weights uniformly and run exponential weights with learning rate $\eta_H=\sqrt{8\log(P)/T}$ and sample $p_t$ from the current weights before observing $\ell_t^{\rm pol}$. Conditional on the pre-draw history, the current policy-loss vector is independent of that draw. For $P=1$, the selected and comparator policy coincide, so the regret is exactly zero and both overhead terms may be omitted. For every $0<\delta<1$, with probability at least $1-\delta$,
\begin{equation}\label{eq:hedge}
\sum_{t=1}^T\ell_t^{\rm pol}(p_t)
\le \min_{p\in\mathcal P}\sum_{t=1}^T\ell_t^{\rm pol}(p)
+\sqrt{\frac{T\log P}{2}}
+\sqrt{\frac{T\log(1/\delta)}{2}}.
\end{equation}
Equivalently, multiplying by $\Lambda$,
\[
\sum_{t=1}^T \mathcal V_t(p_t)
\le \min_{p\in\mathcal P}\sum_{t=1}^T \mathcal V_t(p)
+\Lambda\sqrt{\frac{T\log P}{2}}
+\Lambda\sqrt{\frac{T\log(1/\delta)}{2}}.
\]
The selected frontier costs contain exactly the spectral residual entering the memory certificate and exactly the current gradient fraction blocked by protection. Hence the timescale/rank controller is coupled simultaneously to retention and compatible adaptation rather than analyzed separately.
\end{theorem}

\begin{proof}
The exponential-weights potential argument bounds the cumulative conditional mean loss by the best expert plus $\log P/\eta_H+\eta_HT/8$. The difference between sampled and conditional-mean loss is a bounded martingale difference; Azuma-Hoeffding adds the final term. Simplifying the chosen learning rate gives~\eqref{eq:hedge}.
\end{proof}

\begin{proposition}[Metaplastic atomization invariance]\label{prop:atomization}
The trace, policy, spectral-loss, allocation-transfer, and Hedge results remain valid if each record contribution $B_j^TB_j$ is decomposed into any finite PSD atomization
\[
B_j^TB_j=\sum_{a\in\mathcal I_j}M_a,
\qquad M_a\succeq0.
\]
Give atom $a$ the explicit realized signal
\[
\chi_{a,t}=\min\left\{1,
\frac{g_t^TM_ag_t}{1+\norm{g_t}^2}\right\},
\]
its own $K$ traces, and a policy weight, and form the predicted and realized covariances from the corresponding nonnegative weighted sums $\sum_a\widehat u_{a,t}A_tM_aA_t$ and $\sum_aw_{a,t}A_tM_aA_t$. The memory-certificate result also remains valid under atom-specific weights provided each atom carries a declared certified PSD sensitivity/error decomposition whose weighted sum upper-bounds the same conceptual whole-behavior derivative energy, Frequent-Directions error, and anchor-transport terms used in Theorem~\ref{thm:memory}. A purely algebraic split of $B_j^TB_j$ does not by itself create a stronger independently weighted derivative certificate; when no atom-specific certificate is supplied, the original record-level memory certificate is retained while atomization is used only by the allocation controller. In particular, certified atoms may be individual sketch eigen-directions, parameter groups, or rank-one individual synaptic sensitivities. Persistent state grows with the fixed number of retained atoms, not with elapsed lifetime.
\end{proposition}

\begin{proof}
For the controller results, every relevant argument uses only nonnegative weighted sums of PSD contributions, their trace residual under $I-Q_t$, operator-norm perturbations of those sums, and bounded expert losses. Replacing one PSD term by a finite certified sum therefore reproduces the trace, spectral-cost, allocation-transfer, and Hedge arguments verbatim after relabeling records by atoms. For the derivative-memory statement, Theorem~\ref{thm:memory} additionally uses a certified upper bound on conceptual derivative energy plus sketch and anchor-transport errors. If these quantities are decomposed atomwise with a weighted sum that upper-bounds the same whole-behavior quantities, the identical Minkowski and trace argument applies; otherwise the record-level certificate is left unchanged. This is precisely the stated condition.
\end{proof}

\subsection{Function-preserving structural renewal}
\label{app:structural-renewal}

Structural renewal is restricted to modules whose zero functional gate makes internal reset exactly function preserving. Any subsequent activation must pass the same protected transaction as ordinary learning.

A renewal batch consists of reserved modules $a_i\varphi_{w_i}$ whose gates initially satisfy $a_i=0$. The architecture applies the same gate to the module's contribution to the realized predictor and to every active protected behavior map. Equivalently, for every active $j$, $\Phi_j$ is independent of $w_i$ whenever $a_i=0$. This is the functional, rather than merely output-level, definition of a dormant renewable module.

\begin{theorem}[Exact reset and certified activation]\label{thm:renew-safe}
Resetting any internal module parameter while its functional gate is zero changes the realized predictor, current loss value, and every active protected behavior by exactly zero. After reset, recompute $A_t$, $\Pi_t$, $\varepsilon_t$, $e_t^{\rm pop}$, and the full current gradient $g_t^{\rm reset}$. Let
\[
q_t^{\rm reset}=\Pi_tg_t^{\rm reset},\qquad
r_t=M_tq_t^{\rm reset},
\]
where $M_t$ projects onto an eligible named trial slot satisfying the zero-gate invariant above. The ordinary AFM update in direction $-q_t^{\rm reset}$ satisfies the behavioral bound~\eqref{eq:retention-step} and loss decrease~\eqref{eq:descent-step}. Moreover, whenever that update is accepted with step length $s_t>0$ and $r_t\ne0$, its trial-slot displacement is
\[
M_td_t=-\frac{s_t}{\norm{q_t^{\rm reset}}}r_t\ne0,
\]
and, because the dormant internal nongate component is identically zero under the invariant, this nonzero slot displacement contains nonzero gate motion. The accepted gate motion is covered by the same certificate. If, additionally, its resulting gated module contribution is nonzero as a function, activation is functionally nonvacuous; nonzero gate motion alone does not imply this additional condition. If no such accepted motion occurs, rolling back the reset restores the exact pretrial parameter state.
\end{theorem}

\begin{proof}
By construction, a zero functional gate removes the module from both the realized predictor and every active behavior map, so reset has exactly zero immediate functional effect. The post-reset proposal is the ordinary global compatible direction in $\Range(\Pi_t)$, not a separately projected partial gradient; Theorem~\ref{thm:onestep} therefore applies directly after recomputing the certificates. Applying $M_t$ to the displayed update gives the trial-slot displacement formula. The dormant-slot invariant eliminates nongate internal motion while the gate is zero, so $r_t\ne0$ implies nonzero gate displacement. Multiplication by the reset module function gives a nonzero contribution only under the stated nondegeneracy condition; a reset module that is identically zero provides the immediate counterexample without it. Transactional rollback is exact because the pretrial slot state and zero gate are retained until acceptance.
\end{proof}

\begin{theorem}[Renewal sampling under actual conditional richness]\label{thm:renew-richness}
During stall episode $e$, use fresh conditional randomness at each trial. Let $I_{e,n}$ indicate a $\gamma$-useful batch and let the predictable conditional richness be
$\rho_{e,n}(\gamma)=\E[I_{e,n}\mid\Ff_{e,n-1}]$. For every $x>0$ and every $N$,
\begin{equation}\label{eq:no-renew}
\Prob\left(
I_{e,1}=\cdots=I_{e,N}=0,
\quad \sum_{n=1}^N\rho_{e,n}(\gamma)\ge x
\,\middle|\,\Ff_{e,0}
\right)\le e^{-x}.
\end{equation}
Consequently, whenever the realized cumulative conditional richness reaches $\log(1/\delta_e)$, failure beyond that point has probability at most $\delta_e$. No known lower bound is required by the algorithm. If every richness value is zero, trials under the declared sampling law cannot produce a $\gamma$-useful direction.
\end{theorem}

\begin{proof}
Define
\[
M_N=\mathbf 1\{I_{e,1}=\cdots=I_{e,N}=0\}
\exp\left(\sum_{n=1}^N\rho_{e,n}(\gamma)\right).
\]
On the event of no previous success,
\[
\E[(1-I_{e,N})e^{\rho_{e,N}}\mid\Ff_{e,N-1}]
=(1-\rho_{e,N})e^{\rho_{e,N}}\le1.
\]
Hence $(M_N)$ is a nonnegative supermartingale with $M_0=1$. Markov's inequality gives~\eqref{eq:no-renew} for each fixed $N$. Moreover, if $\tau_x=\inf\{N:\sum_{n=1}^N\rho_{e,n}(\gamma)\ge x\}$, then failure through $\tau_x$ implies $M_{\tau_x}\ge e^x$; Ville's inequality therefore gives $\Prob(\tau_x<\infty,\ I_{e,1}=\cdots=I_{e,\tau_x}=0\mid\Ff_{e,0})\le e^{-x}$, which proves the stated random-hitting-time consequence. The zero-richness statement follows from the definition.
\end{proof}

\begin{corollary}[Isotropic trial-slot compatible-gradient model]\label{cor:gaussian}
Condition on the composed linear map $P=M\Pi$ and suppose it is an orthogonal projector of rank $h$ on the trial coordinates. If the post-reset full gradient is $g^{\rm reset}=\omega z$ with $z\sim\mathcal N(0,I_d)$ independent of $P$, then
\begin{equation}\label{eq:chi}
\frac{\norm{Pg^{\rm reset}}^2}{\omega^2}\sim\chi^2_h.
\end{equation}
Hence $h>0$ and $\omega>0$ imply explicit positive richness. This corollary is optional; the universal theorem uses the actual richness~\eqref{eq:rho} and does not assume isotropy or commutation of $M$ and $\Pi$.
\end{corollary}

\begin{proof}
Under the stated projector condition, rotational invariance of the standard Gaussian makes its projection onto the $h$-dimensional trial-compatible subspace a standard $h$-variate Gaussian.
\end{proof}

A countable sequence of episodes uses summable $\delta_e$ and a fixed structural pool. Failed or unactivated trials reuse their zero-gated slots exactly. A successful activation consumes its slot until a separately certified structural move returns that module to an exactly zero-gated dormant state. If no dormant slot remains, the renewal family's realized richness is zero and the algorithm reports a structural-capacity obstruction and does not allocate additional parameters. Thus lifetime resource use remains fixed; computation and delay are the sums of the trial budgets, with the horizon-dependent precision stated in~\eqref{eq:word-resource}.

\subsection{Integrated frontier theorem and convex specialization}
\label{app:integrated-theorem}

The convex branch supplies a global dynamic-regret statement when an exact projection oracle is available. The integrated theorem then combines the pathwise, statistical, resource, and obstruction-aware components of AFM.

\subsubsection{Convex affine-behavior dynamic regret}

The nonlinear theorem above is pathwise and stationarity-based. A stronger global regret statement requires convexity.

\begin{assumption}[Convex affine branch]\label{ass:convex}
There is a compact convex set $\Kk\subset\R^d$ of diameter $D$. Every $\ell_t$ is convex and differentiable on a neighborhood of $\Kk$, with $\norm{\nabla\ell_t(\theta)}\le G$ for every $\theta\in\Kk$ (and hence is $G$-Lipschitz on $\Kk$). The exact retention-compatible set $\Cc_t\subseteq\Kk$ is nonempty, closed, and convex; this holds, for example, for affine behavior maps and convex norm tolerances. The initial point satisfies $\theta_1\in\Cc_1$, so every projected iterate is retention feasible.
\end{assumption}

For step size $\eta>0$, let $\Pi_{\Cc}$ denote Euclidean projection onto $\Cc$. AFM's exact-convex mode uses
\begin{equation}\label{eq:convex-update}
\theta_{t+1}=\Pi_{\Cc_{t+1}}(\theta_t-\eta g_t).
\end{equation}
Let
\[
\theta_t^*\in\argmin_{\theta\in\Kk}\ell_t(\theta),\qquad
\theta_t^\dagger\in\argmin_{\theta\in\Cc_t}\ell_t(\theta),
\]
and define
\[
\Gamma_t=\ell_t(\theta_t^\dagger)-\ell_t(\theta_t^*)\ge0,
\qquad P_T^\dagger=\sum_{t=1}^{T-1}\norm{\theta_{t+1}^\dagger-\theta_t^\dagger}.
\]
For the final projection inequality use the harmless convention
$\Cc_{T+1}=\Cc_T$ and $\theta_{T+1}^\dagger=\theta_T^\dagger$.

\begin{theorem}[Dynamic regret with unavoidable compatibility price]\label{thm:dynreg}
Under Assumption~\ref{ass:convex}, for every $\eta>0$,
\begin{equation}\label{eq:dynreg}
\sum_{t=1}^T\bigl[\ell_t(\theta_t)-\ell_t(\theta_t^*)\bigr]
\le
\sum_{t=1}^T\Gamma_t
+\frac{D^2}{2\eta}
+\left(\frac D\eta+G\right)P_T^\dagger
+\frac{\eta G^2T}{2}.
\end{equation}
The first term cannot be removed by any retention-feasible learner.
\end{theorem}

\begin{proof}
Convexity gives
$\ell_t(\theta_t)-\ell_t(\theta_t^\dagger)\le g_t^T(\theta_t-\theta_t^\dagger)$.
Projection nonexpansiveness with comparator $\theta_{t+1}^\dagger\in\Cc_{t+1}$ yields
\[
2\eta g_t^T(\theta_t-\theta_{t+1}^\dagger)
\le \norm{\theta_t-\theta_{t+1}^\dagger}^2
-\norm{\theta_{t+1}-\theta_{t+1}^\dagger}^2+\eta^2G^2.
\]
The diameter bound implies
\[
\norm{\theta_t-\theta_{t+1}^\dagger}^2
\le\norm{\theta_t-\theta_t^\dagger}^2
+2D\norm{\theta_{t+1}^\dagger-\theta_t^\dagger},
\]
and
$g_t^T(\theta_{t+1}^\dagger-\theta_t^\dagger)\le G\norm{\theta_{t+1}^\dagger-\theta_t^\dagger}$.
Sum and telescope. Finally add $\Gamma_t$.
\end{proof}

\subsubsection{Integrated frontier theorem}

\begin{theorem}[Maximal Adaptive Functional Metaplasticity Frontier]\label{thm:main}
Fix any finite horizon $T$ and any nonanticipating stream-generating law, with the convention that at each round the current data item and loss function are fixed before the controller's current private policy or renewal draw. Conditional on every realized stream and learner history, the deterministic retention, endpoint-emulation, stationarity, obstruction, and resource statements below hold pathwise whenever their declared certificates accept. For those clauses that invoke conditional-risk, null, iid population, policy-randomization, or renewal-richness assumptions, suppose the corresponding stated conditions hold under the stream-generating law. Run AFM with fixed structural budgets and summable statistical error allocations. Allocate the total failure budget $\delta$ across consolidation, outcome reopening, observable route splitting, optional reference calibration, optional population certificates, policy randomization, refresh or append events, and renewal episodes; for an unbounded lifetime, use any summable allocation across doubling epochs and event indices. AFM executes a nonzero step only when the pre-step certificates~\eqref{eq:directional-smooth}-\eqref{eq:behavior-curvature} and every invoked population certificate are valid; Proposition~\ref{prop:cert-construct} supplies them constructively for a broad finite-graph class; otherwise it sets the trust radius to zero. Then, jointly over the stochastic observations covered by the invoked statistical assumptions and AFM's private randomization, with probability at least $1-\delta$ all probabilistic conclusions below hold simultaneously, while the pathwise conclusions hold on every realized trajectory.

\begin{enumerate}[label=(\roman*)]
\item \textbf{Operational task-free routing, constructive safe transfer, and consolidation.} Routing is defined on every stream and exact under the non-oracular separation conditions of Theorem~\ref{thm:routing-positive}. Frozen-representation calibration is honest in the sense of Proposition~\ref{prop:signature-calibration}: inadequate calibration produces an explicit obstruction and no positive routing claim. A finite predeclared signature family retains lifetime false-alarm control by Proposition~\ref{prop:signature-mixture}. Dual-evidence route refinement has lifetime false-split probability at most its independently allocated signature budget by Proposition~\ref{prop:route-split}, even when predictive contradiction is real, and it never deletes source protection. Every committed record satisfies its first-crossing time-uniform average conditional risk certificate~\eqref{eq:commit-cert} and its predeclared activation-gap admission bound. Post-certification staleness is monitored with the false-cancellation control of Proposition~\ref{prop:frozen-certificate}. Every accepted shielded candidate-transfer step satisfies Proposition~\ref{prop:safe-transfer}: it reproduces the full ordinary current endpoint, exactly preserves active empirical behaviors and unselected candidates, and completes the selected empirical transfer in one service under $\xi=1$. Theorem~\ref{thm:priority-oracle} gives the declared deterministic construction for consistent finite constraints and a sharp functional obstruction for contradictory duplicate features. Proposition~\ref{prop:fair-service} is predictable and non-starving, and Theorem~\ref{thm:transfer-completion} gives the displayed finite service bound without a projected-compatibility assumption. No uncommitted candidate receives a retention claim.

\item \textbf{Whole-behavior retention.} On every accepted protected round the persistent safe-base leg obeys the budget-controlled fixed-shield bound~\eqref{eq:retention-step}. The complete declared finite deployed evidence is stronger: it is restored exactly in the mathematical construction, so its active-interval drift is zero and its snapshot discrepancy is bounded by the behavior-space activation gap $a_{j,\Phi}^{\rm act}$. A finite-precision realization with nonzero certified endpoint allowances additionally pays $\mathcal E_{j,c_j:T}$ from Corollary~\ref{cor:interval}. Any broader deployed or evaluator map carries the explicit internal structural-shield charge $\omega_t^S$, the structural routing-mismatch charge $\omega_t^{\Delta}$, and the fixed-shield routing derivative terms in~\eqref{eq:external-semantic-step}-\eqref{eq:semantic-interval}. Proposition~\ref{prop:release} separates activation, exact active finite retention, and post-release drift.

\item \textbf{Arbitrary non-convex adaptation.} The persistent metaplastic base satisfies the exact pathwise projected-stationarity inequality~\eqref{eq:stationarity}. In counterfactual-normalized mode, every accepted protected round additionally stores at least the persistent current-loss fraction~\eqref{eq:persistent-ratio}, with the compatible-gradient fraction $\kappa_t$ displaying the explicit local compatibility price for the selected protection geometry; $\eta_t=1$ recovers the full projected comparator. Independently, every accepted deployed update reproduces one hundred percent of the genuine no-protection current endpoint by~\eqref{eq:genuine-progress-ratio}. On tails satisfying Corollary~\ref{cor:plasticity}, finite cumulative certified safe-base behavior charge, exact finite deployed protection, and realized lifelong compatible base plasticity coexist; a broader joint map additionally pays its explicit structural-shield charge.

\item \textbf{Metaplastic timescale optimality within the declared family.} Whenever the declared finite family contains the normalized scale-free coefficient choices above, the logarithmic EMA bank contains constant-distortion scale-free policies by Theorem~\ref{thm:multiscale}, whereas one exponential has polynomial distortion by Theorem~\ref{thm:single}. The deterministic transfer bound~\eqref{eq:allocation-transfer} converts temporal prediction error into excess leakage, while~\eqref{eq:hedge} competes on the realized composite frontier cost~\eqref{eq:frontier-cost}, including the gradient energy blocked by protection.

\item \textbf{Evidence-based reopening and bounded route refinement.} False outcome reopening and false observable splitting are controlled by separate lifetime allocations.  Under~\eqref{eq:null}, false reopening is bounded by $\alpha_j^{\rm open}$; under~\eqref{eq:signature-null}, false splitting is bounded by $\alpha_j^{\rm split}$ even when a real semantic conflict violates the outcome null.  Theorem~\ref{thm:delay} gives a history-dependent reopening delay under positive outcome information, with the fresh-process bound as a special case. Theorem~\ref{thm:split-delay} applies at eligible evaluations with sufficient outcome and distinct-block observations, positive drift, and satisfied effect/capacity gates; eligibility must persist through that evaluation.  With semantic-only change, ordinary reopening or release remains valid but no observable route separation is claimed.

\item \textbf{Structural renewal.} Every reset is exactly preserving for the realized predictor and all active protected behaviors. The probability and delay of finding nonzero trial-slot motion inside the global compatible descent direction are controlled by the realized richness sequence through~\eqref{eq:no-renew}. Failed trials reuse fixed slots; successful activations consume them until separately certified exact dormancy is restored. Exhausted structural capacity and zero richness are reported as obstructions, not replaced by assumptions.

\item \textbf{Bounded structural resources and explicit precision.} Let $d_{\rm sig}$ bound the total observable-signature dimension and $N_{\rm val}$ the maximum candidate validation count. The persistent state admits the bound
\begin{equation}\label{eq:word-resource}
\begin{aligned}
W={}&O(d)+O(L_{\max}d)+O(C_{\max}\ell_{\rm cand}d)
+O((J_{\max}+C_{\max})d)\\
&+O(C_{\max}N_{\rm val}d_{\rm sig})+O(C_{\max}m_{\rm fit}d_{\rm ref})\\
&
+O\!\left((J_{\max}+C_{\max})n_{\max}(d_{\rm ref}+d_{\Phi})\right)\\
&+O(KA_{\max})+O(PK)+O(J_{\max}+C_{\max})\\
&+O((1+J_{\max}+C_{\max})N_{\rm shield}(d_Z+d_\Phi+2))\\
&+O(Q_{\max}d_Z)+O(\text{renewal slots}),
\end{aligned}
\end{equation}
independent of elapsed lifetime. Here $L_{\max}$ bounds the total retained active-sketch rows; $\ell_{\rm cand}$ bounds staging-sketch rows per candidate; $J_{\max}$ and $C_{\max}$ bound active records and candidate/context slots; $A_{\max}$ bounds retained sensitivity atoms; $K$ and $P$ count trace timescales and allocation policies. The dimensions $d_{\rm ref}$ and $d_\Phi$ bound a stored observation reference and its output vector, respectively. Here $B_{\max}$ is the declared current-minibatch cardinality bound and $N_{\rm shield}\le B_{\max}+J_{\max}n_{\max}+C_{\max}n_{\max}$ is the fixed cardinal-node cap and $Q_{\max}$ is the fixed label-free guard-address cap. The $C_{\max}$ terms include one private candidate initialization or frozen snapshot, the finite $m_{\rm fit}$ fitting-reference block before freezing, bounded staging sketches, frozen first-crossing validation sufficient statistics and evidence, fixed candidate-transfer ledger and queue state, and fixed sufficient state for the separately allocated staleness, outcome, and distinct-block signature e-processes. The shield term includes the deployed residual and frozen copies retained by candidates, records and shadow challengers. The fitting block is deleted from decision state after the frozen snapshot is produced. The empirical endpoint-check mode stores at most $n_{\max}$ fixed-size observation references and behavior outputs per candidate or active segment; a certified population-map mode may omit the active-segment evidence term after commitment. Referenced observations belong to the external input source, not to learner-owned adaptive state. Likewise, an optional append-only execution log is external output never read by the algorithm and is not counted as persistent decision state; retaining such a log is an experimental reproducibility choice. Through horizon $T$, integer counters and error-allocation indices require $O(\log T+\log(1/\delta))$ bits under the declared allocation schedule. Model entries, sketches and log-wealth also require declared numerical precision; certified rounding errors enter the numerical certificates. A fixed-bit implementation therefore declares a maximum certified horizon or an explicit restart/precision policy. New anchors are append-only while active; when the registry is full, any deletion is an explicit capacity release whose guarantee ends and whose heavy segment state is removed.

\item \textbf{Convex global performance.} Under Assumption~\ref{ass:convex}, an admissible specification that supplies a deterministic exact Euclidean projection oracle for each closed convex retention set instantiates the exact-convex safe-base mode and satisfies~\eqref{eq:dynreg}. The compatibility price $\sum_t\Gamma_t$ is unavoidable. Oracle failure or uncertified feasibility is an explicit convex-oracle obstruction; no uncertified approximate solver is substituted.
\end{enumerate}

The theorem is universal in the following precise sense: it does not require routing mismatch, spectral tails, anchor drift, curvature, conflict, variation, test delay, or renewal delay to be small. When they are small, the displayed inequalities yield strong continual learning. When they are large, the lower bounds below show why no bounded learner can generally remove them.
\end{theorem}

\begin{proof}
Item (i) combines the routing algorithm, Theorem~\ref{thm:routing-positive}, Propositions~\ref{prop:signature-calibration},~\ref{prop:signature-mixture},~\ref{prop:route-split},~\ref{prop:frozen-certificate},~\ref{prop:fair-service}, and~\ref{prop:safe-transfer}, and Theorems~\ref{thm:commit} and~\ref{thm:transfer-completion}, with a union bound over summably allocated records and statistical processes.  The optional finite initialization prefix supplies the arbitrary initial state; its bounded references are replayed through the final frozen representation before signature calibration. Item (ii) follows from Theorems~\ref{thm:memory},~\ref{thm:priority-oracle}, and~\ref{thm:onestep}, Corollary~\ref{cor:interval}, and Proposition~\ref{prop:release}. Item (iii) combines Theorems~\ref{thm:normalized-assimilation},~\ref{thm:priority-oracle},~\ref{thm:onestep}, and~\ref{thm:stationarity} with Corollary~\ref{cor:plasticity}. Item (iv) intersects Theorems~\ref{thm:multiscale},~\ref{thm:single},~\ref{thm:allocation-transfer}, and~\ref{thm:hedge}. Item (v) is Proposition~\ref{prop:route-split} and Theorems~\ref{thm:falseopen},~\ref{thm:delay}, and~\ref{thm:split-delay}. Item (vi) is Theorems~\ref{thm:renew-safe} and~\ref{thm:renew-richness}. Item (vii) follows by counting fixed arrays, bounded empirical evidence, the declared $A_{\max}$ trace bank, and sketches; Frequent Directions does not store the conceptual row stacks. The counters and allocation indices have the stated logarithmic bit length; real-valued quantities additionally use the declared certified precision. Item (viii) is Theorem~\ref{thm:dynreg} under the supplied exact projection-oracle condition. The deterministic implications in items (i)-(viii) are pathwise once their premises and certificates are fixed. The union of all invoked statistical, policy-randomization, and renewal failure events has probability at most $\delta$ under the stated stream-generating law by the summable allocation.
\end{proof}

\subsection{Obstructions and maximality}
\label{app:obstructions}

The main theorem does not assume away incompatibility, observational ambiguity, finite memory, or finite structural capacity. The following constructions show why the corresponding terms or abstentions are necessary.

\subsubsection{Matching obstructions and maximality}

\begin{proposition}[Routing unidentifiability]
If two semantic contexts induce the same law for every observable context signature and outcome available to the router, but require different retained behaviors, then no task-free router can identify the correct context with error probability below $1/2$ under an equal prior.  If the complete signature histories have identical laws under two route hypotheses, any test based only on those signatures has equal rejection probability under the two hypotheses. The same conclusion holds after observing outcomes only when the conditional signature-history laws remain identical given that information. Outcome evidence may justify reopening without identifying a distinct observable route. Consequently, a nonzero routing-mismatch term or outcome-only release is unavoidable.
\end{proposition}

\begin{proof}
Identical observation laws have zero total variation distance, so binary classification under an equal prior cannot outperform random guessing. A signature-only test is a measurable function of the signature history and private randomness; identical laws give identical rejection probabilities. Applying the same argument conditionally requires equality of the corresponding conditional laws, not merely equality of signature marginals.
\end{proof}

\begin{proposition}[No distribution-free future-risk certificate from a finite commit history]\label{prop:future-risk}
In binary classification with zero--one loss, every consolidation rule that commits after a finite observed history admits two data streams that are identical through the commit time but have different future outcome laws, one on which the committed predictor remains valid and one on which its future risk is as poor as permitted by the bounded loss range. Therefore Theorem~\ref{thm:commit} can certify its observed conditional evidence sequence, while a future-risk guarantee requires an explicit recurrence, exchangeability, drift, or other predictive assumption.
\end{proposition}

\begin{proof}
Fix the realized finite history that triggers commitment. Extend it in one world with outcomes predicted perfectly by the snapshot and in the other with outcomes chosen to maximize its bounded loss. The learner has identical information at commitment in both worlds.
\end{proof}

\begin{proposition}[Curvature obstruction]
For every linear operator $J$ with nontrivial nullspace and every $H>0$, there is a smooth behavior map into an enlarged Hilbert space whose derivative at zero has first component $J$ and whose change along every $d\in\Null(J)$ equals $H\norm d^2/2$ in an orthogonal output coordinate.
\end{proposition}

\begin{proof}
Use $\widetilde\Phi(\theta)=(J\theta,(H/2)\norm\theta^2)$.
\end{proof}

\begin{proposition}[Anchor-drift obstruction]\label{prop:anchor-lower}
A certificate based only on the derivative at a frozen anchor cannot generally control current leakage without a term proportional to representation drift. Specifically, for every $L>0$ and displacement $a$, the one-dimensional map
\[
\Phi(\theta)=\frac{L}{2}\theta^2
\]
has $D\Phi(0)=0$ but $|D\Phi(a)|=L|a|$. Hence any universally valid anchor-based derivative certificate must charge an order-$L\norm{\theta-\bar\theta}$ term or refresh the anchor.
\end{proposition}

\begin{proof}
Differentiate the displayed map. At the anchor the sensitivity is zero, while at displacement $a$ it is $La$. This saturates the Lipschitz-Jacobian transport order used in~\eqref{eq:anchor-drift}.
\end{proof}

\begin{proposition}[Compatibility lower bound]
Every learner constrained to $\Cc_t$ has regret relative to the unconstrained optimum at least $\Gamma_t$ on round $t$. Hence $\sum_t\Gamma_t$ in~\eqref{eq:dynreg} cannot be removed.
\end{proposition}

\begin{proof}
By definition, $\ell_t(\theta)\ge\ell_t(\theta_t^\dagger)=\ell_t(\theta_t^*)+\Gamma_t$ for every $\theta\in\Cc_t$.
\end{proof}

\begin{proposition}[Reopening information lower bound]
Let validity and obsolescence generate iid laws $P_0,P_1$ with $I=D_{\rm KL}(P_1\|P_0)$, where $0<\alpha<1-\beta<1$. If a rule stops by time $N$ with probability at least $1-\beta$ under $P_1$ and at most $\alpha$ under $P_0$, then
\[
NI\ge {\rm kl}(1-\beta,\alpha).
\]
If $I=0$, reliable finite-delay reopening is impossible.
\end{proposition}

\begin{proof}
Apply data processing for KL divergence to the event that the rule stops by time $N$.
\end{proof}

\begin{proposition}[Renewal zero-richness obstruction]
If $\rho_{e,n}(\gamma)=0$ for every trial, no number of samples from that renewal family can produce a $\gamma$-useful compatible direction.
\end{proposition}

\begin{proof}
The union of countably many probability-zero success events has probability zero.
\end{proof}

\begin{proposition}[Finite-memory obstruction]
A learner with at most $B$ persistent bits cannot exactly recall every sequence of $T>B$ independent unbiased bits.
\end{proposition}

\begin{proof}
There are $2^T$ histories and at most $2^B$ persistent states, so two histories collide and require different answers to some recall query.
\end{proof}

\begin{proposition}[Finite-state lifetime-testing obstruction]\label{prop:finite-test}
Consider a time-homogeneous detector with finitely many nonalarm states under an iid null on a finite alphabet with full support. Suppose that from every nonalarm state there exists an observation word that triggers alarm. Then the detector eventually false alarms with probability one. Consequently, a truly finite-state detector cannot have both false-alarm probability below one over an unbounded lifetime and nonzero eventual sensitivity from every state; exact anytime guarantees require growing precision, an externally supplied clock/epoch policy, or sacrificed delayed sensitivity.
\end{proposition}

\begin{proof}
Because the nonalarm state set is finite, choose for each state a triggering word and let $L$ be their maximum length. Full support and finiteness give a uniform lower bound $p>0$ on the probability of the selected word from any state. Conditional on survival at the start of each successive block of length $L$, alarm occurs within the block with probability at least $p$. Hence survival through $m$ blocks is at most $(1-p)^m\to0$.
\end{proof}

The operational scope of consolidation is forced by Proposition~\ref{prop:future-risk}; the anchor-drift order is forced by Proposition~\ref{prop:anchor-lower}; the exact spectral lower bound is Theorem~\ref{thm:spectral}; the single-timescale obstruction is Theorem~\ref{thm:single}; dynamic-regret lower bounds of matching path-variation order are known in online convex optimization~\citep{yang2016,zhang2018}.

\begin{corollary}[Componentwise maximality of the universal claim]\label{cor:maximality}
No bounded task-free learner can, over every realized stream, simultaneously guarantee all of the following with the corresponding obstruction term deleted: always-correct semantic routing; future-valid consolidation from an arbitrary finite past; zero finite-rank functional leakage uniformly over all unit plastic directions; zero nonlinear or anchor-drift damage; exact adaptation to directly incompatible objectives; finite-delay reopening at zero observational information; successful renewal at zero compatible richness; exact recall of unbounded independent information with fixed bounded-bit persistent state; and exact infinite-horizon testing with both finite state and sensitivity after arbitrary delay. Therefore every universal theorem retaining all AFM mechanisms must either display these terms, abstain, release information, or impose assumptions that make them small.
\end{corollary}

\begin{proof}
Apply, respectively, routing unidentifiability, Proposition~\ref{prop:future-risk}, Theorem~\ref{thm:spectral}, the curvature and anchor-drift obstructions, the compatibility lower bound, the reopening information lower bound, the zero-richness obstruction, the finite-memory counting argument, and Proposition~\ref{prop:finite-test}. Each construction is a member of the universal stream class, so deleting any associated term falsifies the corresponding universal guarantee.
\end{proof}

These results establish maximality of the theorem's structure and quantifiers, not sharpness of every numerical constant or minimax optimality for every restricted subclass.

\subsubsection{Guarantee and obstruction summary}
The obstructions separate three requirements for continual learning: information sufficient to identify and validate retained behavior; geometry permitting compatible descent; and capacity sufficient to store the retained constraints. AFM accounts for each requirement through observable tests, derivative and endpoint certificates, and explicit resource bounds. The following experiments vary compatibility while controlling the current state, unrestricted learning opportunity and retention allowance.

\section{Controlled causal compatibility experiment}
\subsection{Compatibility as a causal variable}

We analyzed 750 admitted fully unfrozen pre-update states from three systems: a CIFAR-10 convolutional network, a CIFAR-10 vision transformer, and a character-level transformer trained on WikiText-2 \citep{krizhevsky2009cifar,dosovitskiy2021vit,vaswani2017attention,merity2017pointer}. Each intervention at a state began from the same complete parent checkpoint, including model parameters, optimizer state, stream position, replay reservoir, and random-number-generator state. We then generated learning signals spanning six requested compatibility levels, from the minimum attainable boundary to near-complete compatibility, using a generalized function-space eigenproblem. Realized compatibility was recomputed and used in all analyses.

Changing direction can also change the amount of finite learning available without protection. To isolate compatibility from this opportunity, we matched the genuine unrestricted decrease $\DeltaZero$ across compatibility conditions within each state. The mean within-state spread in $\DeltaZero$ was 0.27--0.28\% and never exceeded 0.40\% (Table~\ref{tab:ed-fidelity}). Each condition was then evaluated under the same absolute retention budget. The inferential unit was the seed trajectory: slopes were estimated within neural state across compatibility levels, averaged over 50 states within seed, and uncertainty was computed across independent seed effects rather than across the much larger number of method-level rows.

Compatibility causally changed persistent learning (Figure~\ref{fig:causal}). Writing $\beta$ for the common protected-logit drift allowance relative to unrestricted drift (Equation~\eqref{causal:eq:budget}), at $\beta=0.5$, the slope of the persistent-progress ratio $\rhoPers=\DeltaPers/\DeltaZero$ on realized $\kappaF$ for the compatible-projection branch was 1.014 in the convolutional network, 1.001 in the vision transformer, and 0.965 in the text transformer. Every seed-level slope was positive. Thus, for this branch, persistent progress increased approximately in proportion to realized compatibility. The effect was not a generic consequence of changing direction: unrestricted, distillation, and EWC branches converted compatibility less efficiently, whereas replay and DER++ showed little seed-resolved slope under the same matched test \citep{buzzega2020der}. Nor did compatibility act independently of the retention constraint. At the numerical-floor setting $\beta=0$, the observed slope was zero; as the budget relaxed, the projected slope became strongly positive, approaching one in all systems, more gradually in text (Table~\ref{tab:ed-budget-slopes}). Compatibility therefore describes available persistent-learning geometry, while the learning rule and retention budget determine how much of that geometry is actually used.

\begin{figure}[tbp]
\centering

\begin{minipage}{0.90\textwidth}
\raggedright
{\bfseries\large a}\par
\vspace{0.05em}
\centering
\includegraphics[width=\linewidth]{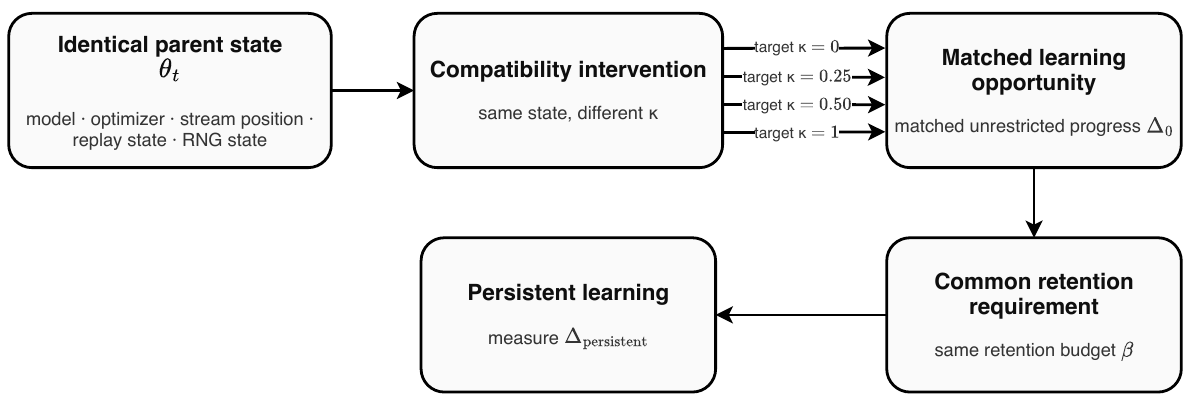}
\end{minipage}

\vspace{0.30em}

\begin{minipage}[t]{0.47\textwidth}
\raggedright
{\bfseries\large b}\par
\vspace{0.05em}
\centering
\includegraphics[width=\linewidth]{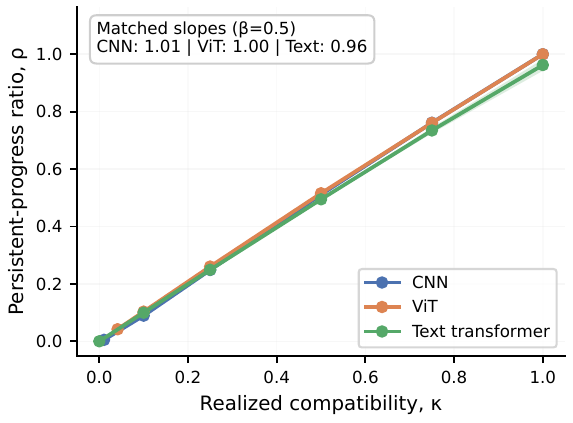}
\end{minipage}
\hfill
\begin{minipage}[t]{0.44\textwidth}
\raggedright
{\bfseries\large c}\par
\vspace{0.05em}
\centering
\includegraphics[width=\linewidth]{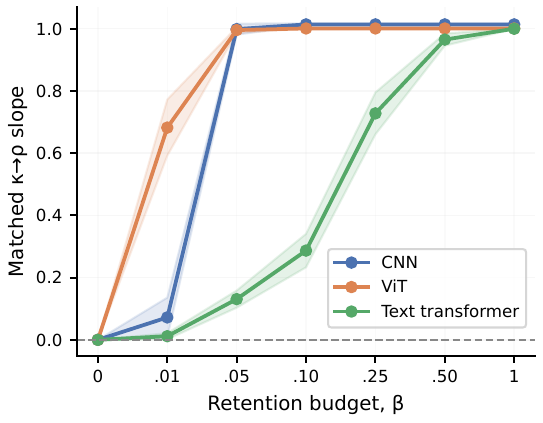}
\end{minipage}

\vspace{0.30em}

\begin{minipage}{0.64\textwidth}
\raggedright
{\bfseries\large d}\par
\vspace{0.05em}
\centering
\includegraphics[width=\linewidth]{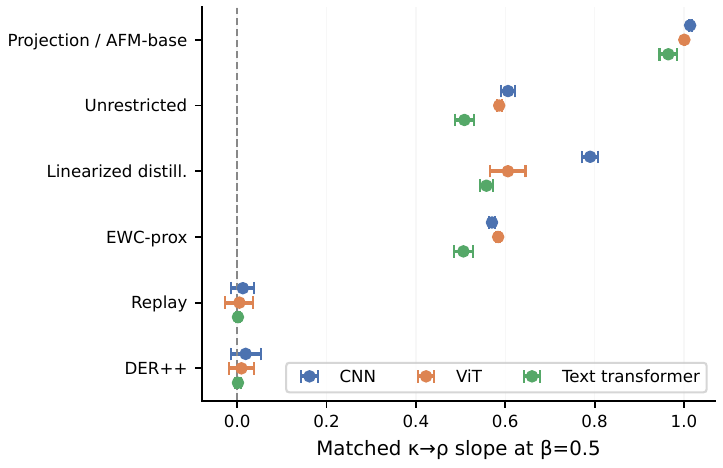}
\end{minipage}

\caption{\textbf{Functional compatibility causally controls persistent learning.}
\textbf{a}, Experimental logic. All interventions at a causal state start from the identical parent neural state, including the same model, optimizer, stream position, replay state and random-number-generator state. Functional compatibility is manipulated while unrestricted learning progress $\Delta_0$ is matched, after which the same retention requirement is applied and persistent learning is measured.
\textbf{b}, Persistent-progress ratio versus realized functional compatibility at $\beta=0.5$ in three fully unfrozen systems.
\textbf{c}, Compatibility slope as a function of retention budget.
\textbf{d}, Matched compatibility slopes for distinct learning rules at $\beta=0.5$.
For \textbf{b--d}, $n=5$ independent seed trajectories per system and 50 admitted causal states per seed; state-level slopes are aggregated within seed, and shaded bands and error bars show two-sided 95\% Student-$t$ intervals across the five seed effects.}
\label{fig:causal}

\end{figure}

\subsection{Matched intervention and observational validation}
The controlled comparison holds the neural state, protected evidence, current inputs and unrestricted finite decrease fixed while varying compatibility. Its estimand is the resulting change in persistent learning under a common retention criterion. Natural-state validation complements this intervention by measuring compatibility during ordinary supervised learning without targeted teacher signals. These designs answer different questions: one identifies the local intervention effect, while the other tests whether the geometry and protected transaction remain observable on unmodified streams. Both use reconstructed states from the evaluated systems rather than an independent laboratory replication.

\subsection{Experimental systems and state reconstruction}

Table~\ref{causal:tab:systems} specifies the three fully trainable systems.
\begin{table}[ht]
\centering
\caption{Experimental matrix. All neural representations remain unfrozen during the reported causal and natural-state experiments.}
\label{causal:tab:systems}
\begin{tabular}{lllll}
\toprule
System & Modality & Seeds & Causal states/seed & Natural states/seed \\
\midrule
CIFAR-10 CNN & vision & 11, 29, 47, 71, 101 & 50 & 50 \\
CIFAR-10 ViT & vision & 11, 29, 47, 71, 101 & 50 & 50 \\
Text transformer & text & 11, 29, 47, 71, 101 & 50 & 50 \\
\bottomrule
\end{tabular}
\end{table}

The CNN has width 32. The baseline ViT uses $32\times32$ inputs, $4\times4$ patches, embedding dimension 48, depth 3 and four attention heads. The character transformer has context length 64, dimension 64, depth 3 and four heads. Parent training uses batch size 64 and weight decay $10^{-4}$; learning rates are $10^{-3}$ for the CNN and $7\times10^{-4}$ for the transformers. The pre-probe histories contain 4,000+500 CNN updates and 5,000+500 transformer updates. The complete configurations are provided in the public repository.

The two vision systems use CIFAR-10 \citep{krizhevsky2009cifar}; the vision-transformer system follows the Vision Transformer architecture \citep{dosovitskiy2021vit}. The text system uses the WikiText-2 training corpus \citep{merity2017pointer} obtained from Salesforce's published dataset distribution and is evaluated as a character-level transformer \citep{vaswani2017attention}.

Every intervention restores the complete parent checkpoint specified above, preserving the same parameters and stored history across branches.

The scheduled probe stream generates candidate parent states. Probing occurs at every parent update for at most 5,000 candidate steps. A candidate enters the matched causal analysis only if the complete six-level intervention is jointly admissible: the construction-eigenvalue span is at least 0.5, every requested level admits a positive unrestricted endpoint, and the coefficient of variation of matched comparator decreases is at most 0.15. The bisection criterion is specified below. If any of these checks fails, that candidate is skipped and probing continues until 50 admitted states are obtained per seed or the probe horizon is exhausted. Thus the reported causal effect is estimated on states for which the declared matched intervention is feasible; the intervention itself and the causal interpretation of the within-state comparison are unchanged.

The design contains 750 admitted causal states:
\[
3\ \text{systems}\times 5\ \text{seeds}\times 50\ \text{states}=750.
\]
Each state is evaluated under six requested compatibility conditions and seven method branches:
\[
750\times 6\times 7=31{,}500
\]
method-level causal outcomes. These 31,500 rows are repeated measurements nested within 750 admitted causal states and 15 system-seed trajectories. They must not be treated as 31,500 statistically independent experimental units.

\subsection{Controlled causal compatibility intervention}

\subsubsection{Compatibility construction}

Let $X_c$ be the fixed current input batch and $h_\theta(X_c)$ its functional-output array. At parent state $\theta_0$, write $h_0=h_{\theta_0}(X_c)$ and $J_c=D_\theta\operatorname{vec}h_\theta(X_c)|_{\theta_0}$. A same-shaped residual $r$ defines frozen targets $y^*=h_0-r$ and the summed teacher loss, using the Frobenius norm and implicit vectorization in Jacobian products,
\begin{equation}\label{eq:causal-teacher-loss}
\mathcal L_r(\theta)=\tfrac12\norm{h_\theta(X_c)-y^*}_F^2,
\qquad \nabla\mathcal L_r(\theta_0)=J_c^\top r.
\end{equation}
There is no averaging over examples or output coordinates. The residual is scaled to target gradient norm one, the native-precision gradient is recomputed and rescaled once, and realized compatibility is measured from that gradient. Endpoint loss uses float64 accumulation with native-precision model execution. Thus the intervention changes a realizable function-space learning signal.

Both Jacobians differentiate flattened functional outputs with respect to every trainable parameter. Vision uses all ten logits on 16 protected and 16 current inputs; text uses a fixed orthonormal projection to 16 output coordinates on 12 inputs of each kind. Half of each current batch is generated near protected inputs. Retention checks use the full logits, including text coordinates outside the functional projection. The ridge is $\gamma=10^{-7}$. Let $J_p$ denote the protected Jacobian, $J_c$ the current functional Jacobian, and $q=J_c^\top r$ the parameter gradient associated with residual $r$. The experiment uses the ridge-compatible symmetric contraction
\begin{equation}
R_\gamma=I-J_p^\top(J_pJ_p^\top+\gamma I)^{-1}J_p.
\end{equation}
For $\gamma>0$, $R_\gamma$ is generally not an orthogonal projector and is therefore distinct from the analytical projector $\Pi_t$ used in the theory. The residual construction uses the generalized eigenproblem
\begin{equation}
J_c R_\gamma J_c^\top r
=
\lambda J_cJ_c^\top r,
\label{causal:eq:gen-eigen}
\end{equation}
whose Rayleigh coordinate is $q^\top R_\gamma q/\|q\|^2$. The generalized solve retains current-Gram eigenvalues above $10^{-10}\max(\lambda_{\max}(J_cJ_c^\top),1)$. It mixes normalized extreme modes with square-root weights $(\sqrt{1-u},\sqrt u)$, where $u$ places the requested coordinate between their eigenvalues, to target
\[
\kap_{\mathrm{request}}\in\{0,\ 0.1,\ 0.25,\ 0.5,\ 0.75,\ 1\}.
\]
The actual compatibility used in the analyses is recomputed after construction:
\begin{equation}
\kap
=
\frac{\|R_\gamma q\|^2}{\|q\|^2}
=
\frac{q^\top R_\gamma^2q}{\|q\|^2}.
\label{causal:eq:kappa}
\end{equation}
Because $R_\gamma^2\ne R_\gamma$ in general, the construction coordinate and the reported realized compatibility are distinct ridge-regularized quantities. All statistical analyses use the recomputed realized compatibility rather than assuming that the requested value or the generalized eigenvalue was achieved exactly.

The directions share the state, protected evidence and current inputs.

\subsubsection{Matching the genuine unrestricted decrease}

Matching unrestricted finite decrease across the six directions controls differences in available learning opportunity.

For each causal state, the common target decrease was set to 50\% of the smallest already-positive unrestricted decrease across the six compatibility conditions. The initial comparator step is 0.01 for the CNN and 0.008 for the transformers, with at most 20 halvings to obtain positive decrease. Bisection then uses at most 40 iterations and relative target tolerance 0.002 along each original direction, without rotation. With $\mathcal L_{\rm current}=\mathcal L_r$ for each intervention, $\theta_{\rm pre}=\theta_0$, and $\theta_{\rm unrestricted}$ its accepted comparator endpoint, the matched denominator is
\begin{equation}
\DeltaZero
=
\mathcal{L}_{\mathrm{current}}(\theta_{\mathrm{pre}})
-
\mathcal{L}_{\mathrm{current}}(\theta_{\mathrm{unrestricted}}).
\label{causal:eq:delta0}
\end{equation}

Gradient-norm matching is imposed by the construction above. The exported causal summaries permit direct verification of finite-$\DeltaZero$ matching but do not retain the causal gradient norms.

\subsubsection{Common retention budget}

For protected inputs $X_p$, retention is the maximum absolute change in full raw logits,
\[
D(\theta')=\max_{x\in X_p,k}|h_{\theta'}^{\rm full}(x)_k-h_{\theta_0}^{\rm full}(x)_k|.
\]
RMS logit drift is a separate diagnostic. The common reference is
\begin{equation}
D_{\mathrm{ref}}
=
\max_{\kap}D_{\mathrm{unrestricted}}(\kap).
\label{causal:eq:dref}
\end{equation}
Every method and every compatibility condition at that state is evaluated under the same absolute budget
\begin{equation}
D
\le
\max\left(10^{-8},\beta D_{\mathrm{ref}}\right),
\label{causal:eq:budget}
\end{equation}
with
\[
\beta\in\{0,\ 0.01,\ 0.05,\ 0.1,\ 0.25,\ 0.5,\ 1\}.
\]
Thus \(\beta\) controls retention strictness independently of the requested compatibility level.

For each proposal $d$, the implementation evaluates $\theta_0+sd$ at every $s\in\{0,1/32,\ldots,1\}$ and selects the largest scale satisfying~\eqref{causal:eq:budget}, with $10^{-12}$ comparison slack. This rule accommodates nonmonotone drift and defines the reported frontier endpoint; it does not maximize loss decrease over the grid. The $\beta=0$ condition retains the $10^{-8}$ numerical floor.

\subsubsection{Proposal mechanisms}\label{sec:causal-proposals}
Let $g$ be the manipulated teacher gradient, $g_{\rm CE}$ the protected mean-cross-entropy gradient, and $g_{\rm MSE}$ the gradient of half the summed squared discrepancy from stored protected logits. With matched comparator step $\alpha$, the generic proposals are
\[
\begin{array}{ll}
\text{unrestricted:}&-\alpha g,\\
\text{projection/AFM base:}&-\alpha R_\gamma g,\\
\text{replay:}&-\alpha(g+g_{\rm CE}),\\
\text{linearized distillation:}&-\alpha(I+25\alpha J_p^\top J_p)^{-1}g,\\
\text{EWC-prox:}&-\alpha g\oslash(1+25\alpha g_{\rm CE}^{\odot2}),\\
\text{DER++ probe:}&-\alpha(g+0.5g_{\rm MSE}+0.5g_{\rm CE}).
\end{array}
\]
Here $\oslash$ and $\odot2$ denote coordinatewise division and squaring. The EWC diagonal is a squared batch-mean-gradient proxy. The causal runner captures stored replay logits at the parent state, so $g_{\rm MSE}=0$ at that origin. These equations specify the tested local mechanisms, including their differences from historical replay trajectories \citep{rolnick2019replay,kirkpatrick2017ewc,buzzega2020der}. Projection and the generic AFM-compatible proposal are identical; seven recorded labels therefore correspond to six distinct frontier families.

Native AFM is a separate probe. Starting from $-\alpha R_\gamma g$, it halves the path fraction at most 20 times until current teacher loss decreases and maximum protected-logit drift is at most 0.005, then attempts finite completion. It uses a fixed 64-dimensional input-address encoder with seed 20260817: a Gaussian projection of flattened vision inputs, or a projection of ordinary and position-weighted means of fixed token embeddings for text. Here the predictor has the form $F_t(x)=h_{\theta_t}(x)+S_t(z(x))$: the base consumes the original input and the residual consumes its fixed address. Thus the trainable network remains unfrozen while addresses remain fixed. Address consistency is checked on the requested finite nodes; global injectivity of this encoder is not assumed. This native probe measures the finite-drift transaction, whereas the chronological learner uses the counterfactual-normalized charge and full record lifecycle.

\subsection{Natural-state validation}

The natural-state validation removes the compatibility intervention entirely. It reuses the same three architectures, five seeds, and parent checkpoints, but follows the ordinary deterministic continual-learning stream with the true current labels.

Starting from each parent checkpoint, 50 pre-update states are sampled by a fixed schedule, every tenth parent step over the following 500 steps. No state is selected or rejected according to its measured compatibility. At each sampled state the experiment records:
\begin{itemize}[leftmargin=2em]
\item the supervised current gradient and its norm;
\item naturally occurring functional compatibility \(\kap\);
\item the genuine unrestricted decrease \(\DeltaZero\);
\item method-specific persistent decrease \(\DeltaP\);
\item the normalized persistent progress
\begin{equation}
\rhoP=\frac{\DeltaP}{\DeltaZero};
\label{causal:eq:rho}
\end{equation}
\item protected retention drift and retention-pass status.
\end{itemize}

Protected examples use a separate deterministic probe RNG. The validation branches are discarded after measurement, and only the ordinary supervised parent update is committed before continuing along the stream. Consequently, measurement does not alter the subsequent parent trajectory through branch-specific parameter updates or reservoir RNG consumption.

The natural validation contains
\[
3\times5\times50=750
\]
ordinary states and
\[
750\times7=5{,}250
\]
method outcomes. The stored retention tolerance is \(0.005\) for every natural-state row, and the recorded pass indicator agrees exactly with the condition \(D\le 0.005\).

\subsection{Seed-level statistical analysis}

\subsubsection{Inferential unit}

Inference respects the nested design: six interventions per state, 50 states per seed, and five independent seed trajectories per system.

A compatibility slope is first computed within every causal state:
\begin{equation}
b_s
=
\frac{
\sum_j(\kap_{sj}-\bar{\kap}_s)(\rho_{sj}-\bar{\rho}_s)
}{
\sum_j(\kap_{sj}-\bar{\kap}_s)^2
},
\label{causal:eq:state-slope}
\end{equation}
Here $s$ indexes states, $j$ the six interventions, and bars denote within-state intervention means; $\rho_{sj}$ is the persistent-progress ratio for the branch and budget being analyzed. The 50 state slopes are averaged within each seed. The five resulting seed-level slopes are the inferential replicates for each system. The report gives the mean seed slope and a two-sided 95\% Student-\(t\) interval with four degrees of freedom.

Seed-resolved intervals support primary inference; cross-system aggregates are descriptive.

\subsubsection{Interpretation of the normalized ratio}

The ratio \(\rhoP\) is useful because it normalizes persistent progress by the genuine same-state unrestricted decrease. It is not, however, mathematically bounded above by one for arbitrary proposal mechanisms. A method-specific direction can occasionally produce a larger finite decrease than the reference unrestricted direction.

The ratio is also numerically sensitive when \(\DeltaZero\) is very small. This matters particularly in the natural CNN data, where the minimum natural \(\DeltaZero\) is \(3.89\times10^{-4}\) and some normalized ratios become very large. Natural-state analysis therefore emphasizes medians, retention status, and AFM's pointwise theorem-aligned margin rather than interpreting the raw mean ratio alone.

\subsubsection{AFM theorem-aligned diagnostic}

For native AFM, $\widehat\lambda$ is the accepted fraction of the projected comparator proposal. Its stored theorem-aligned reference is
\begin{equation}
\rho_{\mathrm{ref}}
=
\frac{\lhat\kap}{3}.
\label{causal:eq:theory-ref}
\end{equation}
The empirical margin is
\begin{equation}
M
=
\rhoP-\frac{\lhat\kap}{3}.
\label{causal:eq:margin}
\end{equation}
This is an empirical theorem-alignment diagnostic. It is not presented as a proof that a nonlinear neural execution satisfies every outward-certified assumption of the analytical theorem.

\subsection{Results}

\subsubsection{Intervention fidelity}

Table~\ref{causal:tab:fidelity} audits the two most important controls. For all nonzero requested compatibility levels, realized compatibility tracks the target with extremely small mean absolute error. The requested-zero condition is different: the lowest attainable mean realized compatibility is approximately \(0.0107\) for the CNN, \(0.0413\) for the ViT, and \(8.33\times10^{-5}\) for the text transformer. The zero request therefore reaches the minimum attainable compatibility in the two vision systems, where exact \(\kap=0\) is not attained.

The second control is the finite unrestricted decrease. Within each causal state, the relative spread
\[
\frac{\max_\kap \DeltaZero-\min_\kap\DeltaZero}
{\operatorname{mean}_\kap \DeltaZero}
\]
averages only \(0.27\%\) to \(0.28\%\) and never exceeds \(0.40\%\) in any system. This is strong evidence that the six causal directions were compared at essentially matched finite unrestricted progress.

\begin{table}[ht]
\centering
\caption{Audit of causal intervention fidelity. Compatibility errors exclude the requested-zero condition. Relative $\DeltaZero$ spread is computed within each matched causal state.}
\label{causal:tab:fidelity}\label{tab:ed-fidelity}
\tablefont
\setlength{\tabcolsep}{3.2pt}
\begin{tabular}{lrrrrr}
\toprule
System & Mean $\kappa$ at request 0 & Nonzero MAE & Nonzero max error & Mean $\DeltaZero$ spread & Max spread \\
\midrule
CIFAR-10 CNN & 0.010681 & $1.1\times10^{-5}$ & $1.16\times10^{-4}$ & 0.283\% & 0.391\% \\
CIFAR-10 ViT & 0.041266 & $4.0\times10^{-6}$ & $8.9\times10^{-5}$ & 0.282\% & 0.396\% \\
Text transformer & 0.000083 & $5.0\times10^{-6}$ & $1.30\times10^{-3}$ & 0.268\% & 0.390\% \\
\bottomrule
\end{tabular}
\end{table}

\subsubsection{Causal effect of compatibility under an intermediate retention budget}

Table~\ref{causal:tab:slopes05} reports the matched seed-level slope analysis at \(\beta=0.5\). This intermediate budget illustrates the effect; the full budget sweep is reported below.

The projected AFM-compatible proposal shows a slope very close to one in every system:
\[
1.014\quad\text{CNN},\qquad
1.001\quad\text{ViT},\qquad
0.965\quad\text{text}.
\]
All five seed-level slopes are positive in all three systems, and each system-level 95\% interval lies well above zero. The cross-system descriptive mean over the 15 system-seed slopes is \(0.993\) with a 95\% \(t\) interval of \([0.980,1.006]\).

This means that, for the projected compatible proposal under this retention budget, a one-unit increase in realized functional compatibility is accompanied by an approximately one-unit increase in the normalized amount of current learning retained persistently. The result replicates across a convolutional vision model, a vision transformer, and a text transformer with all representations trainable.

\begin{table}[ht]
\centering
\caption{Matched causal slopes of $\rhoP$ on realized $\kap$ at $\beta=0.5$. Each system interval uses five seed-level matched slopes. The final column is a descriptive mean over 15 system-seed slopes and is not a claim about an unobserved population of architectures.}
\label{causal:tab:slopes05}
\tablefont
\setlength{\tabcolsep}{4pt}
\resizebox{\textwidth}{!}{
\begin{tabular}{lcccc}
\toprule
Method & CIFAR-10 CNN & CIFAR-10 ViT & Text transformer & Cross-system \\
\midrule
Projection / AFM-base & 1.014 [1.010, 1.017] & 1.001 [1.000, 1.002] & 0.965 [0.945, 0.985] & 0.993 [0.980, 1.006] \\
Unrestricted & 0.606 [0.591, 0.622] & 0.586 [0.581, 0.592] & 0.509 [0.488, 0.529] & 0.567 [0.542, 0.592] \\
Linearized distillation & 0.790 [0.772, 0.809] & 0.606 [0.567, 0.645] & 0.558 [0.542, 0.573] & 0.651 [0.593, 0.710] \\
EWC-prox & 0.570 [0.564, 0.577] & 0.584 [0.580, 0.589] & 0.506 [0.485, 0.528] & 0.554 [0.533, 0.574] \\
Replay & 0.013 [-0.013, 0.039] & 0.005 [-0.026, 0.036] & 0.002 [-0.003, 0.006] & 0.007 [-0.004, 0.017] \\
DER++ & 0.019 [-0.015, 0.053] & 0.010 [-0.019, 0.038] & 0.002 [-0.004, 0.008] & 0.010 [-0.001, 0.021] \\
\bottomrule
\end{tabular}}
\end{table}

Unrestricted, linearized-distillation and EWC-proximal branches also have positive slopes under the common \(\beta=0.5\) constraint. The replay and DER++ intervals include zero in each system. Thus the conversion of available compatibility into persistent progress differs across the tested proposal mechanisms.

\subsubsection{Retention strength changes the observable compatibility slope}

Table~\ref{causal:tab:beta-sensitivity} shows the projected AFM-compatible proposal across all stored retention budgets. At $\beta=0$, the numerical-floor condition yields zero observed persistent progress and slope. For every positive \(\beta\), the matched slope is positive in all three systems. The two CIFAR-10 systems reach an approximately unit slope by \(\beta=0.05\) to \(0.10\), whereas the text transformer requires a more permissive retention budget before approaching unit conversion.

\begin{table}[ht]
\centering
\caption{Retention-budget sensitivity for the projected AFM-compatible proposal. Entries are mean matched seed slopes with 95\% Student-\(t\) intervals.}
\label{causal:tab:beta-sensitivity}\label{tab:ed-budget-slopes}
\tablefont
\begin{tabular}{rccc}
\toprule
\(\beta\) & CIFAR-10 CNN & CIFAR-10 ViT & Text transformer \\
\midrule
0 & 0.000 [0.000, 0.000] & 0.000 [0.000, 0.000] & 0.000 [0.000, 0.000] \\
0.01 & 0.072 [0.008, 0.136] & 0.682 [0.592, 0.772] & 0.012 [0.002, 0.021] \\
0.05 & 0.999 [0.981, 1.016] & 0.996 [0.987, 1.005] & 0.131 [0.103, 0.158] \\
0.10 & 1.014 [1.010, 1.017] & 1.001 [1.000, 1.002] & 0.287 [0.233, 0.341] \\
0.25 & 1.014 [1.010, 1.017] & 1.001 [1.000, 1.002] & 0.728 [0.660, 0.795] \\
0.50 & 1.014 [1.010, 1.017] & 1.001 [1.000, 1.002] & 0.965 [0.945, 0.985] \\
1.00 & 1.014 [1.010, 1.017] & 1.001 [1.000, 1.002] & 1.000 [1.000, 1.000] \\
\bottomrule
\end{tabular}
\end{table}

The retention budget controls how much of an available compatible direction can be accepted. In AFM notation, this is the distinction between the compatibility factor \(\kap\) and the accepted path fraction \(\lhat\).

\subsubsection{Native AFM mechanism audit}

Native AFM is evaluated independently of the generic 33-scale projection frontier. Across all 3,750 controlled observations with requested \(\kap\in\{0.1,0.25,0.5,0.75,1\}\), provisional base-step acceptance is 3,750/3,750. The theorem-aligned empirical margin in Equation~\eqref{causal:eq:margin} is nonnegative in 3,750/3,750 of these accepted observations. Finite functional completion succeeds in 3,745/3,750 cases, or 99.87\%. The five unsuccessful completions are functional-constraint inconsistencies at CNN seed 71, state 81, one at each nonzero requested compatibility. Their provisional base measurements remain in this mechanism audit; they are not successful joint transactions. All probe branches are discarded after evaluation and leave the parent trajectory unchanged. No other obstruction type is recorded for these nonzero conditions.

For all successful finite completions, the stored finite current error, finite endpoint error, and protected endpoint error are \(0.0\) at the recorded precision, and the stored deployed progress ratio is exactly \(1.0\). These observations verify execution of the finite endpoint-completion mechanism on the declared finite constraints; they must not be interpreted as evidence of population retention away from the finite support.

Table~\ref{causal:tab:afm-native} gives the detailed native AFM summary.

\begin{table}[htbp]
\centering
\caption{Native AFM controlled-intervention results. ``Margin pass'' is the fraction of accepted rows with \(\rhoP-\lhat\kap/3\ge0\). ``Finite'' is successful finite completion divided by all rows at that condition.}\label{causal:tab:afm-native}
\resizebox{\textwidth}{!}{
\begin{tabular}{llrrrrrrr}
\toprule
System & Req. \(\kap\) & Realized \(\kap\) & Mean \(\rhoP\) & Mean \(\lhat\) & Min. margin & Accept & Margin pass & Finite \\
\midrule
CIFAR-10 CNN & 0.00 & 0.010681 & 0.00524 & 0.9519 & -0.003857 & 97.2\% & 78.6\% & 97.2\% \\
CIFAR-10 CNN & 0.10 & 0.100000 & 0.08937 & 1.0000 & 0.008058 & 100\% & 100\% & 99.6\% \\
CIFAR-10 CNN & 0.25 & 0.250000 & 0.24836 & 1.0000 & 0.083615 & 100\% & 100\% & 99.6\% \\
CIFAR-10 CNN & 0.50 & 0.499998 & 0.51097 & 1.0000 & 0.241686 & 100\% & 100\% & 99.6\% \\
CIFAR-10 CNN & 0.75 & 0.749999 & 0.76107 & 1.0000 & 0.407596 & 100\% & 100\% & 99.6\% \\
CIFAR-10 CNN & 1.00 & 0.999986 & 0.99974 & 1.0000 & 0.593211 & 100\% & 100\% & 99.6\% \\
\addlinespace
CIFAR-10 ViT & 0.00 & 0.041266 & 0.04206 & 1.0000 & -0.001655 & 100\% & 99.6\% & 100\% \\
CIFAR-10 ViT & 0.10 & 0.100000 & 0.10326 & 1.0000 & 0.025761 & 100\% & 100\% & 100\% \\
CIFAR-10 ViT & 0.25 & 0.250000 & 0.26034 & 1.0000 & 0.089506 & 100\% & 100\% & 100\% \\
CIFAR-10 ViT & 0.50 & 0.500000 & 0.51547 & 1.0000 & 0.274879 & 100\% & 100\% & 100\% \\
CIFAR-10 ViT & 0.75 & 0.750001 & 0.76052 & 1.0000 & 0.418082 & 100\% & 100\% & 100\% \\
CIFAR-10 ViT & 1.00 & 0.999990 & 0.99964 & 1.0000 & 0.649724 & 100\% & 100\% & 100\% \\
\addlinespace
Text transformer & 0.00 & 0.000083 & 0.00054 & 0.6902 & -0.000032 & 96.8\% & 95.0\% & 96.8\% \\
Text transformer & 0.10 & 0.100000 & 0.09963 & 1.0000 & 0.050312 & 100\% & 100\% & 100\% \\
Text transformer & 0.25 & 0.250000 & 0.24924 & 1.0000 & 0.144862 & 100\% & 100\% & 100\% \\
Text transformer & 0.50 & 0.500000 & 0.49905 & 1.0000 & 0.304095 & 100\% & 100\% & 100\% \\
Text transformer & 0.75 & 0.750000 & 0.74915 & 1.0000 & 0.474841 & 100\% & 100\% & 100\% \\
Text transformer & 1.00 & 0.999985 & 0.99991 & 1.0000 & 0.654466 & 100\% & 100\% & 100\% \\
\bottomrule
\end{tabular}}
\end{table}

The requested-zero boundary is reported separately. Among accepted requested-zero native AFM rows, 65 have a negative theorem-aligned empirical margin: 52 CNN rows, one ViT row, and 12 text rows. The negative margins are small in absolute value, with the most negative value \(-0.003857\). None of these 65 rows had certification of the relevant finite curvature or step condition, so they are not treated as certified theorem violations. The nonzero intervention levels have a uniformly nonnegative theorem-aligned empirical margin.

\subsubsection{Natural-state validation}

During ordinary supervised learning, natural compatibility occupies different ranges in the three systems:
\[
\begin{aligned}
\text{CNN: }&0.067\text{ to }0.443,\qquad \operatorname{mean}=0.228,\\
\text{ViT: }&0.278\text{ to }0.618,\qquad \operatorname{mean}=0.422,\\
\text{text: }&0.903\text{ to }0.976,\qquad \operatorname{mean}=0.950.
\end{aligned}
\]
The observed compatibility ranges differ across architectures and streams.

Native AFM satisfies the stored \(D\le0.005\) retention criterion on all 750 natural states. Its theorem-aligned margin is positive on all 750 observations. Table~\ref{causal:tab:natural-afm} reports the natural AFM summary.

\begin{table}[ht]
\centering
\caption{Native AFM on unmanipulated natural states. Median \(\rhoP\) is emphasized because the ratio can be heavy-tailed when \(\DeltaZero\) is small.}
\label{causal:tab:natural-afm}
\begin{tabular}{lrrrrrr}
\toprule
System & \(n\) & Mean \(\kap\) & \(\kap\) range & Mean \(\lhat\) & Median \(\rhoP\) & Min. margin \\
\midrule
CIFAR-10 CNN & 250 & 0.228 & [0.067, 0.443] & 0.552 & 0.348 & 0.0995 \\
CIFAR-10 ViT & 250 & 0.422 & [0.278, 0.618] & 0.670 & 0.277 & 0.0848 \\
Text transformer & 250 & 0.950 & [0.903, 0.976] & 0.107 & 0.119 & 0.0199 \\
\bottomrule
\end{tabular}
\end{table}

The text system illustrates why compatibility alone is not a complete predictor of realized persistent progress. Its natural compatibility is very high, approximately \(0.95\), but its mean native AFM path fraction is only \(0.107\). The resulting median persistent ratio is approximately \(0.119\). Together with the causal experiment, these observations show that realized persistent progress depends on both compatibility and the accepted retention-constrained path fraction.

Table~\ref{causal:tab:natural-pass} compares the natural-state retention-pass rates under the common criterion.

\begin{table}[ht]
\centering
\caption{Retention-pass rate on the 250 natural states per system under the stored common tolerance \(D\le0.005\). This table measures retention feasibility under this specific criterion, not general method quality.}
\label{causal:tab:natural-pass}
\begin{tabular}{lrrr}
\toprule
Method & CIFAR-10 CNN & CIFAR-10 ViT & Text transformer \\
\midrule
AFM & 100.0\% & 100.0\% & 100.0\% \\
Projection & 12.0\% & 34.8\% & 0.0\% \\
Linearized distillation & 11.6\% & 32.0\% & 0.0\% \\
EWC-prox & 0.0\% & 0.0\% & 0.0\% \\
Replay & 0.0\% & 0.0\% & 0.0\% \\
DER++ & 0.0\% & 0.0\% & 0.0\% \\
Unrestricted & 0.0\% & 0.0\% & 0.0\% \\
\bottomrule
\end{tabular}
\end{table}

Under the common strict retention criterion, native AFM's acceptance mechanism consistently produces a feasible update in the sampled natural states, whereas the other stored branch proposals often exceed that threshold. These branch-level feasibility results do not establish universal algorithmic dominance.

\subsubsection{Local causal scale and ordinary natural scale}

The controlled intervention deliberately chooses the weakest of six unrestricted endpoints as the common finite-decrease target. This makes the causal probe local, especially in the two vision systems. Table~\ref{causal:tab:delta-scales} makes this scale difference explicit.

\begin{table}[ht]
\centering
\caption{Finite unrestricted decrease in the controlled causal experiment and in the ordinary natural-state validation.}
\label{causal:tab:delta-scales}
\begin{tabular}{lrrrr}
\toprule
System & Causal mean \(\DeltaZero\) & Causal median & Natural mean \(\DeltaZero\) & Natural median \\
\midrule
CIFAR-10 CNN & \(1.64\times10^{-6}\) & \(1.48\times10^{-6}\) & 0.1199 & 0.1204 \\
CIFAR-10 ViT & \(8.85\times10^{-6}\) & \(7.01\times10^{-6}\) & 0.1582 & 0.1535 \\
Text transformer & \(2.32\times10^{-4}\) & \(1.67\times10^{-4}\) & 0.0270 & 0.0266 \\
\bottomrule
\end{tabular}
\end{table}

The matched experiment isolates local geometry. The natural-state validation uses a different objective and ordinary supervised updates; its raw loss decreases therefore do not define a common scale with the teacher-loss decreases. Update norms provide the cross-experiment calibration in Table~\ref{follow:tab:scale-calibration}.

\subsection{Scientific interpretation}
Functional compatibility is a causal control variable for the retention-constrained persistent-learning frontier in the admitted local regime. The matched-state design controls unrestricted opportunity and retention allowance separately; the seed-resolved effects reproduce across three fully trainable systems. The approximately unit slope for the projected branch at $\beta=0.5$ quantifies conversion of compatible gradient energy into persistent progress under those conditions. Differences across proposal mechanisms and retention budgets show that compatibility governs available geometry while the learning rule and accepted path determine its use. Natural-state validation and finite-scale interventions test the extent of this local result.

\subsection{Scope and limitations of the causal analysis}

The intervention, admission criteria and numerical conditions delimit the empirical conclusions.

\begin{enumerate}[leftmargin=2em]
\item \textbf{The requested-zero condition is a minimum-compatibility boundary in the vision systems.}
Mean realized compatibility at the requested-zero condition is 0.0107 for the CNN and 0.0413 for the ViT. We therefore describe this condition as minimum compatibility rather than exact zero.

\item \textbf{The matched controlled intervention is deliberately local.}
The weakest of the six unrestricted directions determines the common positive learning target, making the matched causal probe small relative to ordinary vision updates. This is useful for causal identification but does not imply that the same normalized slope must hold at arbitrary finite step size. The multiscale study and fixed-norm bridge quantify this boundary directly.

\item \textbf{Near-zero theorem-aligned margins require the theorem conditions to be distinguished from the diagnostic quantity.}
All 65 accepted requested-zero observations with negative theorem-aligned margins lacked certification of the relevant finite curvature or step condition. Numerical verification agreed with the reported compatibility values to within $1.5883\times10^{-9}$ and with the persistent-ratio values exactly. These rows are therefore not counted as certified theorem violations. Nonzero intervention levels have uniformly nonnegative theorem-aligned empirical margins.

\item \textbf{Natural-state ratios can be unstable when unrestricted progress is small.}
In the CNN natural-state data, small values of $\DeltaZero$ can produce large ratios. Medians and absolute persistent progress are therefore reported alongside ratio summaries where this matters.

\item \textbf{Natural-state validation is observational.}
It demonstrates that compatibility, accepted path length and retention-feasible AFM updates arise during ordinary stream learning, but the controlled matched-state intervention supplies the causal identification. Complete architecture, optimizer, replay, projector and analysis configurations are available with the released code identified in the paper.
\end{enumerate}

\subsection{From matched interventions to natural learning}
The controlled and natural-state results connect available compatibility to accepted persistent movement. Varying the direction construction, architecture and update scale tests the dependence of that relation on the particular local intervention.

\subsection{Complete cross-system slope summary}

Table~\ref{causal:tab:pipeline-slopes} reports aggregate cross-system point slopes across the full budget grid. Confidence intervals for the main comparisons use the seed-level effects defined above.

\begin{longtable}{llrrrr}
\caption{Aggregate cross-system slopes.}\label{causal:tab:pipeline-slopes}\\
\toprule
Method & \(\beta\) & CNN & ViT & Text & Pooled \\
\midrule
\endfirsthead
\toprule
Method & \(\beta\) & CNN & ViT & Text & Pooled \\
\midrule
\endhead
Projection / AFM-base & 0.00 & 0.000 & 0.000 & -0.000 & -0.000 \\
Projection / AFM-base & 0.01 & 0.072 & 0.682 & 0.012 & 0.255 \\
Projection / AFM-base & 0.05 & 0.999 & 0.996 & 0.131 & 0.708 \\
Projection / AFM-base & 0.10 & 1.014 & 1.001 & 0.287 & 0.767 \\
Projection / AFM-base & 0.25 & 1.014 & 1.001 & 0.728 & 0.914 \\
Projection / AFM-base & 0.50 & 1.014 & 1.001 & 0.965 & 0.993 \\
Projection / AFM-base & 1.00 & 1.014 & 1.001 & 1.000 & 1.005 \\
\addlinespace
Unrestricted & 0.00 & 0.000 & 0.000 & 0.000 & 0.000 \\
Unrestricted & 0.01 & 0.055 & 0.537 & 0.010 & 0.201 \\
Unrestricted & 0.05 & 0.728 & 0.769 & 0.105 & 0.534 \\
Unrestricted & 0.10 & 0.723 & 0.743 & 0.195 & 0.554 \\
Unrestricted & 0.25 & 0.668 & 0.679 & 0.442 & 0.597 \\
Unrestricted & 0.50 & 0.606 & 0.586 & 0.509 & 0.567 \\
Unrestricted & 1.00 & 0.000 & 0.000 & 0.000 & 0.000 \\
\addlinespace
Linearized distillation & 0.00 & 0.000 & 0.000 & 0.000 & 0.000 \\
Linearized distillation & 0.01 & 0.060 & 0.552 & 0.010 & 0.208 \\
Linearized distillation & 0.05 & 0.838 & 0.826 & 0.101 & 0.589 \\
Linearized distillation & 0.10 & 0.910 & 0.833 & 0.212 & 0.652 \\
Linearized distillation & 0.25 & 0.801 & 0.689 & 0.483 & 0.658 \\
Linearized distillation & 0.50 & 0.790 & 0.606 & 0.558 & 0.651 \\
Linearized distillation & 1.00 & 0.789 & 0.597 & 0.511 & 0.632 \\
\addlinespace
EWC-prox & 0.00 & 0.000 & 0.000 & 0.000 & 0.000 \\
EWC-prox & 0.01 & 0.057 & 0.535 & 0.010 & 0.200 \\
EWC-prox & 0.05 & 0.731 & 0.768 & 0.105 & 0.535 \\
EWC-prox & 0.10 & 0.728 & 0.741 & 0.195 & 0.555 \\
EWC-prox & 0.25 & 0.678 & 0.683 & 0.443 & 0.601 \\
EWC-prox & 0.50 & 0.570 & 0.584 & 0.506 & 0.554 \\
EWC-prox & 1.00 & 0.010 & 0.007 & 0.006 & 0.008 \\
\addlinespace
Replay & 0.00 & 0.000 & 0.000 & 0.000 & 0.000 \\
Replay & 0.01 & 0.000 & 0.000 & 0.000 & 0.000 \\
Replay & 0.05 & 0.000 & 0.000 & 0.000 & 0.000 \\
Replay & 0.10 & 0.001 & 0.000 & 0.000 & 0.001 \\
Replay & 0.25 & 0.007 & 0.002 & 0.001 & 0.003 \\
Replay & 0.50 & 0.013 & 0.005 & 0.002 & 0.007 \\
Replay & 1.00 & 0.022 & 0.009 & 0.001 & 0.011 \\
\addlinespace
DER++ & 0.00 & 0.000 & 0.000 & 0.000 & 0.000 \\
DER++ & 0.01 & 0.000 & 0.000 & 0.000 & 0.000 \\
DER++ & 0.05 & 0.001 & 0.000 & 0.000 & 0.001 \\
DER++ & 0.10 & 0.003 & 0.003 & 0.001 & 0.002 \\
DER++ & 0.25 & 0.006 & 0.002 & 0.003 & 0.004 \\
DER++ & 0.50 & 0.019 & 0.010 & 0.002 & 0.010 \\
DER++ & 1.00 & 0.049 & 0.020 & 0.000 & 0.023 \\
\bottomrule
\end{longtable}

\subsection{Natural-state descriptive ratios}

Table~\ref{causal:tab:natural-rho} reports both means and medians to expose ratio skew. The unrestricted branch is exactly one by construction because it defines \(\DeltaZero\). Large positive or negative means in other CNN branches arise from a small number of states with small \(\DeltaZero\); the corresponding medians are substantially more stable.

\begin{longtable}{llrrr}
\caption{Natural-state persistent-progress ratios and retention pass rate.}\label{causal:tab:natural-rho}\\
\toprule
System & Method & Mean \(\rhoP\) & Median \(\rhoP\) & Retention pass \\
\midrule
\endfirsthead
\toprule
System & Method & Mean \(\rhoP\) & Median \(\rhoP\) & Retention pass \\
\midrule
\endhead
CIFAR-10 CNN & AFM & 0.937 & 0.348 & 100.0\% \\
CIFAR-10 CNN & Projection & 1.787 & 0.631 & 12.0\% \\
CIFAR-10 CNN & Linearized distillation & 1.789 & 0.632 & 11.6\% \\
CIFAR-10 CNN & EWC-prox & 1.075 & 1.008 & 0.0\% \\
CIFAR-10 CNN & Replay & -13.393 & -3.674 & 0.0\% \\
CIFAR-10 CNN & DER++ & -2.477 & -0.223 & 0.0\% \\
CIFAR-10 CNN & Unrestricted & 1.000 & 1.000 & 0.0\% \\
\addlinespace
CIFAR-10 ViT & AFM & 0.341 & 0.277 & 100.0\% \\
CIFAR-10 ViT & Projection & 0.503 & 0.503 & 34.8\% \\
CIFAR-10 ViT & Linearized distillation & 0.505 & 0.505 & 32.0\% \\
CIFAR-10 ViT & EWC-prox & 0.999 & 0.999 & 0.0\% \\
CIFAR-10 ViT & Replay & 0.651 & 0.748 & 0.0\% \\
CIFAR-10 ViT & DER++ & 0.929 & 0.941 & 0.0\% \\
CIFAR-10 ViT & Unrestricted & 1.000 & 1.000 & 0.0\% \\
\addlinespace
Text transformer & AFM & 0.102 & 0.119 & 100.0\% \\
Text transformer & Projection & 0.950 & 0.952 & 0.0\% \\
Text transformer & Linearized distillation & 0.952 & 0.953 & 0.0\% \\
Text transformer & EWC-prox & 0.999 & 1.000 & 0.0\% \\
Text transformer & Replay & 1.003 & 0.999 & 0.0\% \\
Text transformer & DER++ & 1.005 & 1.004 & 0.0\% \\
Text transformer & Unrestricted & 1.000 & 1.000 & 0.0\% \\
\bottomrule
\end{longtable}

\section{Generality, natural-state behavior, and finite-scale boundary}
\subsection{Robust across directions and models}

To test dependence on the direction construction, we generated four distinct learning directions at each nonzero compatibility level. At $\beta=0.5$, the within-$\kappaF$ variation across these directions was only 3.8\%, 2.1\% and 3.0\% of the between-$\kappaF$ variation in the convolutional network, vision transformer and text transformer, respectively, while the corresponding compatibility slopes were 1.046, 1.029 and 0.986 (Figure~\ref{fig:generality}a,b). This substantially weakens the alternative explanation that the original dose response arose from a particular generalized-eigenmode construction.

We next doubled the number of seed trajectories and added a stronger vision transformer. Across ten seeds, the $\beta=0.5$ slopes were 1.015 for the convolutional network, 1.001 for the baseline vision transformer, 0.999 for the stronger vision transformer, and 0.960 for the text transformer; the descriptive equal-system pooled slope was 0.994 (Figure~\ref{fig:generality}c; Table~\ref{tab:ed-tenseed}). A separate multiscale experiment enlarged the designed intervention while preserving the matched causal construction. Vision slopes remained near one throughout this expanded local regime (Figure~\ref{fig:generality}d; Table~\ref{follow:tab:multiscale}). Calibration against ordinary training showed that even the largest controlled multiscale interventions remained substantially smaller than ordinary updates in vision and closer, although still smaller, in text (Figure~\ref{fig:boundary}). The multiscale experiment therefore tests the robustness of the local causal relation, not invariance at arbitrary step magnitude.

\begin{figure}[tbp]
\centering

\begin{minipage}{0.90\textwidth}
\raggedright
{\bfseries\large a}\par
\vspace{0.15em}
\centering
\includegraphics[width=\linewidth]{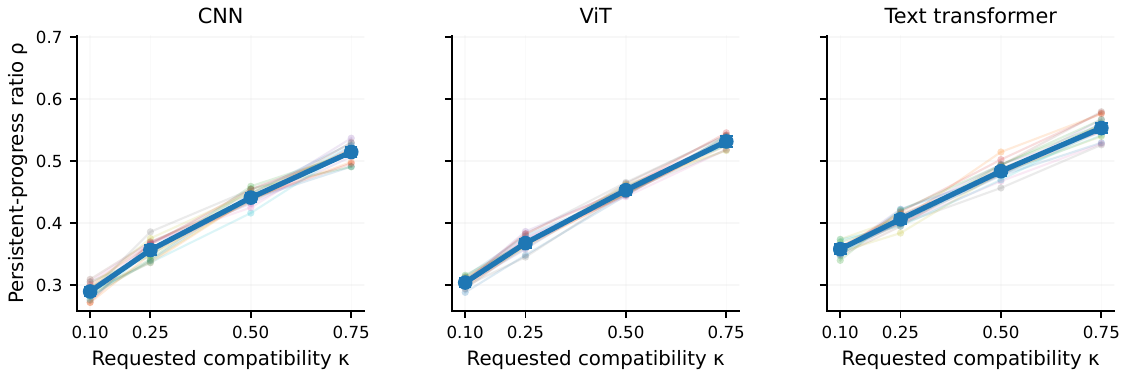}
\end{minipage}

\vspace{0.6em}

\begin{minipage}[t]{0.37\textwidth}
\raggedright
{\bfseries\large b}\par
\vspace{0.15em}
\centering
\includegraphics[width=\linewidth]{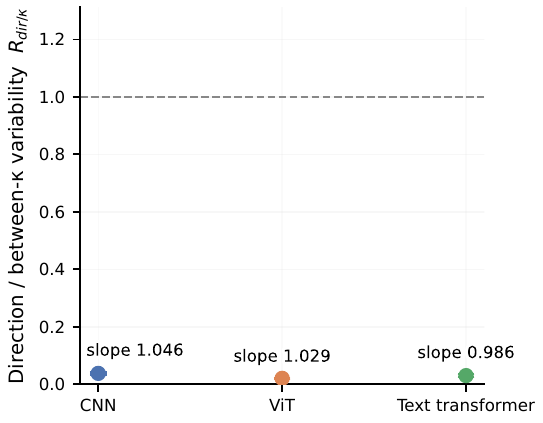}
\end{minipage}
\hfill
\begin{minipage}[t]{0.45\textwidth}
\raggedright
{\bfseries\large c}\par
\vspace{0.15em}
\centering
\includegraphics[width=\linewidth]{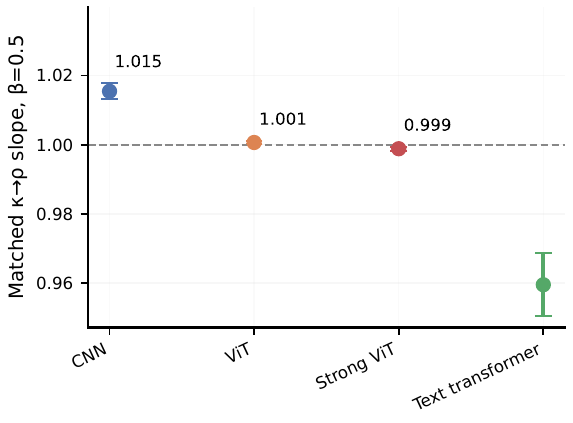}
\end{minipage}

\vspace{0.5em}

\begin{minipage}{0.55\textwidth}
\raggedright
{\bfseries\large d}\par
\vspace{0.15em}
\centering
\includegraphics[width=\linewidth]{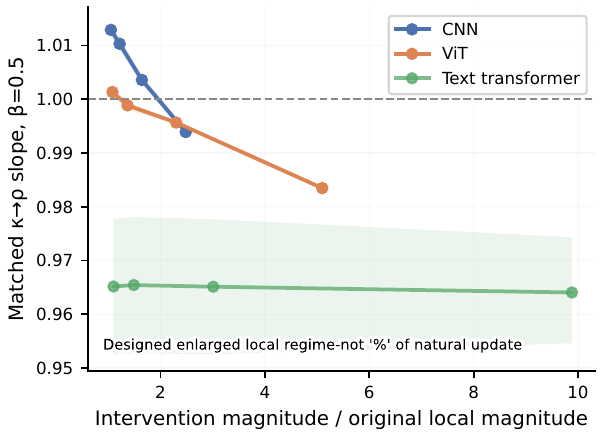}
\end{minipage}

\caption{\textbf{The compatibility effect generalizes across directions, models and local scales.}
\textbf{a}, Four independently constructed directions at each nonzero compatibility level; light traces and points show individual directions and the dark line shows the mean.
\textbf{b}, Ratio of within-compatibility directional variation to between-compatibility variation at $\beta=0.5$.
For \textbf{a,b}, $n=5$ seed trajectories per system, 50 states per seed, and four directions at each of four nonzero compatibility levels.
\textbf{c}, Ten-seed compatibility slopes for the convolutional network, baseline vision transformer, stronger vision transformer and text transformer; $n=10$ seed trajectories per system and intervals use 100,000 deterministic seed-bootstrap resamples.
\textbf{d}, Multiscale compatibility slopes in the designed local expansion; $n=5$ seed trajectories per system and 50 states per seed. The scale coordinate is an enlargement of the local intervention and is not a percentage of a natural training update.}
\label{fig:generality}
\end{figure}

\subsection{Compatibility in ordinary learning}

Compatibility also arose spontaneously during ordinary, unmanipulated learning. Across 750 natural states, mean $\kappaF$ was 0.228 in the convolutional network, 0.422 in the vision transformer, and 0.950 in the text transformer (Figure~\ref{fig:natural}a). Native AFM satisfied the common recorded retention criterion at all 750 states. Yet the text model, despite its high compatibility, accepted a mean persistent path fraction of only 0.107 and had a median persistent-progress ratio of 0.119. The two vision systems accepted substantially larger path fractions. Natural-state behavior is therefore consistent with the controlled causal result while demonstrating that $\kappaF$ is not a complete scalar law: even a highly compatible direction must fit inside the retention-constrained path that the learner can safely accept.

AFM provides a constructive realization of this geometry rather than only a measurement procedure (Figure~\ref{fig:natural}d). In the controlled causal suite, every nonzero native-AFM request accepted a persistent update and every such row had a nonnegative theorem-aligned empirical margin; finite completion succeeded in 3,745 of 3,750 nonzero cases. The chronological program evaluated AFM on CORe50~\citep{lomonaco2017core50}, CLEAR-10~\citep{lin2021clear} and CLAD-C~\citep{verwimp2023clad} over five frozen seeds per protocol. AFM produced retention-plasticity hypervolumes of 0.248, 0.381 and 0.513, respectively, and exceeded the strongest tested classical non-oracle family on every paired seed in all three protocols. Against six recent challengers under a common predictive-base and learner-visible-information interface, AFM had the highest dataset-mean primary score on CLEAR-10 and CLAD-C; FGH~\citep{michel2026fgh} was the only dataset-level reversal, exceeding AFM on CORe50 by 2.8\% (Table~\ref{tab:ed-benchmarks}). On the equal-weighted three-dataset primary summary, AFM exceeded CCL-DC~\citep{wang2024ccldc}, MKD~\citep{michel2024mkd}, LPR~\citep{yoo2024lpr}, FGH, aL-SAR~\citep{seo2025budgeted} and OCAR~\citep{urettini2025ocar} by 7.0\%, 11.4\%, 16.3\%, 21.7\%, 31.7\% and 36.8\%, respectively. These are not equal-compute or equal-storage claims, and the benchmark representation prefix was frozen after initial learning. They establish that the empirically checked retention-plasticity transaction is executable and empirically competitive, while the causal experiments address the broader scientific question of why persistent learning is possible in some directions and not others.

\begin{figure}[tbp]
\centering

\begin{minipage}[t]{0.47\textwidth}
\raggedright
{\bfseries\large a}\par
\vspace{0.15em}
\centering
\includegraphics[width=\linewidth]{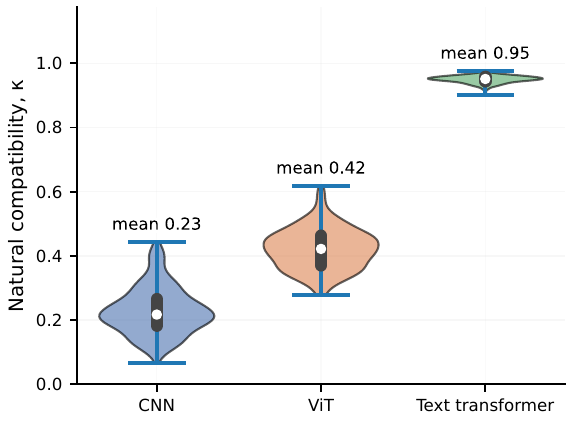}
\end{minipage}
\hfill
\begin{minipage}[t]{0.47\textwidth}
\raggedright
{\bfseries\large b}\par
\vspace{0.15em}
\centering
\includegraphics[width=\linewidth]{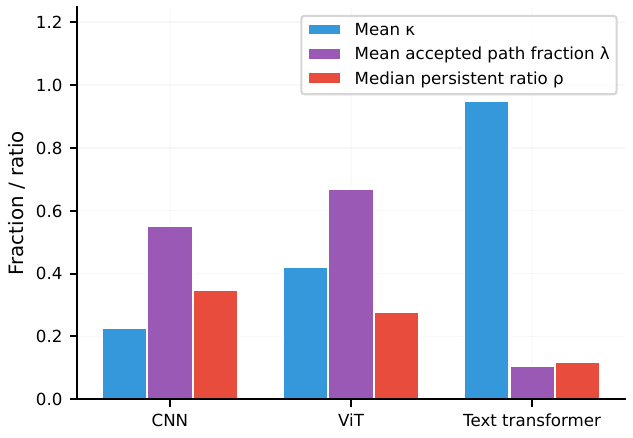}
\end{minipage}

\vspace{0.6em}

\begin{minipage}[t]{0.47\textwidth}
\raggedright
{\bfseries\large c}\par
\vspace{0.15em}
\centering
\includegraphics[width=\linewidth]{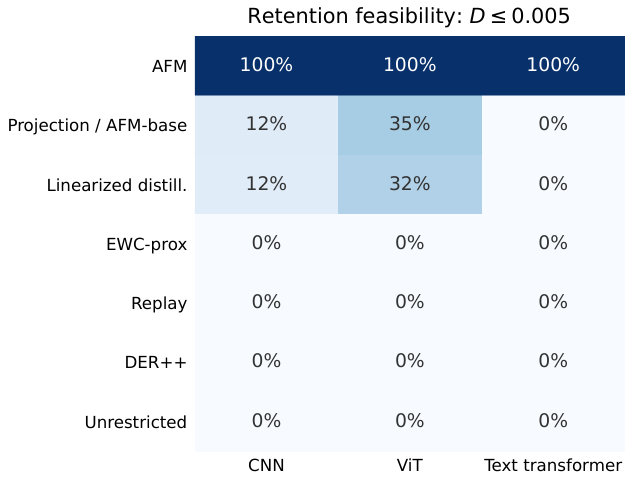}
\end{minipage}
\hfill
\begin{minipage}[t]{0.47\textwidth}
\raggedright
{\bfseries\large d}\par
\vspace{0.15em}
\centering
\includegraphics[width=\linewidth]{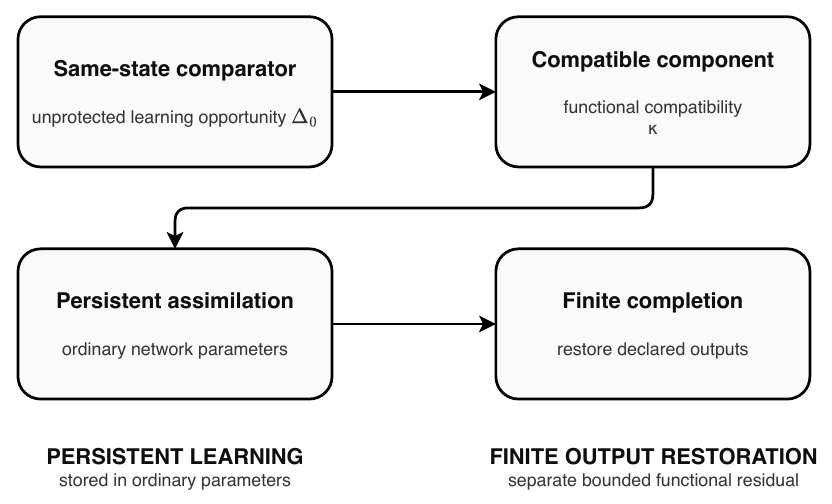}
\end{minipage}

\caption{\textbf{Functional compatibility arises during ordinary learning and AFM converts it into a protected learning frontier.}
\textbf{a}, Natural compatibility distributions across 750 unmanipulated states, comprising $n=5$ seed trajectories and 50 states per seed in each of three systems.
\textbf{b}, Natural compatibility, accepted AFM path fraction and persistent-progress ratio.
\textbf{c}, Retention-pass rates under the common recorded $D\leq0.005$ criterion; rates describe feasibility under this criterion rather than universal method quality.
\textbf{d}, Schematic of the AFM transaction separating persistent learning in ordinary network parameters from finite output restoration.}
\label{fig:natural}
\end{figure}

\begin{table}[tbp]\centering\tablefont
\caption{AFM chronological retention-plasticity benchmarks}\label{tab:ed-benchmarks}
\begin{tabular}{lrrrrrrr}\toprule
Dataset & AFM & CCL-DC & MKD & FGH & aL-SAR & LPR & OCAR\\\midrule
CORe50 & 0.2478 & 0.2016 & 0.1799 & \textbf{0.2546} & 0.1801 & 0.1640 & 0.1436\\
CLEAR-10 & \textbf{0.3807} & 0.3731 & 0.3513 & 0.3773 & 0.3572 & 0.3485 & 0.3164\\
CLAD-C & \textbf{0.5129} & 0.4921 & 0.4931 & 0.3058 & 0.3296 & 0.4694 & 0.3747\\\bottomrule
\end{tabular}
\vspace{1em}
\begin{tabular}{lrrrr}\toprule
Dataset & AFM & Strongest classical non-oracle & Mean difference & 95\% paired CI\\\midrule
CORe50 & 0.2478 & No protection 0.2388 & +0.0090 & [0.0059,0.0126]\\
CLEAR-10 & 0.3807 & No protection 0.3740 & +0.0068 & [0.0053,0.0082]\\
CLAD-C & 0.5129 & A-GEM 0.5019 & +0.0110 & [0.0054,0.0166]\\\bottomrule
\end{tabular}
\vspace{1em}
\parbox{0.94\textwidth}{\tablefont Primary score is the protocol-specific retention-plasticity hypervolume. Modern challengers use frozen method-native configurations; this comparison controls predictive base and learner-visible information but is not an equal-compute/storage study.}
\end{table}

\subsection{Nonlinear geometry sets the boundary}

The matched-$\DeltaZero$ intervention was deliberately constructed in a local regime, where the available unrestricted learning opportunity could be controlled across compatibility levels. We therefore asked how far the intervention remains realizable as update magnitude approaches ordinary training scales. We constructed a fixed-norm bridge in the two vision systems using ten seeds and 500 preselected states per architecture. At each state we attempted four compatibility levels at each of four target update norms: 1\%, 10\%, 50\%, and 100\% of the median natural unrestricted update norm. Unlike the matched-$\DeltaZero$ experiment, $\DeltaZero$ was now an outcome rather than a matched quantity. A state-scale condition was admitted only if all four same-norm directions produced positive unrestricted finite progress.

The intervention encountered a sharp nonlinear boundary (Figure~\ref{fig:boundary}; Table~\ref{tab:ed-bridge}). At 1\% natural norm, 231 of 500 vision-transformer states across all ten seeds supported the four-level comparison, whereas only one convolutional-network state did. At 10\% and above, neither architecture supported a jointly feasible four-level continuum. Failure of the four-level comparison at these scales does not estimate a zero or negative compatibility effect; rather, the matched intervention itself becomes infeasible because finite nonlinear geometry makes one or more same-norm directions non-descending.

Where the fixed-norm bridge remained feasible, compatibility still increased the absolute amount of persistent learning. In the 231 feasible vision-transformer states at 1\% natural norm, the slope of $\DeltaPers$ on $\kappaF$ for the projection/AFM-base branch was $3.064\times10^{-4}$ with a 95\% interval of $[3.031,3.093]\times10^{-4}$. The slope of the normalized ratio $\rhoPers$, however, was $-0.572$ with an interval crossing zero because the denominator $\DeltaZero$ itself changed with compatibility. Stratifying states by the coefficient of variation of $\DeltaZero$ left the absolute persistent slope essentially unchanged while the ratio slope shifted from positive to strongly negative (Figure~\ref{fig:boundary}). Thus the matched-$\DeltaZero$ experiment measures the fraction of an equal learning opportunity that can persist; at fixed finite norm, absolute persistent progress is the cleaner causal outcome.

\begin{figure}[tbp]
\centering

\begin{minipage}[t]{0.47\textwidth}
\raggedright
{\bfseries\large a}\par
\vspace{0.15em}
\centering
\includegraphics[width=\linewidth]{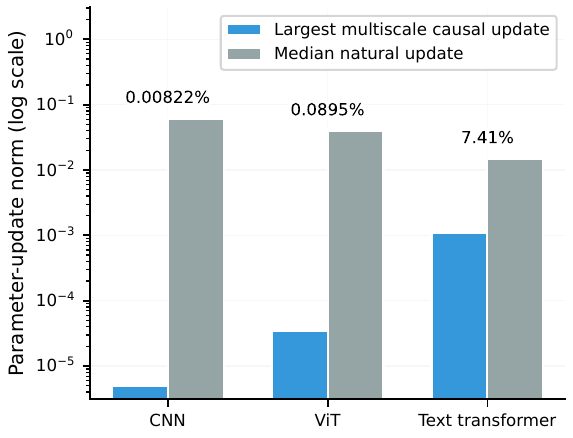}
\end{minipage}
\hfill
\begin{minipage}[t]{0.47\textwidth}
\raggedright
{\bfseries\large b}\par
\vspace{0.15em}
\centering
\includegraphics[width=\linewidth,trim=0bp 16bp 0bp 0bp,clip]{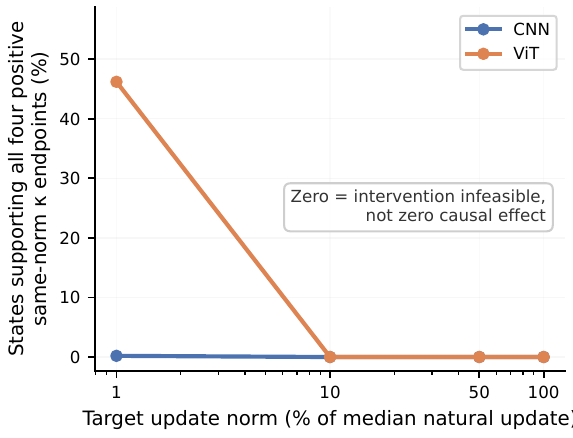}\par
\vspace{1pt}
\resizebox{0.94\linewidth}{!}{\sffamily Target update norm (\% of median natural update)}
\end{minipage}

\vspace{0.6em}

\begin{minipage}[t]{0.58\textwidth}
\raggedright
{\bfseries\large c}\par
\vspace{0.15em}
\centering
\includegraphics[width=\linewidth]{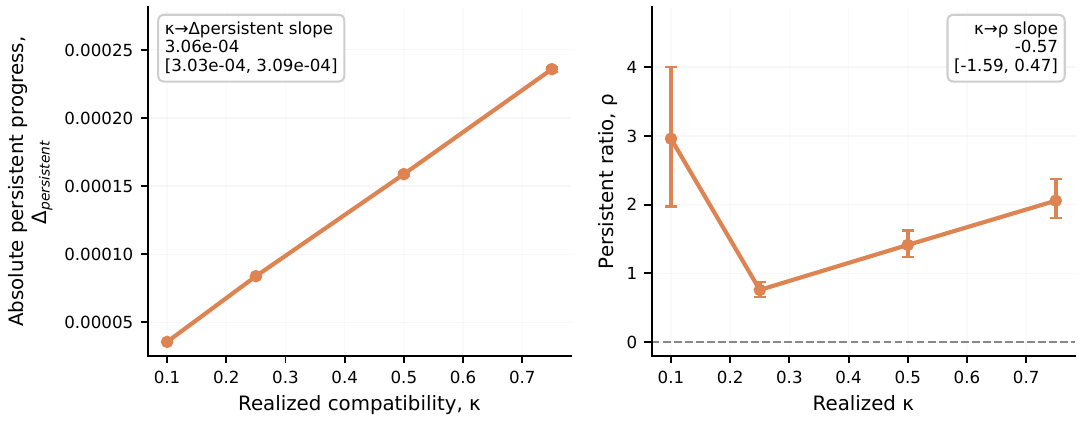}
\end{minipage}
\hfill
\begin{minipage}[t]{0.36\textwidth}
\raggedright
{\bfseries\large d}\par
\vspace{0.15em}
\centering
\includegraphics[width=\linewidth]{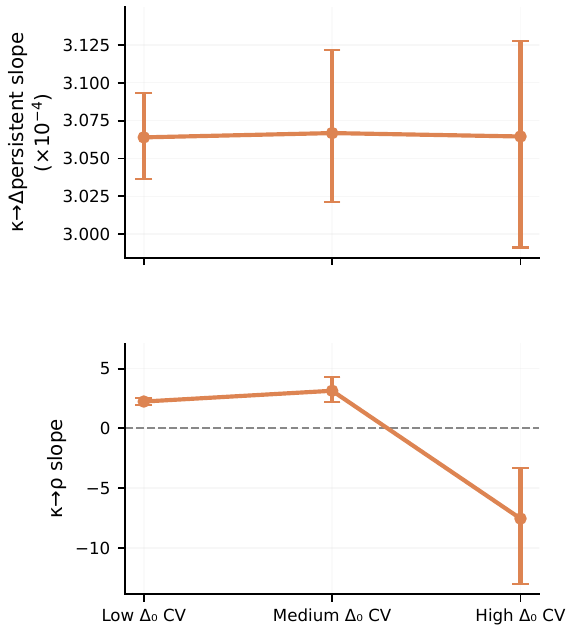}
\end{minipage}

\caption{\textbf{Nonlinear geometry sets the finite-scale boundary of compatibility-controlled learning.}
\textbf{a}, Largest designed causal update norms compared with median natural unrestricted update norms; natural calibration uses 250 ordinary states per original system.
\textbf{b}, Fraction of 500 preselected states per architecture, 10 seeds by 50 states, supporting all four positive same-norm compatibility endpoints at 1\% to 100\% of median natural update norm. Zero feasibility denotes absence of the matched four-level intervention, not a zero causal effect.
\textbf{c}, In 231 feasible vision-transformer states at 1\% natural norm, distributed across all 10 seeds, compatibility increases absolute persistent progress whereas the normalized ratio does not show a positive slope.
\textbf{d}, Stratification of the same 231 states by unrestricted-progress heterogeneity leaves the absolute persistent slope essentially invariant while the ratio slope reverses sign. Bridge intervals use $n=10$ seed-level effects and 100,000 deterministic seed-bootstrap resamples.}
\label{fig:boundary}
\end{figure}

\begin{table}[tbp]\centering\tablefont
\caption{Fixed-norm bridge feasibility and decomposition}\label{tab:ed-bridge}
\begin{tabularx}{0.98\textwidth}{l >{\centering\arraybackslash}X >{\centering\arraybackslash}X >{\centering\arraybackslash}X >{\centering\arraybackslash}X}\toprule
System & Natural-norm fraction & Feasible / attempted & Seeds with feasibility & Target norm\\\midrule
CNN & 0.01 & 1/500 (0.2\%) & 1 & 0.00060570\\
CNN & 0.10-1.00 & 0/1500 & 0 & 0.00605702-0.06057018\\
ViT & 0.01 & 231/500 (46.2\%) & 10 & 0.00039144\\
ViT & 0.10-1.00 & 0/1500 & 0 & 0.00391441-0.03914410\\\bottomrule
\end{tabularx}
\vspace{1em}
\begin{tabularx}{0.98\textwidth}{>{\raggedright\arraybackslash}p{0.22\textwidth} >{\centering\arraybackslash}X >{\centering\arraybackslash}X}\toprule
Method, ViT 1\%, $\beta=0.5$ & Absolute persistent-progress slope [95\% CI] & Normalized-ratio slope [95\% CI]\\\midrule
EWC-prox & $4.3693\times10^{-5}$ [3.5585,5.1115]$\times10^{-5}$ & $-7.895$ [$-10.701$,$-5.197$]\\
Linearized distillation & $2.59767\times10^{-4}$ [2.54328,2.65080]$\times10^{-4}$ & $-3.494$ [$-5.588$,$-1.502$]\\
Projection / AFM base & $3.06428\times10^{-4}$ [3.03053,3.09322]$\times10^{-4}$ & $-0.572$ [$-1.591$,0.455]\\
Unrestricted & $4.38287\times10^{-5}$ [3.55718,5.12805]$\times10^{-5}$ & $-7.891$ [$-10.717$,$-5.190$]\\\bottomrule
\end{tabularx}
\vspace{1em}
\parbox{0.94\textwidth}{\tablefont Bridge admission requires positive unrestricted finite progress for all four same-norm compatibility directions. Across 4,000 attempted state-scale conditions, 232 were feasible. Maximum absolute compatibility error was $9.656\times10^{-6}$ and maximum relative update-norm error was $3.674\times10^{-5}$.}
\end{table}

\subsection{Common estimands and reference intervention}
The follow-up studies test whether the local effect survives alternative direction constructions, larger intervention scales, additional seeds and a higher-capacity architecture. They also examine the near-zero boundary and the feasibility of approaching ordinary update norms. Each retains the distinction between a controlled compatibility intervention, an observational natural-state measurement and a conditional finite-norm comparison.

The incoming teacher gradient is $q$, realized compatibility is~\eqref{causal:eq:kappa}, and $\rho=\Delta_{\rm persistent}/\Delta_0$. The retention grid remains $\beta\in\{0,0.01,0.05,0.1,0.25,0.5,1\}$ unless stated otherwise. State slopes are averaged within seed before inference. The native diagnostic $M=\rho-\widehat\lambda\kappa/3$ compares an observed ratio with a theorem-aligned reference; it becomes a certified lower-bound test only when its curvature, step and projection premises are certified.

The original experiment supplied 31,500 generic method outcomes and 220,500 frontier rows from 750 admitted states, with complete expected coverage. Its matched-$\Delta_0$ slopes in Table~\ref{follow:tab:original-reference} establish the reference effect. Subsequent designs retain the same within-state controls where feasible and report the boundary where those controls cease to be jointly realizable.

\begin{table}[ht]
\centering
\caption{Reference primary causal slopes at retention budget $\beta=0.5$. Intervals are two-sided 95\% Student-$t$ intervals over seed-level effects; the cross-system row is a descriptive interval over the 15 system--seed effects.}
\label{follow:tab:original-reference}
\begin{tabular}{lcc}
\toprule
System & Matched slope & 95\% CI \\
\midrule
CIFAR-10 CNN & 1.01377 & [1.01027, 1.01727] \\
CIFAR-10 ViT & 1.00082 & [0.99999, 1.00165] \\
Text transformer & 0.96497 & [0.94527, 0.98466] \\
Cross-system descriptive mean & 0.99319 & [0.98043, 1.00594] \\
\bottomrule
\end{tabular}
\end{table}

\subsection{Experiment 1: multiscale causal compatibility}
\label{follow:sec:multiscale}

\subsubsection{Scale intervention}

The first stress test asked whether the matched compatibility effect survives when the intervention is expanded beyond the original very local operating scale. The design retained the same three systems, five seeds, seven methods, six requested compatibility levels, and seven retention budgets. Four intervention scale indices were used:
\begin{equation}
    s\in\{0.05,0.2,0.5,0.9\}.
\end{equation}
The multiscale design uses a logarithmic expansion from the local intervention to a common discrete-grid peak. For each direction, the calibration evaluates $\alpha_0 2^{-j}$ from the system-specific initial comparator step $\alpha_0$, $j=0,\ldots,20$, stopping early if parameter rounding produces no movement; it retains the largest positive decrease found. A condition without the required positive peaks or successful target matching is not admitted. For each admitted parent state, let $\Delta_0^{\rm local}$ be the original local matching anchor and let $\Delta_0^{\rm peak}$ be the common feasible peak defined by the minimum, over the six requested compatibility conditions, of the largest positive unrestricted decrease attained on the declared calibration grid. The target at scale coordinate $s$ is
\begin{equation}
\Delta_0(s)=\exp\!\left[(1-s)\log \Delta_0^{\rm local}+s\log \Delta_0^{\rm peak}\right].
\label{follow:eq:multiscale-target}
\end{equation}
Thus $s$ interpolates geometrically between two positive decrease scales; it must not be interpreted as 5\%, 20\%, 50\%, or 90\% of a natural training update.

The expected generic-frontier row count is
\[
15\times 50\times 4\times 6\times 7\times 7=882{,}000,
\]
corresponding to 15 system--seed runs, 50 admitted parent states, four scale coordinates, six requested compatibility conditions, seven generic method labels, and seven retention budgets. Validation reported all 882,000 expected frontier rows, 15/15 completed runs, and no validation failures.

\subsubsection{Key result}

At $\beta=0.5$, AFM/projection slopes remain near one in both vision systems across the designed scales (Table~\ref{follow:tab:multiscale}). Figure~\ref{fig:ed-multiscale} shows the native-AFM dose responses for all three systems.

\begin{table}[ht]
\centering
\caption{AFM/projection matched $\kappa$ slopes across the designed multiscale range at $\beta=0.5$.}
\label{follow:tab:multiscale}
\begin{tabular}{lcccc}
\toprule
System & $s=0.05$ & $s=0.2$ & $s=0.5$ & $s=0.9$ \\
\midrule
CIFAR-10 CNN & 1.01290 & 1.01032 & 1.00359 & 0.99390 \\
CIFAR-10 ViT & 1.00135 & 0.99886 & 0.99567 & 0.98347 \\
\bottomrule
\end{tabular}
\end{table}

The CNN also illustrates the interaction between compatibility and the retention budget. At the tight budget $\beta=0.01$, the AFM/projection matched slope rose from 0.10272 at $s=0.05$ to 0.16003 at $s=0.2$, 0.48762 at $s=0.5$, and 0.96935 at $s=0.9$. At $\beta=0.05$, the corresponding CNN slopes were 1.01171, 1.01032, 1.00359, and 0.99390. Thus compatibility is not the sole determinant of attainable persistent progress: the path fraction permitted by the retention budget also matters.

\begin{figure}[tbp]
\centering

\begin{minipage}[t]{0.47\textwidth}
\raggedright
{\bfseries\large a}\par
\vspace{0.15em}
\centering
\includegraphics[width=\linewidth,trim=0bp 0bp 0bp 26bp,clip]{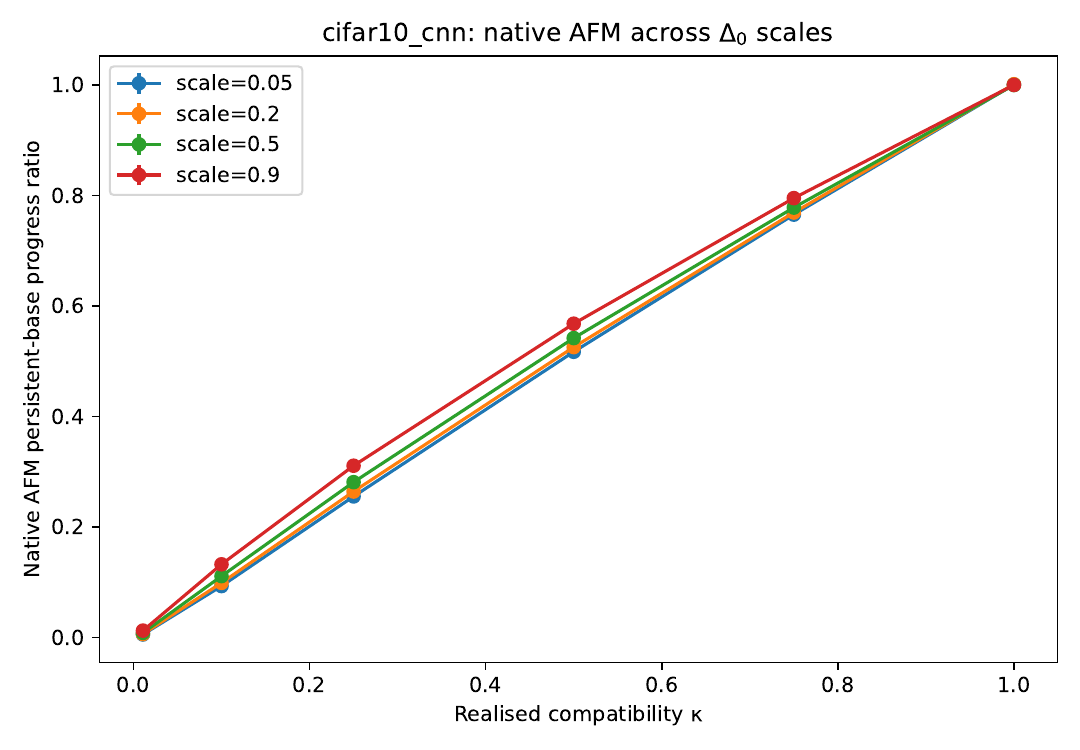}
\end{minipage}
\hfill
\begin{minipage}[t]{0.47\textwidth}
\raggedright
{\bfseries\large b}\par
\vspace{0.15em}
\centering
\includegraphics[width=\linewidth,trim=0bp 0bp 0bp 26bp,clip]{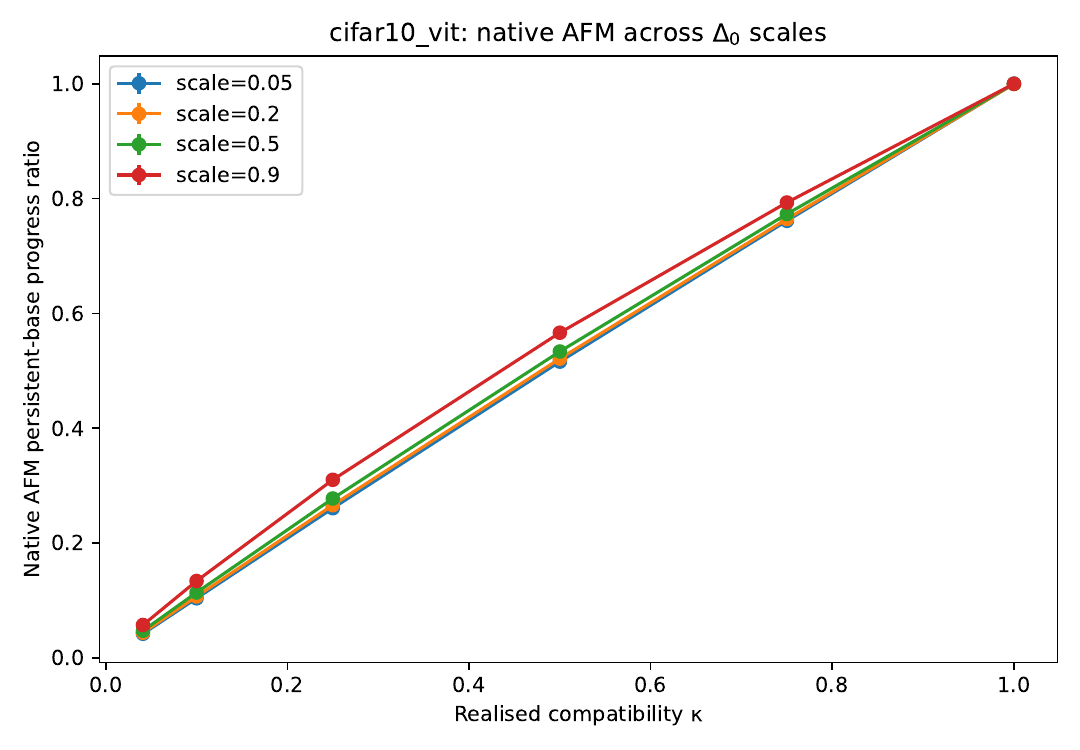}
\end{minipage}

\vspace{1em}

\begin{minipage}[t]{0.58\textwidth}
\raggedright
{\bfseries\large c}\par
\vspace{0.15em}
\centering
\includegraphics[width=\linewidth,trim=0bp 0bp 0bp 26bp,clip]{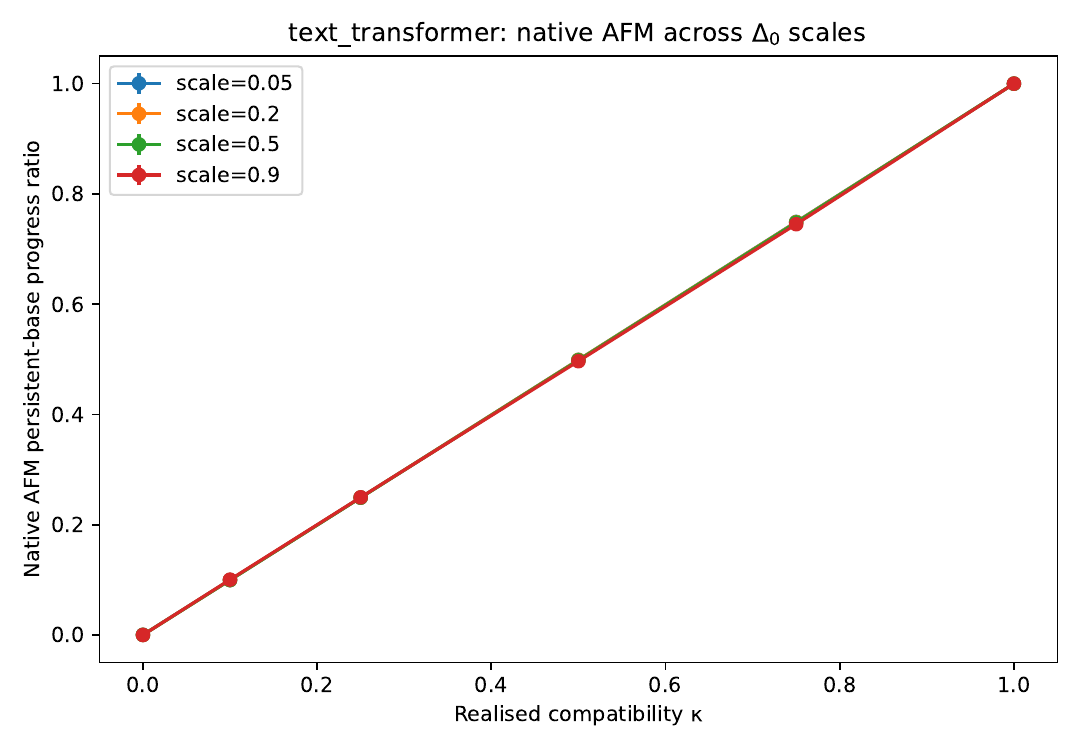}
\end{minipage}

\caption{\textbf{Native AFM multiscale dose responses.}
\textbf{a--c}, Native-AFM responses across the designed local multiscale suite for the
\textbf{a}, CIFAR-10 convolutional network,
\textbf{b}, CIFAR-10 vision transformer, and
\textbf{c}, text transformer.
Internal scale indices are logarithmic expansion coordinates and must not be interpreted as percentages of a natural update.}
\label{fig:ed-multiscale}

\end{figure}

\subsubsection{Scale calibration against natural updates}

A calibration compares the multiscale intervention norms with ordinary unrestricted update norms measured in the natural-state validation. Update norm is the meaningful cross-experiment scale quantity because the causal and natural experiments use different current objectives; their raw $\DeltaZero$ values are not directly commensurate.

At the largest multiscale index $s=0.9$, the causal updates were still much smaller than ordinary natural updates, especially in vision (Table~\ref{follow:tab:scale-calibration}).

\begin{table}[ht]
\centering
\caption{Largest designed multiscale intervention compared with the median natural unrestricted update norm.}
\label{follow:tab:scale-calibration}
\begin{tabular}{lrrr}
\toprule
System & Mean causal norm & Median natural norm & Causal/natural ratio \\
\midrule
CIFAR-10 CNN & $4.9766\times10^{-6}$ & 0.0605702 & $8.2163\times10^{-5}$ \\
CIFAR-10 ViT & $3.5026\times10^{-5}$ & 0.0391441 & $8.9479\times10^{-4}$ \\
Text transformer & $1.0863\times10^{-3}$ & 0.0146534 & 0.0741306 \\
\bottomrule
\end{tabular}
\end{table}

These correspond approximately to gaps of 12,170-fold for the CNN, 1,118-fold for the ViT, and 13.5-fold for the text transformer. The resulting scale gap motivates the fixed-norm natural-scale bridge in Section~\ref{follow:sec:bridge}.

\subsubsection{Interpretation}

The multiscale experiment extends robustness within a calibrated \emph{local} regime. The measured scale gap limits its implication for ordinary training updates.

\subsection{Experiment 2: independently constructed directions at fixed compatibility}
\label{follow:sec:directions}

\subsubsection{Intervention design}

The second stress test addressed a construction-specific confound: perhaps the original dose response was a property of the particular generalized-eigenmode direction used to realize each $\kappa$, rather than of compatibility itself. Four independently constructed directions were therefore generated for each nonzero requested compatibility level
\begin{equation}
    \kappa^\star\in\{0.1,0.25,0.5,0.75\}.
\end{equation}
Mode pairs bracketing each requested coordinate are shuffled with a seeded generator and mixed with both relative signs. Native gradients are recomputed; candidates must meet realized-compatibility tolerance, and nearly duplicate axes are excluded before selecting four diverse comparator-norm-matched directions. The directions share a parent geometry and selection constraints; ``independent'' denotes separate construction, not statistical independence. The experiment used the CNN, ViT, and text transformer; five seeds; all seven retention budgets; and seven method branches. Validation reported four directions per $\kappa$, 147,000 fixed groups, and 588,000 expected and observed frontier rows, with all 15 runs complete.

The principal diagnostic compares variability across independently generated directions at a fixed $\kappa$ with variability across $\kappa$. Let
\begin{equation}
R_{\mathrm{dir}/\kappa}
=\frac{K_\kappa^{-1}\sum_{k=1}^{K_\kappa}\operatorname{SD}_{d}(\rho_{k,d})}
{\operatorname{SD}_{k}(\bar\rho_k)},\qquad
\bar\rho_k=D^{-1}\sum_{d=1}^{D}\rho_{k,d},\quad K_\kappa=D=4.
\end{equation}
Here $k$ indexes compatibility levels, $d$ indexes constructed directions, and $\rho_{k,d}$ is their persistent-progress ratio. Both standard deviations use population normalization within the state. Undefined zero-denominator ratios are excluded with their counts recorded; defined state ratios are averaged within seed before seed-level inference. Small values indicate less variation across directions at fixed compatibility than across compatibility levels.

\subsubsection{Results}

Table~\ref{follow:tab:independent-directions} records the AFM/projection results across retention budgets. At nontrivial budgets, the matched compatibility slope approaches one and the within-$\kappa$ directional variation is generally far smaller than the between-$\kappa$ variation. The text transformer is the most budget-limited system at tight budgets, but it approaches the same near-unit regime as the retention budget is relaxed.

\begin{table}[ht]
\centering
\caption{Independent-direction robustness for AFM/projection. Each cell gives $R_{\mathrm{dir}/\kappa}$ / matched $\kappa$ slope.}
\label{follow:tab:independent-directions}\label{tab:ed-direction}
\tablefont
\begin{tabular}{lccc}
\toprule
$\beta$ & CIFAR-10 CNN & CIFAR-10 ViT & Text transformer \\
\midrule
0.01 & 0.3718 / 0.0725 & 0.2467 / 0.2740 & 0.4353 / $-0.0066$ \\
0.05 & 0.0922 / 0.6876 & 0.0409 / 0.8535 & 0.5321 / 0.0927 \\
0.10 & 0.0577 / 0.8917 & 0.0220 / 0.9566 & 0.3445 / 0.3367 \\
0.25 & 0.0400 / 1.0130 & 0.0254 / 1.0133 & 0.1151 / 0.8285 \\
0.50 & 0.0379 / 1.0465 & 0.0206 / 1.0294 & 0.0302 / 0.9861 \\
1.00 & 0.0347 / 1.0556 & 0.0205 / 1.0295 & 0.0178 / 1.0033 \\
\bottomrule
\end{tabular}
\end{table}

At $\beta=0.5$, the 95\% seed-level confidence intervals for the matched slope were [1.03046, 1.06024] in the CNN, [1.02516, 1.03442] in the ViT, and [0.97849, 0.99454] in the text transformer. The mean fraction of states whose within-$\kappa$ direction SD was below the between-$\kappa$ variation was 1.0 for all three systems at this budget.

\subsubsection{Interpretation and reporting limitation}

This experiment substantially weakens the explanation that the effect is tied to a single specially chosen generalized-eigenmode recipe. It does not imply that all directions sharing the same $\kappa$ are equivalent; rather, at the tested local scale and ordinary retention budgets, direction-to-direction variation is much smaller than the systematic variation induced by changing compatibility.

The independent-direction analysis does not assume exact equality of realized $\kappa$, update norm or $\DeltaZero$ across the four constructed directions. The reported inference is based on the observed within-compatibility directional variation relative to the systematic between-compatibility effect.

\subsection{Experiment 3: 10-seed generality and a stronger ViT}
\label{follow:sec:generality}

\subsubsection{Intervention design}

The generality experiment augments seeds 11, 29, 47, 71 and 101 with 131, 149, 167, 191 and 223 and adds a stronger CIFAR-10 ViT, yielding four systems:
\begin{enumerate}
    \item CIFAR-10 CNN,
    \item original CIFAR-10 ViT,
    \item stronger CIFAR-10 ViT,
    \item text transformer.
\end{enumerate}
The baseline and higher-capacity ViTs used the same $32\times32$ input and $4\times4$ patch size. The higher-capacity ViT increased embedding dimension from 48 to 96, depth from 3 to 6, and attention heads from 4 to 6; all remaining causal intervention and training configuration entries were unchanged. Uncertainty was summarized with a deterministic Monte Carlo seed bootstrap using 100,000 resamples.

\subsubsection{System-level results}

At $\beta=0.5$, all four AFM/projection matched slopes were strongly positive and near one (Table~\ref{follow:tab:generality-system}).

\begin{table}[ht]
\centering
\caption{Ten-seed AFM/projection matched slopes at $\beta=0.5$.}
\label{follow:tab:generality-system}
\begin{tabular}{lcc}
\toprule
System & Mean matched slope & 95\% CI \\
\midrule
CIFAR-10 CNN & 1.01544 & [1.01324, 1.01778] \\
CIFAR-10 ViT & 1.00064 & [1.00010, 1.00114] \\
Stronger CIFAR-10 ViT & 0.99883 & [0.99824, 0.99942] \\
Text transformer & 0.95950 & [0.95049, 0.96863] \\
\bottomrule
\end{tabular}
\end{table}

\subsubsection{Equal-system pooled dose response}

For each common seed, the four system-specific slopes are averaged with equal weights; bootstrap resampling then acts on these ten seed averages. The resulting pooled projection/AFM-base slope increased with retention budget (Table~\ref{follow:tab:generality-pooled}). This pattern is consistent with a compatibility signal whose exploitable persistent fraction is itself retention-budget dependent.

\begin{table}[ht]
\centering
\caption{Equal-system pooled AFM/projection slope across ten common seeds. Intervals use 100,000 deterministic bootstrap resamples of common seed indices after equal-system averaging.}
\label{follow:tab:generality-pooled}\label{tab:ed-tenseed}
\begin{tabular}{ccc}
\toprule
$\beta$ & Pooled slope & 95\% CI \\
\midrule
0.01 & 0.36331 & [0.35075, 0.37667] \\
0.05 & 0.77171 & [0.76572, 0.77759] \\
0.10 & 0.82114 & [0.81433, 0.82756] \\
0.25 & 0.93121 & [0.92425, 0.93841] \\
0.50 & 0.99360 & [0.99123, 0.99600] \\
1.00 & 1.00379 & [1.00326, 1.00442] \\
\bottomrule
\end{tabular}
\end{table}

\subsubsection{Interpretation}

The ten-seed experiment extends the dose response across two ViT configurations, a CNN and a text transformer. Increasing ViT capacity preserves the effect under the same intervention. These experiments evaluate compact models; behavior at foundation-model scale remains untested.

\subsection{Experiment 4: near-zero-compatibility boundary assessment}
\label{follow:sec:kzero}

\subsubsection{Boundary cases}

Among accepted native-AFM rows near requested $\kappa=0$, 65 have negative margins relative to $\lambdahat\kappa/3$, with minimum approximately $-0.003857$ (Table~\ref{tab:ed-zeroaudit}). Their certification status determines whether this empirical comparison tests the theorem under its stated assumptions.

\begin{table}[tbp]\centering\tablefont
\caption{Near-zero theorem-aligned boundary assessment}\label{tab:ed-zeroaudit}\label{follow:tab:kzero}
\begin{tabular}{lr}\toprule
System & Negative requested-zero empirical margins\\\midrule
CNN & 52\\
ViT & 1\\
Text transformer & 12\\\midrule
Total & 65\\\bottomrule
\end{tabular}
\vspace{1em}
\parbox{0.94\textwidth}{\tablefont Every row lacked certification of the relevant finite curvature or step condition. Numerical verification agreed with the reported $\kappa$ values to within $1.5883\times10^{-9}$ and with the persistent-ratio values exactly.}
\end{table}

All 65 negative-margin rows lack certification of the finite curvature or step conditions; numerical reconstruction agrees with the reported compatibility to within $1.5883\times10^{-9}$ and with the persistent ratio exactly. The diagnostic therefore does not establish a counterexample under the theorem's premises.

\subsection{Natural-update-scale bridge}
\label{follow:sec:bridge}

\subsubsection{Motivation}

The calibration in Section~\ref{follow:sec:multiscale} showed that the largest designed causal interventions remained far below ordinary unrestricted update norm in the two vision systems. The fixed-norm bridge therefore targets update norms defined directly as fractions of the median natural unrestricted update norm. Only the CIFAR-10 CNN and existing CIFAR-10 ViT are used, because these are the systems with the largest scale gap. Ten seed trajectories and 50 preselected states per seed are used; the states are fixed independently of the resulting feasibility outcomes.

The requested compatibility levels were
\begin{equation}
    \kappa^\star\in\{0.1,0.25,0.5,0.75\},
\end{equation}
With $R_{\mathrm{natural,median}}$ the architecture-specific median natural unrestricted update norm, the target norms were
\begin{equation}
    R^\star \in \{0.01,0.1,0.5,1.0\}\times R_{\mathrm{natural,median}}.
\end{equation}
The primary intervention therefore controls two quantities:
\begin{equation}
    \kappa\approx\kappa^\star,
    \qquad
    \|d\|=R^\star.
\end{equation}
For a state-scale condition to support a four-level within-state comparison, all four same-norm compatibility interventions must yield positive unrestricted finite progress,
\begin{equation}
    \DeltaZero(\kappa)>0
    \qquad \text{for all four tested }\kappa.
\end{equation}

\subsubsection{Admission criterion}

Because the bridge fixes the update norm while varying direction across requested compatibility levels, finite unrestricted progress $\DeltaZero$ is an outcome of the nonlinear loss surface. For each state and requested compatibility level, a finite bank of up to 64 candidate directions is generated. Candidates must satisfy the requested compatibility tolerance, the prescribed common update norm, numerical endpoint checks, and positive unrestricted finite progress. Update norms are Euclidean norms over flattened trainable parameters. For each candidate decrease used as an anchor, the selector takes the candidate nearest in log decrease at each compatibility level; it ranks the resulting sets by relative range, population coefficient of variation, then larger mean decrease. Relative range is $(\max\Delta_0-\min\Delta_0)/|\overline\Delta_0|$. This finite anchor-based search selects one direction per level without asserting a global optimum over all combinations. The resulting $\DeltaZero$ spread is reported diagnostically and is not used as the final admission criterion in the reported bridge analysis. Thus $\DeltaZero$ similarity affects direction selection within the feasible candidate bank but does not determine whether a preselected state-scale condition is admitted: feasibility requires at least one positive same-norm endpoint for every tested compatibility level.

Across 4,000 attempted state-scale conditions, 232 satisfied the positive-endpoint criterion. Validation reported:
\begin{itemize}
    \item 232 feasible state-scale conditions among 4,000 attempted conditions;
    \item maximum absolute compatibility error $9.656\times10^{-6}$;
    \item maximum relative update-norm error $3.674\times10^{-5}$;
    \item 14,848 expected and observed merged frontier rows;
    \item 928 expected and observed native-AFM rows;
    \item no validation failures.
\end{itemize}

\subsubsection{Feasibility}

The feasibility result is shown in Table~\ref{follow:tab:bridge-feasibility}. The ViT supports a substantial four-level comparison at 1\% of natural update norm: 231/500 states across all ten seeds. The CNN supports only one such state. At 10\%, 50\%, and 100\% of natural update norm, neither architecture has a state for which all four target-compatibility directions remain positive same-norm finite endpoints.

\begin{table}[ht]
\centering
\caption{Fixed-norm bridge feasibility.}
\label{follow:tab:bridge-feasibility}
\tablefont
\begin{tabular}{lrrrr}
\toprule
System & Natural-norm fraction & Target norm & Feasible/attempted & Seeds with any feasible \\
\midrule
CNN & 0.01 & 0.00060570 & 1/500 (0.2\%) & 1 \\
CNN & 0.10 & 0.00605702 & 0/500 & 0 \\
CNN & 0.50 & 0.03028509 & 0/500 & 0 \\
CNN & 1.00 & 0.06057018 & 0/500 & 0 \\
ViT & 0.01 & 0.00039144 & 231/500 (46.2\%) & 10 \\
ViT & 0.10 & 0.00391441 & 0/500 & 0 \\
ViT & 0.50 & 0.01957205 & 0/500 & 0 \\
ViT & 1.00 & 0.03914410 & 0/500 & 0 \\
\bottomrule
\end{tabular}
\end{table}

Among the 231 feasible ViT states at 1\%, the mean within-state $\DeltaZero$ coefficient of variation is 0.22586 and the maximum is 0.58485. This is direct evidence that finite unrestricted progress varies appreciably across same-norm compatibility directions once the intervention reaches this regime.

\subsubsection{Ratio and absolute-progress slopes}

In the 1\% ViT bridge, the raw fixed-norm $\kappa\mapsto\rhoPers$ slope for the projection/AFM-base frontier at $\beta=0.5$ is
\begin{equation}
    -0.57177, \qquad \CI=[-1.59,0.46],
\end{equation}
so the estimated $\rhoPers$ slope is negative, with an interval that includes zero. The absolute-progress analysis separates the numerator and denominator contributions to this ratio.

However, this fixed-norm intervention does \emph{not} hold $\DeltaZero$ constant. Because
\begin{equation}
    \rhoPers=\frac{\DeltaPers}{\DeltaZero},
\end{equation}
along any differentiable family with $\DeltaZero>0$, its derivative with respect to compatibility is
\begin{equation}
    \frac{d\rhoPers}{d\kappa}
    =
    \frac{
      \DeltaZero\,\frac{d\DeltaPers}{d\kappa}
      -\DeltaPers\,\frac{d\DeltaZero}{d\kappa}
    }{\DeltaZero^2}.
\end{equation}
A flat or negative ratio slope can therefore coexist with strongly increasing absolute persistent learning if the unrestricted denominator also increases with compatibility.

\subsection{Decomposition of the fixed-norm bridge}
\label{follow:sec:decomposition}

\subsubsection{Primary decomposition}

The bridge was decomposed into three within-state effects across the four compatibility interventions:
\begin{enumerate}
    \item $\kappa\mapsto\DeltaZero$,
    \item $\kappa\mapsto\DeltaPers$,
    \item $\kappa\mapsto\rhoPers$.
\end{enumerate}
The state-level slopes were averaged within seed and summarized across the ten ViT seeds.

For the projection/AFM-base frontier at 1\% natural update norm, let $\widehat b_Y$ denote the seed-averaged within-state slope of outcome $Y$ on compatibility. The estimates are:
\begin{align}
    \widehat b_{\DeltaZero} &= 7.78981\times10^{-5},\\
    \widehat b_{\DeltaPers} &= 3.06428\times10^{-4},
    \qquad \CI=[3.031,3.093]\times10^{-4},\\
    \widehat b_{\rhoPers} &= -0.57177,
    \qquad \CI=[-1.59145,0.45542].
\end{align}
The separately recorded native audit has mean persistent-update acceptance fraction 1.0 and finite-completion fraction 0.994118.

Table~\ref{follow:tab:bridge-decomp-methods} shows that the positive absolute-persistent effect is not restricted to one method, although its magnitude is method dependent.

\begin{table}[ht]
\centering
\caption{ViT at 1\% natural update norm, $\beta=0.5$: absolute persistent-progress and normalized-ratio slopes.}
\label{follow:tab:bridge-decomp-methods}
\tablefont
\begin{tabular}{lcc}
\toprule
Method & $\kappa\mapsto\DeltaPers$ slope [95\% CI] & $\kappa\mapsto\rhoPers$ slope [95\% CI] \\
\midrule
EWC-prox & $4.3693\times10^{-5}$ [3.5585, 5.1115]$\times10^{-5}$ & $-7.895$ [$-10.701$, $-5.197$] \\
Linearized distillation & $2.59767\times10^{-4}$ [2.54328, 2.65080]$\times10^{-4}$ & $-3.494$ [$-5.588$, $-1.502$] \\
Projection / AFM base & $3.06428\times10^{-4}$ [3.031, 3.093]$\times10^{-4}$ & $-0.572$ [$-1.591$, 0.455] \\
Unrestricted branch & $4.38287\times10^{-5}$ [3.55718, 5.12805]$\times10^{-5}$ & $-7.891$ [$-10.717$, $-5.190$] \\
\bottomrule
\end{tabular}
\end{table}

A separate regression that adjusts for realized finite $\DeltaZero$ yields a positive compatibility coefficient for the projection branch (about 2.694 with a 95\% interval of roughly [2.36, 3.08]). This regression is a sensitivity analysis conditional on $\DeltaZero$, a post-intervention quantity in the fixed-norm experiment, rather than the primary causal estimand.

\subsubsection{\texorpdfstring{$\DeltaZero$}{Delta0}-heterogeneity stratification}

The 231 feasible ViT states were divided into low-, medium-, and high-$\DeltaZero$-CV strata using the first and second tertile boundaries, 0.07620 and 0.33483. The absolute persistent-progress slope is essentially unchanged across all three strata, while the normalized ratio slope reverses sign in the high-heterogeneity stratum (Table~\ref{follow:tab:heterogeneity}).

\begin{table}[ht]
\centering
\caption{Projection branch at ViT 1\% natural update norm, stratified by $\DeltaZero$ heterogeneity. Values are identical across the tested projection retention budgets because the projection endpoint saturates to the same branch in this setting.}
\label{follow:tab:heterogeneity}
\begin{tabular}{lrrr}
\toprule
$\DeltaZero$-CV band & States & $\kappa\mapsto\DeltaPers$ slope & $\kappa\mapsto\rhoPers$ slope \\
\midrule
Low & 78 & $3.06394\times10^{-4}$ & 2.2463 \\
Medium & 77 & $3.06677\times10^{-4}$ & 3.1477 \\
High & 76 & $3.06452\times10^{-4}$ & $-7.5366$ \\
\bottomrule
\end{tabular}
\end{table}

For example, at $\beta=0.5$, the 95\% intervals for the absolute persistent-progress slope are approximately [3.03618, 3.09315]$\times10^{-4}$ in the low band, [3.02131, 3.12192]$\times10^{-4}$ in the medium band, and [2.99120, 3.12784]$\times10^{-4}$ in the high band. By contrast, the ratio-slope intervals are [1.945, 2.561], [2.192, 4.298], and [$-13.001$, $-3.271$], respectively.

\subsubsection{Interpretation}

Across denominator-heterogeneity strata, the absolute persistent slope remains nearly unchanged while the ratio slope reverses sign, consistent with the numerator--denominator decomposition.

This does not make $\rhoPers$ an invalid quantity in general. In the original $\DeltaZero$-matched experiment it answers a clean question: what fraction of an approximately equal unrestricted learning opportunity survives the retention constraint? In the fixed-norm bridge, by contrast, $\DeltaZero$ is itself an outcome of changing direction, so $\rhoPers$ is no longer a stable standalone causal summary.

\subsection{Finite-step feasibility boundary}
\label{follow:sec:boundary}

The natural-scale bridge also exposes a second phenomenon that is distinct from the ratio decomposition. Increasing the fixed update norm eventually destroys the ability to construct a common four-level continuum of positive-descent directions at the target compatibilities. No jointly feasible sampled four-level condition was found at 10\% of the median natural update norm or above in either tested vision system under the present direction construction and current objective.

The evidence distinguishes two regimes:
\begin{enumerate}
    \item \textbf{Where the four-level intervention is jointly feasible}, compatibility can be causally varied at fixed state and fixed update norm, and the ViT 1\% data show a strong positive effect on absolute persistent progress.
    \item \textbf{Where joint feasibility disappears}, there is no matched four-level causal slope to estimate. Those cases are not negative slopes; they are failures of the local compatibility continuum to remain a positive finite descent construction at that step magnitude.
\end{enumerate}

This boundary is itself informative. It separates the locally controlled functional geometry from a larger-step regime in which finite curvature and nonlinear loss-surface structure determine whether the nominally compatible direction remains a useful descent endpoint.

\subsection{Compatibility across directions and finite scales}
The four-direction experiment separates compatibility variation from dependence on one eigenmode construction. The multiscale and ten-seed studies extend that local result, with pooled slope 0.9936 at $\beta=0.5$. The near-zero audit locates all 65 negative empirical margins outside certified finite curvature/step conditions. These are complementary tests of the same local claim.

The fixed-norm bridge exposes a different regime. Across 231 feasible ViT states at 1\% natural update norm, the absolute persistent slope is $3.06428\times10^{-4}$ with a positive interval, whereas the normalized-ratio interval crosses zero. Similar absolute slopes across denominator-heterogeneity strata show why a changing unrestricted opportunity can destabilize the ratio. At larger target norms the tested four-level positive-endpoint construction becomes infeasible. Thus compatibility affects persistent learning where the intervention is realizable; the data do not establish a monotone normalized ratio at every finite scale or a four-level continuum at full natural update magnitude.

\subsection{Additional statistical considerations}

\begin{enumerate}
    \item \textbf{Seed as inferential unit.} The original matched-state, multiscale, and independent-direction analyses aggregate state-level effects within seed. The original matched-state intervals use two-sided 95\% Student-$t$ intervals across the five seed effects. The five-seed multiscale and independent-direction follow-ups use exact seed-bootstrap intervals obtained by enumerating all $5^5=3125$ ordered resamples. The generality and bridge analyses use ten seeds with 100,000 deterministic Monte Carlo bootstrap resamples where reported.
    \item \textbf{No state selection on large-scale success.} Bridge states are fixed before attempting large-scale interventions. Feasibility rates are reported over all attempted states rather than replacing failed states until a balanced successful sample is obtained.
    \item \textbf{Positive-endpoint conditioning.} A bridge slope is conditional on joint feasibility: all tested target-$\kappa$ directions must produce positive same-norm unrestricted progress. The unconditional feasibility rate is therefore a separate result and must accompany conditional slopes.
    \item \textbf{$\DeltaZero$ is post-intervention in the fixed-norm bridge.} Analyses that adjust for realized $\DeltaZero$ are decompositional/sensitivity analyses, not the primary causal estimand.
    \item \textbf{Fixed-norm inclusion and direction selection.} $\DeltaZero$ similarity is not the final admission variable: feasibility requires at least one positive same-norm unrestricted endpoint for every tested target compatibility and is reported over all 4,000 preselected state-scale conditions. Within a feasible condition, however, the reported four-direction set is selected from the finite candidate banks to minimize cross-$\kappa$ finite-$\DeltaZero$ imbalance. Thus $\DeltaZero$ enters within-condition direction selection but not the admission of the preselected parent state-scale condition.
\end{enumerate}

\section{Experimental organization and reproducibility}
The experiments link a controlled local intervention to natural learning and to complete chronological transactions. The causal and follow-up studies keep the base network trainable and restore the same parent state before each probe. The chronological comparisons freeze a shared representation after the initial supervised prefix, then compare methods on the same head and adapter coordinates. Method-native auxiliary state is retained, so these comparisons control the predictive architecture and stream rather than equalizing storage or computation.

The public repository at \url{https://github.com/hosseinjavidnia/afm/} provides configurations and reproduction instructions for the causal objective, matching procedure, retention statistic and proposal mechanisms defined in~\eqref{eq:causal-teacher-loss}--\eqref{causal:eq:budget} and Section~\ref{sec:causal-proposals}.

Each experiment preserves its inferential unit. Original causal intervals use five seed-mean slopes and a two-sided Student-$t$ interval with four degrees of freedom. Five-seed direction and multiscale extensions enumerate all $5^5$ seed-bootstrap resamples. Ten-seed generality and bridge intervals use 100,000 deterministic seed-bootstrap draws. State and method rows are repeated measurements within trajectories. Bridge feasibility uses all attempted states; conditional outcome summaries do not replace that denominator. Chronological family comparisons use the paired procedure defined with~\eqref{eq:metric-hv}.

The chronological evaluation combines compatible base adaptation and finite endpoint completion in the transaction of Figure~\ref{fig:afm-transaction}. Paired predictive scores measure its utility over the stream; ablations and endpoint measurements assess the contribution and execution of its components.

\section{Chronological AFM evaluation and benchmark evidence}
\subsection{Experimental evaluation}
\label{sec:experiments}

The experiments address three questions: whether the AFM transaction yields a useful retention-plasticity frontier on real chronological streams, whether that frontier remains competitive against recent continual-learning methods under a common predictive architecture, and whether the internal execution behaves consistently with the theorem-aligned quantities. The nonlinear experiments use complete empirical endpoint checks; they are not presented as outward-rounded pre-step numerical certificates of the full nonlinear theorem.

\subsubsection{Datasets, architecture, and protocol}

We evaluate on CORe50 \citep{lomonaco2017core50}, CLEAR-10 \citep{lin2021clear}, and CLAD-C \citep{verwimp2023clad}, the chronological object-classification component of CLAD built from labeled SODA10M imagery \citep{han2021soda10m}. Table~\ref{tab:protocols} summarizes the frozen protocols. All methods use the same compact convolutional predictive base and seed-specific initialization. A common representation prefix is learned and then frozen; after that boundary, all methods operate on the same available head and adapter coordinates. Non-oracle methods receive no task or boundary identifiers; the separately reported oracle EWC reference receives task information. No challenger receives an external pretrained representation in the controlled comparison.

\begingroup
\hfuzz=3pt
\begin{table}[t]
\centering
\caption{Experimental protocols. Non-oracle learners receive supervised labels but no task, session, episode, bucket, segment, intervention, or boundary identifiers. Oracle EWC is the explicitly task-aware reference.}
\label{tab:protocols}
\tablefont
\renewcommand{\arraystretch}{1.08}
\begin{tabularx}{\textwidth}{>{\raggedright\arraybackslash}p{0.14\textwidth}*{3}{>{\raggedright\arraybackslash}X}}
\toprule
 & CORe50 & CLEAR-10 & CLAD-C \\
\midrule
Dataset structure &
Official $128\times128$ CORe50 stream~\citep{lomonaco2017core50}; $10$ classes and $12$ predeclared episodes covering new contexts, recurrence, long-dormancy return, an explicit $0\leftrightarrow1$ target conflict, gradual visual drift, and capacity pressure. &
Official CLEAR-10 imagery~\citep{lin2021clear}; $11$ labels including background, $10$ chronological supervised buckets, and natural temporal drift. &
CLAD-C chronological object classification~\citep{verwimp2023clad} on labeled SODA10M~\citep{han2021soda10m}; $6$ classes and $6$ official chronological training segments. Only labeled object crops are used: no SODA10M unlabeled pretraining, no CLAD-D, and no external pretrained backbone. \\
Learner stream &
Batch size $32$; task/session/episode metadata withheld; evaluator-only semantic regimes and validity intervals. &
$32960$ learner examples, $3296$ per bucket, batch size $32$, $1030$ updates; bucket and period metadata withheld. &
$22249$ official object crops in six preserved segments with item counts 5157, 1154, 6742, 2560, 4517 and 2119; maximum batch size $10$ with partial boundary batches retained. Segment metadata and original-label side information remain evaluator-only; learners receive the supervised targets. \\
Representation prefix &
$20$ batches / $640$ images, then backbone freeze. &
Complete first supervised episode: $103$ batches / $3296$ images, then backbone freeze. Optional unlabeled bucket-$0$ pretraining is not used. &
Complete first official training segment: $516$ batches / $5157$ crops, then backbone freeze. \\
Candidate fitting &
$256$ routed examples, $20$ fixed epochs, validation horizon $2048$, risk threshold $0.7$. &
$256$ routed examples, $100$ fixed epochs, validation horizon $4096$, risk threshold $0.7$. &
$256$ routed examples, $100$ fixed epochs, validation horizon $4096$, risk threshold $0.7$. \\
Evaluation &
Checkpoint matrix over active semantic regimes; held-out test and complete original-semantics evaluator. &
Checkpoint at every bucket boundary; $1137$ validation and $4363$ held-out test examples from supervised buckets $1$-$10$; next-bucket evaluation for near-future accuracy. &
Checkpoint after every official training segment; the official data loader yields $32967$ validation and $46915$ held-out test object crops. \\
Matrix & \multicolumn{3}{>{\raggedright\arraybackslash}p{0.80\textwidth}}{$17$ variants per seed and $5$ seeds: $85$ jobs per dataset, $255$ jobs total across the three frozen matrices.} \\
\bottomrule
\end{tabularx}
\end{table}
\endgroup

AFM is evaluated at the predeclared coordinates $\eta\in\{0.10,0.50,1.00\}$. The 17 operating configurations are three AFM coordinates; one no-protection counterpart; SGD learning rates $0.0005,0.001,0.002$; replay and A-GEM memory sizes $32,128,512$ each; online-EWC strengths $1,10,100$; and task-aware oracle EWC with strength 10 \citep{rolnick2019replay,chaudhry2019agem,kirkpatrick2017ewc}. The no-protection counterpart runs the AFM trainer with protection disabled. Matched SGD directly installs a trust-capped scalar-gradient step on the same available post-prefix coordinates. Thus the two controls differ operationally despite sharing the absence of retention constraints. Five paired seeds are used throughout. Benchmark protocol specification and the separation between exploratory and confirmatory runs are documented in Section~\ref{app:complete-results}.

\subsubsection{Metrics and primary analysis}

For semantic regime $r$, let $A_{j,r}$ denote checkpoint-$j$ accuracy, with $i(r)$ its introduction checkpoint and $e(r)$ its final valid checkpoint. Forgetting and backward transfer are
\begin{equation}
F_r=\max_{i(r)\le j\le e(r)}A_{j,r}-A_{e(r),r},
\qquad
\operatorname{BWT}_r=A_{e(r),r}-A_{i(r),r}.
\label{eq:metrics-forgetting}
\end{equation}
Let $\mathcal R$ contain the evaluator's semantic regimes, each restricted to its valid checkpoint interval. The retention score is
\begin{equation}
R=\operatorname{clip}\!\left(1-\frac{1}{|\mathcal R|}\sum_{r\in\mathcal R}F_r,0,1\right).
\label{eq:metric-retention}
\end{equation}
Let $\mathcal E$ be the eligible adaptation episodes defined below. For nonempty $e\in\mathcal E$, order its $n_e$ recorded, metadata-matched prequential predictions by stream position and write their correctness indicators $c_{e,i}$. Set $q_e=\max(\lfloor n_e/4\rfloor,1)$ and
\[
P=\operatorname{clip}\!\left[\frac1{|\mathcal E|}\sum_{e\in\mathcal E}
\frac12\left(\frac1{q_e}\sum_{i=1}^{q_e}c_{e,i}
+\frac1{q_e}\sum_{i=n_e-q_e+1}^{n_e}c_{e,i}\right),0,1\right].
\]
Episodes receive equal weight; an empty episode contributes zero under the analysis convention. CORe50/CLEAR eligibility comprises initial, new-context, target-conflict, gradual-drift and capacity-pressure episode types; CLAD-C comprises initial and natural-temporal-shift types. Quarter membership is determined from the recorded item sequence, including retained partial batches. Online accuracy is prequential: for each incoming minibatch, predictions and minibatch accuracy are recorded before that minibatch is used for the learning update. The same prediction-before-update ordering is used during the initial representation-learning prefix. For a method family's set $\mathcal Q\subset[0,1]^2$ of $(R,P)$ operating points, the primary score is the origin-referenced dominated area,
\begin{equation}
\operatorname{HV}(\mathcal Q)=
\mu\!\left(\bigcup_{(R,P)\in\operatorname{ND}(\mathcal Q)}[0,R]\times[0,P]\right).
\label{eq:metric-hv}
\end{equation}
Here $\operatorname{ND}$ retains non-dominated operating points and $\mu$ is two-dimensional Lebesgue measure. Raw hypervolume is protocol-specific and is not interpreted as directly comparable across datasets. The paired primary statistic is the within-dataset difference between AFM and the strongest eligible non-oracle family. For the original five-seed primary comparisons, 95\% confidence intervals are paired percentile-bootstrap intervals over the five seedwise AFM-minus-comparator differences. We use 20,000 resamples with fixed RNG seed 20260729; each resample draws five paired differences with replacement and recomputes their mean, and the interval endpoints are the 2.5\% and 97.5\% order-statistic positions. These intervals are pointwise for the reported comparisons and are not adjusted for multiple comparisons or comparator selection. For subsequent five-seed ablation and modern-comparison analyses, the paired bootstrap is evaluated exactly by enumerating all $5^5=3125$ possible resamples and taking the 2.5th and 97.5th percentiles with linear interpolation.

\subsubsection{Classical comparison}

Table~\ref{tab:primary-hv} reports the confirmatory primary comparison. AFM exceeds the strongest eligible non-oracle family on every paired seed in all three protocols. The five-seed mean hypervolumes are $0.248$ on CORe50, $0.381$ on CLEAR-10, and $0.513$ on CLAD-C. The corresponding paired mean advantages are approximately $0.0090$, $0.0068$, and $0.0110$, and each paired confidence interval remains above zero.

\begin{table}[t]
\centering
\caption{Primary paired retention-plasticity hypervolume results across five seeds. Confidence intervals are paired across seeds; relative improvement is normalized by the comparator mean.}
\label{tab:primary-hv}
\tablefont
\setlength{\tabcolsep}{3.5pt}
\begin{tabular}{llrrrrrr}
\toprule
Dataset & Comparator & AFM HV & Comparator HV & Mean $\Delta$ & 95\% CI & Wins & Relative \\
\midrule
CORe50 & No protection & 0.2478 & 0.2388 & +0.0090 & [0.0059, 0.0126] & 5/5 & 3.76\% \\
CLEAR-10 & No protection & 0.3807 & 0.3740 & +0.0068 & [0.0053, 0.0082] & 5/5 & 1.81\% \\
CLAD-C & A-GEM & 0.5129 & 0.5019 & +0.0110 & [0.0054, 0.0166] & 5/5 & 2.20\% \\
\bottomrule
\end{tabular}
\end{table}

Table~\ref{tab:family-hv} places the result in the complete classical comparison matrix. No protection is the strongest classical non-oracle family on CORe50 and CLEAR-10, while A-GEM is the strongest on CLAD-C. The task-aware oracle EWC reference is not eligible for the primary non-oracle comparison.

\begin{table}[t]
\centering
\caption{Five-seed mean family hypervolume for AFM and the classical comparison families. Oracle EWC receives task information and is reported only as a task-aware reference.}
\label{tab:family-hv}
\tablefont
\begin{tabular}{lrrr}
\toprule
Family & CORe50 & CLEAR-10 & CLAD-C \\
\midrule
AFM & \textbf{0.2478} & \textbf{0.3807} & \textbf{0.5129} \\
No protection & 0.2388 & 0.3740 & 0.4872 \\
Matched SGD & 0.2031 & 0.3596 & 0.4714 \\
Replay & 0.1839 & 0.3503 & 0.4797 \\
A-GEM & 0.1845 & 0.3521 & 0.5019 \\
Online EWC & 0.1829 & 0.3511 & 0.4834 \\
Oracle EWC & 0.1823 & 0.3506 & 0.4478 \\
\bottomrule
\end{tabular}
\end{table}

AFM is a frontier rather than a single fixed operating point. Table~\ref{tab:operating-points} shows the expected movement from stronger preservation at $\eta=0.10$ toward greater acquisition at $\eta=1.00$. On CLEAR-10, all three coordinates reduce forgetting relative to no protection while maintaining similar plasticity. On CLAD-C, the conservative coordinate substantially reduces forgetting and improves final-test accuracy, whereas the full-assimilation coordinate approaches unrestricted acquisition and exhibits more peak-to-final forgetting. That reversal is examined at class level in Section~\ref{app:complete-results}.

\begin{table}[t]
\centering
\caption{Five-seed mean operating points. Fgt. denotes mean forgetting, BWT backward transfer, Test held-out final-test accuracy, Online chronological stream accuracy, and NF CLEAR-10 next-bucket accuracy.}
\label{tab:operating-points}
\tablefont
\begin{tabular}{llrrrrrrr}
\toprule
Dataset & Point & $R$ & $P$ & Fgt.$\downarrow$ & BWT$\uparrow$ & Test$\uparrow$ & Online$\uparrow$ & NF$\uparrow$ \\
\midrule
CORe50 & AFM 0.10 & 0.9648 & 0.2088 & 0.0352 & -0.0186 & 0.2291 & 0.2125 & - \\
CORe50 & AFM 0.50 & 0.9185 & 0.2458 & 0.0815 & -0.0658 & 0.2362 & 0.2594 & - \\
CORe50 & AFM 1.00 & 0.9020 & 0.2595 & 0.0980 & -0.0842 & \textbf{0.2398} & 0.2764 & - \\
CORe50 & No protection & 0.8866 & \textbf{0.2693} & 0.1134 & -0.0991 & 0.2341 & \textbf{0.2866} & - \\
\addlinespace
CLEAR-10 & AFM 0.10 & \textbf{0.9933} & 0.3750 & \textbf{0.0067} & \textbf{0.0056} & 0.3412 & 0.3753 & 0.3263 \\
CLEAR-10 & AFM 0.50 & 0.9828 & 0.3793 & 0.0172 & -0.0013 & \textbf{0.3417} & 0.3799 & 0.3309 \\
CLEAR-10 & AFM 1.00 & 0.9768 & \textbf{0.3834} & 0.0232 & -0.0065 & 0.3386 & \textbf{0.3835} & 0.3329 \\
CLEAR-10 & No protection & 0.9756 & 0.3834 & 0.0244 & -0.0073 & 0.3376 & 0.3834 & \textbf{0.3338} \\
\addlinespace
CLAD-C & AFM 0.10 & \textbf{0.8534} & 0.5706 & \textbf{0.1466} & -0.0081 & \textbf{0.5732} & 0.5713 & - \\
CLAD-C & AFM 0.50 & 0.7915 & 0.5874 & 0.2085 & 0.0293 & 0.5440 & 0.6025 & - \\
CLAD-C & AFM 1.00 & 0.7433 & 0.6040 & 0.2567 & \textbf{0.0382} & 0.5457 & 0.6155 & - \\
CLAD-C & No protection & 0.7685 & \textbf{0.6341} & 0.2315 & 0.0370 & 0.5497 & \textbf{0.6409} & - \\
\bottomrule
\end{tabular}
\end{table}

\subsubsection{Comparison with modern continual-learning methods}

The modern comparison includes OCAR \citep{urettini2025ocar}, LPR \citep{yoo2024lpr}, aL-SAR \citep{seo2025budgeted}, FGH \citep{michel2026fgh}, CCL-DC \citep{wang2024ccldc}, and MKD \citep{michel2024mkd}. These methods span curvature-aware replay, layerwise proximal replay, adaptive freezing and retrieval, learned hypergradient reweighting, collaborative distillation, and momentum-teacher distillation. Each challenger uses one frozen operating configuration, whereas AFM uses its three predeclared assimilation coordinates. The compared operating sets therefore differ in cardinality. The common predictive base, learner-visible chronological stream, representation-freeze rule, checkpoints, and seeds are held fixed. Method-native auxiliary state is retained rather than removed: for example, CCL-DC keeps its second learner and replay state, and MKD keeps its EMA teacher and replay state. The comparison is therefore controlled for the predictive base and information interface, with method-specific storage and computational costs. Detailed compatibility decisions are reported in Section~\ref{app:complete-results}.

Table~\ref{tab:ocar-primary} gives the dataset-level primary scores. These values compare AFM's declared operating family with the challengers' frozen method-native configurations; they are not estimates of the challengers' complete achievable hypervolumes under arbitrary hyperparameter sweeps. AFM has the highest mean primary score on CLEAR-10 and CLAD-C. FGH is the sole dataset-level mean reversal, exceeding AFM on CORe50 by about $2.8\%$ while operating at a more acquisitive point with lower retention and greater forgetting than every AFM coordinate. AFM exceeds FGH slightly on CLEAR-10 and substantially on CLAD-C.

\begin{table}[t]
\centering
\caption{Five-seed mean retention-plasticity hypervolume for AFM and six modern challengers. AFM contributes its predeclared three-point frontier; each challenger contributes the $RP$ area of its fixed operating configuration.}
\label{tab:ocar-primary}
\tablefont
\begin{tabular}{lrrrrrrr}
\toprule
Dataset &
AFM HV &
CCL-DC HV &
MKD HV &
FGH HV &
aL-SAR HV &
LPR HV &
OCAR HV \\
\midrule
CORe50
& 0.2478
& 0.2016
& 0.1799
& \textbf{0.2546}
& 0.1801
& 0.1640
& 0.1436 \\
CLEAR-10
& \textbf{0.3807}
& 0.3731
& 0.3513
& 0.3773
& 0.3572
& 0.3485
& 0.3164 \\
CLAD-C
& \textbf{0.5129}
& 0.4921
& 0.4931
& 0.3058
& 0.3296
& 0.4694
& 0.3747 \\
\bottomrule
\end{tabular}
\end{table}

For a compact description of conventional metrics, one AFM coordinate is selected per dataset using AFM's own five-seed mean final-test accuracy, independently of challenger identity: $\eta=1.00$ on CORe50, $\eta=0.50$ on CLEAR-10, and $\eta=0.10$ on CLAD-C. Table~\ref{tab:modern-besteta-aggregate} shows the equal-weighted descriptive summary. AFM has the highest aggregate plasticity, online accuracy, and primary score; CCL-DC has the highest final-test accuracy, OCAR the highest retention and lowest forgetting, and aL-SAR the highest mean BWT. The comparison therefore does not support a claim that AFM maximizes every scalar metric. Its advantage is a stronger joint acquisition-retention profile.

\begin{table}[t]
\centering
\caption{Equal-weighted descriptive summary of conventional metrics across the three datasets. AFM uses the coordinate with highest AFM five-seed mean final-test accuracy on each dataset; the Primary column remains the mean of the declared benchmark-level primary scores.}
\label{tab:modern-besteta-aggregate}
\tablefont
\begin{tabular}{lrrrrrrr}
\toprule
Method &
$R\uparrow$ &
$P\uparrow$ &
Fgt.$\downarrow$ &
BWT$\uparrow$ &
Online$\uparrow$ &
Test$\uparrow$ &
Primary$\uparrow$ \\
\midrule
\textbf{AFM}
& 0.9127
& \textbf{0.4031}
& 0.0873
& -0.0312
& \textbf{0.4092}
& 0.3849
& \textbf{0.3805} \\
CCL-DC
& 0.9424
& 0.3850
& 0.0576
& 0.0081
& 0.3955
& \textbf{0.3950}
& 0.3556 \\
MKD
& 0.9634
& 0.3604
& 0.0366
& 0.0175
& 0.3714
& 0.3916
& 0.3414 \\
LPR
& 0.9538
& 0.3512
& 0.0462
& 0.0202
& 0.3634
& 0.3799
& 0.3273 \\
FGH
& 0.8716
& 0.3592
& 0.1284
& -0.0215
& 0.3816
& 0.3591
& 0.3126 \\
aL-SAR
& 0.9625
& 0.3018
& 0.0375
& \textbf{0.0343}
& 0.3277
& 0.2934
& 0.2890 \\
OCAR
& \textbf{0.9952}
& 0.2799
& \textbf{0.0048}
& 0.0050
& 0.3041
& 0.2791
& 0.2782 \\
\bottomrule
\end{tabular}
\end{table}

Table~\ref{tab:modern-besteta-gainloss} quantifies that trade-off against each challenger. The closest aggregate primary challenger is CCL-DC: AFM is about $7.0\%$ higher in the descriptive primary score, with about $4.7\%$ greater plasticity and $3.5\%$ higher online accuracy, while CCL-DC retains higher mean retention and final-test accuracy. Against the six-challenger mean, AFM has a $19.9\%$ higher primary score, together with higher plasticity, online accuracy, and final-test accuracy, but lower mean retention.

\begin{table}[t]
\centering
\caption{AFM gain or loss relative to each modern challenger in the equal-weighted descriptive summary. R, P, Online, Test, and Primary are relative changes; Fgt. and BWT are AFM-minus-challenger percentage-point differences.}
\label{tab:modern-besteta-gainloss}
\tablefont
\resizebox{\textwidth}{!}{
\begin{tabular}{lrrrrrrr}
\toprule
Comparison &
$\Delta R$ &
$\Delta P$ &
$\Delta$Fgt. &
$\Delta$BWT &
$\Delta$Online &
$\Delta$Test &
$\Delta$Primary \\
\midrule
AFM vs CCL-DC
& -3.14\%
& +4.72\%
& +2.96 pp
& -3.93 pp
& +3.46\%
& -2.55\%
& \textbf{+7.01\%} \\
AFM vs MKD
& -5.26\%
& +11.86\%
& +5.07 pp
& -4.87 pp
& +10.19\%
& -1.70\%
& \textbf{+11.44\%} \\
AFM vs LPR
& -4.31\%
& +14.80\%
& +4.11 pp
& -5.14 pp
& +12.62\%
& +1.32\%
& \textbf{+16.26\%} \\
AFM vs FGH
& +4.72\%
& +12.24\%
& -4.12 pp
& -0.97 pp
& +7.22\%
& +7.19\%
& \textbf{+21.73\%} \\
AFM vs aL-SAR
& -5.17\%
& +33.58\%
& +4.98 pp
& -6.55 pp
& +24.88\%
& +31.20\%
& \textbf{+31.67\%} \\
AFM vs OCAR
& -8.29\%
& +44.00\%
& +8.25 pp
& -3.62 pp
& +34.54\%
& +37.93\%
& \textbf{+36.76\%} \\
\midrule
AFM vs mean of six challengers
& -3.74\%
& +18.72\%
& +3.54 pp
& -4.18 pp
& +14.53\%
& +10.08\%
& \textbf{+19.90\%} \\
\bottomrule
\end{tabular}}
\end{table}

\paragraph{Contemporary methods outside the common-architecture comparison.}
The common-base comparison holds the predictive architecture, representation source and learner-visible stream fixed while retaining method-native auxiliary state. Methods whose defining mechanism changes that predictive architecture or requires external pretraining address a different experimental comparison.

DEMD represents past experience through a dynamic memory distribution that expands, augments and reduces its structure as streaming novelty changes \citep{ye2025demd}. AdaLin adapts neuronal linearity to preserve gradient propagation and plasticity \citep{rohani2026adalin}. Both were considered for direct comparison, but a reproducible authors' implementation suitable for integration into the same-network protocol was unavailable when that protocol was frozen. Their published experiments use different datasets and evaluation procedures, so their reported accuracies are not pooled with the present benchmark scores.

SinglePrompt injects learned prompts into Transformer self-attention blocks \citep{park2026singleprompt}. The common convolutional predictor has no such blocks; adding them would replace the shared architecture, while substituting an adapter would change the method. Online-LoRA adapts an externally pretrained Vision Transformer \citep{wei2025onlinelora}. That representation source and architecture differ from the common compact network learned from the declared supervised prefix.

SERENA freezes specialized concept cells within an over-parameterized architecture \citep{yildirim2026serena}, while EG-CNN makes structural evolution of its feature extraction, refinement and classification components part of learning \citep{leite2026egcnn}. S6MOD adds a state-space-model branch and class-conditional routing after the backbone \citep{liu2025s6mod}. Dual-Arch assigns stability and plasticity to separate lightweight networks \citep{lu2025dualarch}. Constraining these methods to the unchanged common predictor would remove architectural mechanisms central to their published designs.

These methods remain relevant comparisons for an evaluation that treats architecture, pretraining, adaptive capacity and resource use as experimental variables. The present numerical comparison instead measures method performance under the declared common predictive-base interface.

\subsubsection{Ablation and theorem-aligned execution evidence}

The main ablation asks whether the joint AFM transaction improves on a persistent protected base update alone. Table~\ref{tab:afm-ablation-rp} reports the paired CORe50 result at $\eta=0.50$. Full AFM improves the single-point $R\times P$ score by about $18.6\%$ relative to both base-only variants, with positive paired intervals. In contrast, collapsing the multiscale trace bank to one timescale produces essentially the same aggregate score at this operating point. The latter result does not invalidate the multiscale theorem; it shows that this particular CORe50 coordinate does not empirically require the full timescale bank. Complete conventional metrics for the ablation are given in Section~\ref{app:mechanism-audits}.

\begin{table}[t]
\centering
\caption{Paired five-seed CORe50 $R\times P$ ablation results. Differences are full AFM at $\eta=0.50$ minus the indicated ablation.}
\label{tab:afm-ablation-rp}
\tablefont
\setlength{\tabcolsep}{3pt}
\begin{tabularx}{\textwidth}{>{\raggedright\arraybackslash}Xrrrrrr}
\toprule
Ablation &
Full AFM &
Ablation &
Mean $\Delta$ &
95\% CI &
Wins &
Relative gain \\
\midrule
Base-only, finite normalized budget
& 0.2257
& 0.1903
& \textbf{+0.0354}
& \textbf{[+0.0276,\,+0.0445]}
& \textbf{5/5}
& \textbf{+18.61\%} \\

Base-only, fixed-absolute budget
& 0.2257
& 0.1903
& \textbf{+0.0354}
& \textbf{[+0.0266,\,+0.0446]}
& \textbf{5/5}
& \textbf{+18.59\%} \\

Single-timescale AFM
& 0.2257
& 0.2257
& +0.0000
& [-0.0008,\,+0.0005]
& 4/5
& +0.01\% \\
\bottomrule
\end{tabularx}
\end{table}

Table~\ref{tab:afm-mechanism-audit} summarizes the execution-level checks over all protected AFM coordinates. Across the three datasets, accepted finite restorations keep the maximum endpoint error below $10^{-6}$. The minimum deployed progress ratio remains extremely close to one, and the realized persistent-base progress ratio exceeds the analytic reference expression on every accepted protected update examined. The near-one deployed ratio verifies execution of the finite endpoint-completion mechanism; it is not interpreted as evidence of population retention away from the finite shield support.

\begin{table}[t]
\centering
\caption{Execution-level mechanism audit over the three protected AFM coordinates and five seeds on each real-world benchmark. The analytic margin is the minimum realized persistent-base progress ratio minus the theorem expression, evaluated with the declared empirical smoothness input.}
\label{tab:afm-mechanism-audit}
\tablefont
\begin{tabular}{lrrr}
\toprule
Quantity & CORe50 & CLEAR-10 & CLAD-C \\
\midrule
Runs & 15 & 15 & 15 \\
Certifications & 90 & 99 & 488 \\
Commits & 90 & 99 & 488 \\
Protected nonzero updates & 8621 & 10416 & 22917 \\
Restorations accepted/attempted & 8636/8829 & 10431/10494 & 22932/24980 \\
Maximum endpoint error & $4.768\times10^{-7}$ & $9.537\times10^{-7}$ & $9.537\times10^{-7}$ \\
Minimum deployed ratio & 0.99999494 & 1.00000000 & 0.99993644 \\
Minimum analytic margin & \textbf{$+2.107\times10^{-3}$} & \textbf{$+2.940\times10^{-2}$} & \textbf{$+1.378\times10^{-4}$} \\
\bottomrule
\end{tabular}
\end{table}

\paragraph{Component-level computational cost.}
The present AFM implementation has substantial wall-clock overhead: median
runtime is approximately $35\times$ matched SGD on CORe50, $15\times$ on
CLEAR-10, and $184\times$ on CLAD-C. To identify the source of this cost,
we instrumented the existing implementation without changing the learning
rule, frozen hyperparameters, data streams, or operating-point selection.
Profiling used the conventional-metric operating
points, $\eta=1.0$ on CORe50, $\eta=0.5$ on CLEAR-10, and $\eta=0.1$ on
CLAD-C, over the same five frozen seeds used elsewhere in the study. For each profiled operation we recorded its accumulated host-observed elapsed
time, invocation count, and per-invocation duration, while recording
CUDA-event elapsed time separately. The default profiling mode deliberately
avoids per-region CUDA synchronization to limit instrumentation-induced
perturbation.

Table~\ref{tab:afm-runtime-profile} shows that the host-observed execution
cost is highly concentrated. The component-wise median shares of protection geometry,
finite counterfactual completion, and endpoint verification sum to
$94.82\%$ on CORe50, $91.70\%$ on CLEAR-10, and $95.78\%$ on CLAD-C.
Ordinary learning accounts for only $0.36\%$, $0.44\%$, and $0.24\%$,
respectively, while construction and evaluation of the same-state
no-protection comparator accounts for $0.52\%$, $0.27\%$, and $0.25\%$.
The host-timing decomposition is therefore concentrated in constructing and
verifying protected executable updates rather than in the ordinary
forward/backward regions or the counterfactual-comparator region. Because
CUDA execution is asynchronous in this profiling mode, these component
shares characterize host-observed execution time and are not interpreted as
a synchronized decomposition of device execution time.

\begin{table}[t]
\centering
\caption{Component-level host-timing decomposition of AFM at the conventional-metric
operating points. Entries are median percentages across the five
frozen seeds. Profiling uses low-perturbation host timers without per-region
CUDA synchronization; CUDA-event timing is recorded separately. Because each
component is summarized independently, columns need not sum exactly to
$100\%$.}
\label{tab:afm-runtime-profile}
\tablefont
\begin{tabular}{lrrr}
\toprule
Component
& CORe50
& CLEAR-10
& CLAD-C \\
\midrule
Ordinary learning
& 0.36\%
& 0.44\%
& 0.24\% \\

Same-state no-protection comparator
& 0.52\%
& 0.27\%
& 0.25\% \\

Protection geometry
& \textbf{37.33\%}
& 16.12\%
& \textbf{42.06\%} \\

Candidate machinery
& 1.94\%
& 2.17\%
& 2.00\% \\

Finite counterfactual completion
& \textbf{25.95\%}
& \textbf{22.90\%}
& \textbf{28.55\%} \\

Endpoint verification
& \textbf{31.54\%}
& \textbf{52.68\%}
& \textbf{25.17\%} \\

Routing/metaplastic/control
& 1.05\%
& 1.77\%
& 0.88\% \\

Other/unattributed
& 1.87\%
& 3.48\%
& 1.39\% \\
\bottomrule
\end{tabular}
\end{table}

The dominant category varies with the stream. Protection geometry is
largest on CORe50 at $37.33\%$ and on CLAD-C at $42.06\%$, whereas endpoint
verification dominates CLEAR-10 at $52.68\%$. Finite counterfactual
completion remains substantial on all three datasets, accounting for
$25.95\%$, $22.90\%$, and $28.55\%$, respectively. The operation-level
CLAD-C decomposition in Section~\ref{app:mechanism-audits} separates these
costs into invocation frequency and per-call duration and identifies the
principal implementation-level bottlenecks.

\subsubsection{Empirical scope and limitations}

On the three chronological protocols, AFM yields a positive paired frontier advantage over the strongest tested classical non-oracle family on all five seeds of each frozen protocol, and its primary performance remains competitive or superior across six recent challengers under the common predictive-base interface. These chronological protocols learn a common representation prefix and then freeze that representation, so the empirical results establish continual adaptation over the shared learned representation rather than unrestricted end-to-end continual representation learning. The experiments do not establish universal empirical dominance or strict outward numerical certification of the nonlinear theorem.

Two limitations are particularly informative. First, all protected CLAD-C runs report insufficient signature calibration for a positive route-identification conclusion, so CLAD-C supports the assimilation, restoration, transfer, and finite-protection mechanisms but not a positive routing claim. Second, the CLAD-C full-assimilation forgetting increase is concentrated in Pedestrian and reflects a population-level collapse that also occurs in the no-protection learner after a Pedestrian-free chronological interval. Exact finite Pedestrian evidence can remain protected while held-out population accuracy collapses. The classwise analysis and an inverse-frequency diagnostic are reported in Section~\ref{app:complete-results}.

\subsection{Complete experimental results and protocol details}
\label{app:complete-results}

\subsubsection{Benchmark protocol specification and analysis separation}

All reported five-seed benchmark results use protocols fixed before analysis of the corresponding benchmark outcomes. For CLEAR-10, protocol-development diagnostics conducted before the five-seed evaluation determined the shared representation prefix and candidate-fitting budget using candidate-side fitting and activation behavior only; held-out test accuracy and hypervolume were not used. The resulting protocol uses the complete first supervised chronological episode as the shared representation prefix and 100 candidate-fitting epochs. Protocol-development runs are excluded from all reported five-seed benchmark summaries. The risk threshold, counterfactual comparator, normalized assimilation rule, and finite restoration construction follow the specification given in this paper.

For CLAD-C, the study uses the chronological object-classification stream defined by CLAD \citep{verwimp2023clad} over labeled SODA10M data \citep{han2021soda10m}. Only labeled bounding-box crops are used; unlabeled pretraining and the detection task are excluded. The official loader version used here yields $22{,}249$ training crops, $32{,}967$ validation crops, and $46{,}915$ test crops; these are the exact cardinalities used in the reported experiments. The official ordering is preserved, and examples are not moved or discarded to match cardinalities from earlier benchmark releases. Segment identity and original-label side information remain evaluator-only. Loader and frontier integrity checks use a separate development seed that is excluded from all reported five-seed results; no CLAD-C test score or exploratory classwise diagnostic is used to tune AFM.

\paragraph{Finite-shield support and guards.}
The finite shield uses the predeclared compact-address support construction
of Section~\ref{app:full-specification}. The replay envelope
$\varepsilon_a$ was fixed at $10^{-8}$, the support multiplier
was fixed at $c_{\rm sup}=4$, and the label-free guard bank was bounded by
$Q_{\max}=640$ addresses. These quantities are shield
support parameters rather than task-routing thresholds. In exact arithmetic, an
accepted transaction reproduces the requested value at each constraint center; changing the support parameters does not define an
approximate center-matching tolerance. Instead, they determine the spatial
extent of the finite residual and whether the required center-center and
center-guard separation conditions are satisfied, thereby affecting
acceptance or obstruction of a shield transaction.

\subsubsection{Classical seedwise results and operating-point trade-offs}

Table~\ref{tab:per-seed-hv} gives the paired seedwise comparison underlying the classical primary analysis. The sign of the difference is positive for every seed in all three protocols.

\begin{table}[t]
\centering
\caption{Per-seed hypervolume for AFM and the strongest non-oracle comparator within each protocol. Hypervolume magnitudes are not compared across datasets.}
\label{tab:per-seed-hv}
\tablefont
\setlength{\tabcolsep}{3.5pt}
\begin{tabular}{rrrrrrrrrr}
\toprule
& \multicolumn{3}{c}{CORe50} & \multicolumn{3}{c}{CLEAR-10} & \multicolumn{3}{c}{CLAD-C} \\
\cmidrule(lr){2-4}\cmidrule(lr){5-7}\cmidrule(lr){8-10}
Seed & AFM & No prot. & $\Delta$ & AFM & No prot. & $\Delta$ & AFM & A-GEM & $\Delta$ \\
\midrule
11  & 0.2481 & 0.2403 & +0.0078 & 0.3852 & 0.3763 & +0.0089 & 0.5065 & 0.4961 & +0.0104 \\
29  & 0.2281 & 0.2205 & +0.0076 & 0.3693 & 0.3652 & +0.0041 & 0.5001 & 0.4951 & +0.0051 \\
47  & 0.2520 & 0.2371 & +0.0149 & 0.3836 & 0.3778 & +0.0058 & 0.5231 & 0.5198 & +0.0033 \\
71  & 0.2621 & 0.2586 & +0.0036 & 0.3752 & 0.3684 & +0.0069 & 0.5216 & 0.5010 & +0.0207 \\
101 & 0.2489 & 0.2377 & +0.0111 & 0.3903 & 0.3822 & +0.0081 & 0.5133 & 0.4977 & +0.0156 \\
\bottomrule
\end{tabular}
\end{table}

Table~\ref{tab:tradeoffs} expresses the operating-point trade-offs relative to no protection. The conservative AFM coordinate strongly reduces forgetting on all three datasets. Higher $\eta$ shifts the frontier toward acquisition, as intended.

\begin{table}[t]
\centering
\caption{AFM operating-point trade-offs relative to no protection. Test differences are percentage points; positive values favor AFM.}
\label{tab:tradeoffs}
\tablefont
\begin{tabular}{llrrr}
\toprule
Dataset & Coordinate & Forgetting reduction & Relative reduction & Test difference \\
\midrule
CORe50 & 0.10 & 0.0783 & 69.00\% & -0.499 \\
CORe50 & 0.50 & 0.0319 & 28.12\% & +0.205 \\
CORe50 & 1.00 & 0.0154 & 13.59\% & +0.565 \\
CLEAR-10 & 0.10 & 0.0177 & 72.66\% & +0.358 \\
CLEAR-10 & 0.50 & 0.0072 & 29.42\% & +0.408 \\
CLEAR-10 & 1.00 & 0.0012 & 5.06\% & +0.101 \\
CLAD-C & 0.10 & 0.0850 & 36.70\% & +2.356 \\
CLAD-C & 0.50 & 0.0231 & 9.96\% & -0.569 \\
CLAD-C & 1.00 & -0.0252 & -10.87\% & -0.394 \\
\bottomrule
\end{tabular}
\end{table}

For CLEAR-10 at $\eta=0.50$, the paired 95\% interval for the final-test difference relative to no protection is approximately $[+0.11,+0.77]$ percentage points. This interval is computed by exact paired bootstrap over the five seedwise AFM-minus-no-protection final-test differences: all $5^5=3125$ resamples with replacement are enumerated, each replicate is the mean of five resampled paired differences, and the 2.5\% and 97.5\% order-statistic endpoints are used without percentile interpolation. Next-bucket accuracy is about $0.29$ percentage points lower. Thus the balanced coordinate improves retained/final performance without dominating the unrestricted trajectory on every temporal generalization metric. On CLAD-C, A-GEM with memory 32 attains lower mean forgetting than AFM $\eta=0.10$ but also lower plasticity; the family hypervolume therefore evaluates the complete retention-plasticity trade-off rather than selecting the single lowest-forgetting operating point.

\subsubsection{Modern challenger compatibility}

All six modern challengers are run on the same chronological observations, common predictive base, representation-freeze boundary, checkpoint schedule, and five seeds as AFM. The method-native state required by each algorithm is retained. OCAR and LPR use replay reservoirs of 128 learner-visible examples. aL-SAR retains its larger 4000-example memory; CCL-DC and MKD retain 1000-example replay memories together with, respectively, a second collaborating learner and a complete EMA teacher. Consequently, this experiment controls architecture and learner-visible information but not total storage or compute.

Only interface adaptations required by the common base are made. OCAR's curvature preconditioning and LPR's proximal preconditioning act on the available affine adapter and classifier coordinates. FGH's architecture-agnostic gradient-reweighting and prototype mechanisms are applied to the same trainable path because its complete published prompt configurations require pretrained Transformer architectures that would violate the common-base condition \citep{michel2026fgh}. CCL-DC retains both collaborating peers and MKD retains the teacher; removing these components would change the methods \citep{wang2024ccldc,michel2024mkd}.

Two compatibility consequences are reported because they affect interpretation. First, aL-SAR's public coordinate-sampling rule selects $0.01\%$ of the weights of a layer after integer truncation. On the 4096-weight adapter matrices used here this gives zero sampled similarity coordinates, so retrieval on those blocks reduces to the method's frequency-balancing component. The rule is not altered to improve aL-SAR under the compact architecture \citep{seo2025budgeted}. Second, FGH's prototype state is indexed by numeric class label. The CORe50 protocol contains an explicit $0\leftrightarrow1$ target conflict, but FGH receives neither the conflict boundary nor evaluator-side semantic remapping. This preserves the same learner-visible information restriction as AFM.

All 90 modern-challenger runs completed with valid evaluation records. The frozen configurations are summarized in Table~\ref{tab:challenger-configurations}.

\begin{table}[t]
\centering
\caption{Frozen configurations used for the six modern challengers. Settings are method-native unless a shared benchmark value is stated.}
\label{tab:challenger-configurations}
\tablefont
\renewcommand{\arraystretch}{1.08}
\begin{tabularx}{\textwidth}{l >{\raggedright\arraybackslash}p{0.20\textwidth} >{\raggedright\arraybackslash}p{0.19\textwidth} >{\raggedright\arraybackslash}X}
\toprule
Method & Persistent auxiliary state & Optimization & Principal frozen settings \\
\midrule
OCAR & Replay memory, 128 examples & Shared learning rate $10^{-3}$ & Curvature/Fisher, damping, and replay rules retained \\
LPR & Replay memory, 128 examples & Learning rate $0.01$ & $\omega_0=1$, $\beta=1$, preconditioner refresh every 100 iterations \\
aL-SAR & Replay memory, 4000 examples & Adam, $3\times10^{-4}$ & Update batch 16, online factor $0.015625$, unfreeze rate $0.5$, temperature $0.125$, $k=4$, warm-up 50 \\
FGH & Prototype state & Adam, $5\times10^{-3}$ & Hypergradient rate 1, gradient-weight clamp 1000, one online epoch, prototype coefficient 1 \\
CCL-DC & Second learner, replay memory 1000 & AdamW, $5\times10^{-4}$, weight decay $10^{-4}$ & Replay batch 64, one memory iteration, distillation weight 2, temperature 4 \\
MKD & EMA teacher, replay memory 1000 & Adam, $5\times10^{-4}$ & Replay batch 64, distillation weight 5.5, temperature 4, EMA $0.01$, correction interval 10 \\
\bottomrule
\end{tabularx}
\end{table}

Table~\ref{tab:ocar-seedwise} gives the complete seedwise primary scores. Across the 90 matched seed-dataset-challenger comparisons shown in the table, AFM has the higher displayed primary score in 82, the lower displayed score in 7, and ties to the displayed precision in 1. At the dataset-mean level it is higher in 17 of 18 pairwise comparisons. Relative to CCL-DC, AFM's dataset-level mean primary scores are higher by approximately $22.9\%$, $2.1\%$, and $4.2\%$ on CORe50, CLEAR-10, and CLAD-C, respectively. The corresponding gains over MKD are $37.8\%$, $8.4\%$, and $4.0\%$; over OCAR, $72.6\%$, $20.3\%$, and $36.9\%$; over LPR, $51.1\%$, $9.3\%$, and $9.3\%$; and over aL-SAR, $37.6\%$, $6.6\%$, and $55.6\%$. FGH is the only dataset-level reversal: it is about $2.8\%$ higher on CORe50, whereas AFM is about $0.9\%$ higher on CLEAR-10 and $67.8\%$ higher on CLAD-C.

\begin{table}[t]
\centering
\caption{Per-seed retention-plasticity hypervolume for AFM and the six modern challengers. AFM contributes its three-point family hypervolume and each challenger its fixed-configuration $RP$ area.}
\label{tab:ocar-seedwise}
\tablefont
\setlength{\tabcolsep}{4pt}
\textbf{CORe50}\\[0.3em]
\begin{tabular}{rrrrrrrr}
\toprule
Seed & AFM & CCL-DC & MKD & FGH & aL-SAR & LPR & OCAR \\
\midrule
11&0.2481&0.2004&0.1854&0.2632&0.1875&0.1702&0.1483\\
29&0.2281&0.1904&0.1739&0.2452&0.1763&0.1618&0.1428\\
47&0.2520&0.2030&0.1827&0.2520&0.1878&0.1592&0.1345\\
71&0.2621&0.2030&0.1817&0.2620&0.1729&0.1672&0.1497\\
101&0.2489&0.2111&0.1758&0.2509&0.1761&0.1616&0.1427\\
\bottomrule
\end{tabular}
\vspace{0.7em}

\textbf{CLEAR-10}\\[0.3em]
\begin{tabular}{rrrrrrrr}
\toprule
Seed & AFM & CCL-DC & MKD & FGH & aL-SAR & LPR & OCAR \\
\midrule
11&0.3852&0.3698&0.3413&0.3832&0.3573&0.3499&0.3101\\
29&0.3693&0.3747&0.3492&0.3736&0.3455&0.3379&0.3181\\
47&0.3836&0.3763&0.3646&0.3847&0.3641&0.3584&0.3303\\
71&0.3752&0.3639&0.3431&0.3658&0.3558&0.3432&0.3111\\
101&0.3903&0.3807&0.3582&0.3794&0.3635&0.3531&0.3124\\
\bottomrule
\end{tabular}
\vspace{0.7em}

\textbf{CLAD-C}\\[0.3em]
\begin{tabular}{rrrrrrrr}
\toprule
Seed & AFM & CCL-DC & MKD & FGH & aL-SAR & LPR & OCAR \\
\midrule
11&0.5065&0.4970&0.4914&0.2995&0.2989&0.4611&0.3499\\
29&0.5001&0.4960&0.4765&0.3413&0.3136&0.4599&0.3490\\
47&0.5231&0.5035&0.4974&0.3062&0.3526&0.4826&0.4221\\
71&0.5216&0.4818&0.4867&0.2905&0.3300&0.4628&0.3419\\
101&0.5133&0.4820&0.5136&0.2913&0.3526&0.4804&0.4105\\
\bottomrule
\end{tabular}
\end{table}

Table~\ref{tab:modern-overall-descriptive} reports the equal-weighted mean of the three benchmark-level primary scores. It is included only as a cross-protocol descriptive summary because the raw hypervolume scale is dataset specific.

\begin{table}[t]
\centering
\caption{Equal-weighted mean of the three benchmark-level primary scores. This descriptive summary complements, but does not replace, the dataset-specific primary analysis.}
\label{tab:modern-overall-descriptive}
\tablefont
\begin{tabular}{lr}
\toprule
Method & Equal-weighted three-benchmark primary score \\
\midrule
\textbf{AFM} & \textbf{0.3805} \\
CCL-DC       & 0.3556 \\
MKD          & 0.3414 \\
LPR          & 0.3273 \\
FGH          & 0.3126 \\
aL-SAR       & 0.2890 \\
OCAR         & 0.2782 \\
\bottomrule
\end{tabular}
\end{table}

Table~\ref{tab:ocar-metrics} gives the complete conventional metrics for every AFM coordinate and modern challenger. It shows explicitly that several challengers exceed AFM on individual stability or final-accuracy measures even when AFM has the stronger joint primary profile.

\begin{table}[t]
\centering
\caption{Five-seed mean conventional continual-learning metrics for AFM and the six modern challengers. Fgt. denotes forgetting, BWT backward transfer, Test held-out final-test accuracy, and Online chronological stream accuracy.}
\label{tab:ocar-metrics}
\tablefont
\begin{tabular}{llrrrrrr}
\toprule
Dataset & Method &
$R\uparrow$ & $P\uparrow$ &
Fgt.$\downarrow$ & BWT$\uparrow$ &
Test$\uparrow$ & Online$\uparrow$ \\
\midrule
CORe50
& AFM 0.10
  & 0.9648 & 0.2088 & 0.0352 & -0.0186
  & 0.2291 & 0.2125 \\
& AFM 0.50
  & 0.9185 & 0.2458 & 0.0815 & -0.0658
  & 0.2362 & 0.2594 \\
& AFM 1.00
  & 0.9020 & 0.2595 & 0.0980 & -0.0842
  & 0.2398 & 0.2764 \\
& CCL-DC
  & 0.9756 & 0.2066 & 0.0244 & 0.0339
  & 0.2648 & 0.2240 \\
& MKD
  & 0.9952 & 0.1808 & 0.0048 & 0.0374
  & 0.2415 & 0.1929 \\
& FGH
  & 0.8755 & 0.2910 & 0.1245 & -0.1088
  & 0.2493 & 0.3087 \\
& aL-SAR
  & 0.9862 & 0.1827 & 0.0138 & 0.0392
  & 0.2338 & 0.1914 \\
& LPR
  & 0.9934 & 0.1651 & 0.0066 & 0.0273
  & 0.2091 & 0.1698 \\
& OCAR
  & 0.9991 & 0.1437 & 0.0009 & 0.0023
  & 0.1469 & 0.1474 \\
\addlinespace

CLEAR-10
& AFM 0.10
  & 0.9933 & 0.3750 & 0.0067 & 0.0056
  & 0.3412 & 0.3753 \\
& AFM 0.50
  & 0.9828 & 0.3793 & 0.0172 & -0.0013
  & 0.3417 & 0.3799 \\
& AFM 1.00
  & 0.9768 & 0.3834 & 0.0232 & -0.0065
  & 0.3386 & 0.3835 \\
& CCL-DC
  & 0.9906 & 0.3766 & 0.0094 & 0.0102
  & 0.3515 & 0.3769 \\
& MKD
  & 0.9931 & 0.3537 & 0.0069 & 0.0106
  & 0.3337 & 0.3509 \\
& FGH
  & 0.9698 & 0.3891 & 0.0302 & -0.0095
  & 0.3407 & 0.3911 \\
& aL-SAR
  & 0.9918 & 0.3602 & 0.0082 & 0.0109
  & 0.3350 & 0.3586 \\
& LPR
  & 0.9969 & 0.3496 & 0.0031 & 0.0167
  & 0.3414 & 0.3486 \\
& OCAR
  & 0.9989 & 0.3168 & 0.0011 & 0.0020
  & 0.3045 & 0.3173 \\
\addlinespace

CLAD-C
& AFM 0.10
  & 0.8534 & 0.5706 & 0.1466 & -0.0081
  & 0.5732 & 0.5713 \\
& AFM 0.50
  & 0.7915 & 0.5874 & 0.2085 & 0.0293
  & 0.5440 & 0.6025 \\
& AFM 1.00
  & 0.7433 & 0.6040 & 0.2567 & 0.0382
  & 0.5457 & 0.6155 \\
& CCL-DC
  & 0.8608 & 0.5716 & 0.1392 & -0.0198
  & 0.5685 & 0.5856 \\
& MKD
  & 0.9020 & 0.5467 & 0.0980 & 0.0044
  & 0.5995 & 0.5703 \\
& FGH
  & 0.7694 & 0.3974 & 0.2306 & 0.0538
  & 0.4872 & 0.4451 \\
& aL-SAR
  & 0.9095 & 0.3625 & 0.0905 & 0.0529
  & 0.3113 & 0.4330 \\
& LPR
  & 0.8710 & 0.5388 & 0.1290 & 0.0165
  & 0.5892 & 0.5717 \\
& OCAR
  & 0.9878 & 0.3793 & 0.0122 & 0.0106
  & 0.3857 & 0.4477 \\
\bottomrule
\end{tabular}
\end{table}

\subsubsection{CLAD-C Pedestrian analysis}

The full-assimilation coordinate has more peak-to-final forgetting than no protection on CLAD-C. Classwise checkpoint analysis localizes most of the difference to Pedestrian. As shown in Table~\ref{tab:cladc-pedestrian-forensics}, AFM $\eta=1.00$ reaches substantially higher Pedestrian accuracy before the third chronological segment, after which both AFM and no protection collapse to zero held-out Pedestrian accuracy at the same boundary. The corresponding margin shift is negative for both learners and predictions are taken over by Car or Truck. The larger AFM forgetting value therefore reflects a higher acquired peak before a collapse shared with the unrestricted learner, not a collapse unique to AFM.

\begin{table}[t]
\centering
\caption{Exploratory CLAD-C Pedestrian diagnostic from saved checkpoints, reported as five-seed means unless stated otherwise. This diagnostic is not part of the primary benchmark analysis, and no retraining is involved.}
\label{tab:cladc-pedestrian-forensics}
\tablefont
\begin{tabular}{lrr}
\toprule
Quantity & AFM $\eta=1.00$ & No protection \\
\midrule
Pedestrian peak accuracy at boundary 2 & 77.24\% & 60.34\% \\
Pedestrian accuracy at boundary 3 & 0.00\% & 0.00\% \\
Final Pedestrian accuracy & 1.58\% & 1.76\% \\
Pedestrian peak-to-final forgetting & 75.66\% & 58.59\% \\
Mean Pedestrian margin at boundary 2 & +0.482 & +0.053 \\
Mean Pedestrian margin at boundary 3 & -4.202 & -5.587 \\
Mean boundary-2-to-3 margin change & -4.684 & -5.640 \\
Dominant wrong class at boundary 3 & Car 3/5; Truck 2/5 & Car 5/5 \\
\bottomrule
\end{tabular}
\end{table}

Finite protection and population performance coexist in a way consistent with the theorem. In individual seeds, active protected Pedestrian evidence remains present while held-out Pedestrian accuracy falls to zero. The theorem protects the declared finite evidence, not the complete unseen class distribution.

A separate exploratory inverse-frequency diagnostic tested whether aggregate post-prefix class frequency alone explained the full-assimilation penalty. Four paired seeds completed. The mean excess forgetting of AFM $\eta=1.00$ over no protection increased from about $0.0269$ in the primary benchmark runs to about $0.0837$ under inverse-frequency weighting, giving a mean attenuation of approximately $-0.0568$ with paired bootstrap interval $[-0.1113,-0.0023]$. The weighting was itself extreme and degraded both learners. Inverse-frequency weighting therefore did not attenuate the penalty in this diagnostic. It does not establish temporal class absence as the unique cause; proving that stronger causal claim would require a dedicated reordering or replay intervention.

\subsubsection{Resource accounting}

The modern comparison preserves method-native auxiliary state. In particular, CCL-DC carries two complete learners and replay, while MKD carries the student, a complete EMA teacher, and replay. AFM carries candidate, projector, metaplastic, event, and finite-shield state. Runtime and memory are therefore interpreted together with predictive performance rather than normalized away. The present AFM implementation has substantial runtime overhead, especially on CLAD-C, as reported in the main text.

\subsection{Ablations, execution audit, and controlled mechanism tests}
\label{app:mechanism-audits}

\subsubsection{Transaction ablation}

The component ablation is performed at the predeclared CORe50 coordinate $\eta=0.50$. Because the normalized persistent update and finite residual are coupled inside the transaction, the experiment does not claim to isolate either component individually. Instead, full AFM is compared with two base-only protected transactions, one using the finite-horizon-normalized power budget and one using the constant absolute budget $b_0=0.05$, together with a single-timescale AFM variant. The first normalization concerns the time schedule, not the counterfactual charge. The single-timescale variant retains the complete transaction at $\eta=0.50$ and reduces the trace bank to one timescale. These are transaction-level interventions; their executable configuration changes are supplied with the public reproduction code. Table~\ref{tab:afm-ablation-metrics} gives the conventional metrics.

\begin{table}[t]
\centering
\caption{Five-seed mean CORe50 metrics for full AFM at $\eta=0.50$ and the targeted ablations. The final column is the single-operating-point $R\times P$ product, not the AFM family hypervolume.}
\label{tab:afm-ablation-metrics}
\tablefont
\resizebox{\textwidth}{!}{
\begin{tabular}{lrrrrrrr}
\toprule
Method &
$R\uparrow$ &
$P\uparrow$ &
Fgt.$\downarrow$ &
BWT$\uparrow$ &
Online$\uparrow$ &
Test$\uparrow$ &
$R\times P\uparrow$ \\
\midrule
\textbf{Full AFM, $\eta=0.50$}
& 0.9185
& \textbf{0.2458}
& 0.0815
& -0.0658
& \textbf{0.2594}
& \textbf{0.2362}
& \textbf{0.2257} \\

Base-only, finite normalized budget
& \textbf{0.9984}
& 0.1906
& \textbf{0.0016}
& -0.0005
& 0.1885
& 0.2007
& 0.1903 \\

Base-only, fixed-absolute budget
& 0.9977
& 0.1908
& 0.0023
& -0.0004
& 0.1886
& 0.2016
& 0.1903 \\

Single-timescale AFM
& 0.9205
& 0.2452
& 0.0795
& -0.0638
& 0.2589
& 0.2354
& 0.2257 \\
\bottomrule
\end{tabular}}
\end{table}

The base-only variants obtain near-perfect retention by suppressing acquisition. Relative to the normalized base-only transaction, full AFM increases plasticity by approximately $0.0551$, online accuracy by $0.0708$, and held-out final-test accuracy by $0.0354$; the corresponding paired intervals remain positive. The two base-only budget parameterizations are almost indistinguishable in $R\times P$, so the ablation supports the importance of completing the deployed AFM transaction beyond the persistent base step, not a claim that one budget parameterization alone explains the gain. The single-timescale result is nearly identical to full AFM at this
coordinate, indicating that the observed CORe50 aggregate performance is not
sensitive to the multiscale bank in this particular setting. Theorems~\ref{thm:multiscale}
and~\ref{thm:single} concern worst-case representation of temporal relevance
over a long age range with an approximately scale-free profile. A finite
stream need not induce that profile or make its approximation affect accuracy. A practical multiscale advantage is therefore expected
when protected relevance spans sufficiently separated ages and the resulting
allocation decisions are sensitive to those temporal weights. The present
CORe50 coordinate does not provide empirical evidence that this regime is
active.

\subsubsection{Execution-level audit}

Table~\ref{tab:ed-ablation} combines the CORe50 ablation with the execution audit. Across 45 protected benchmark runs, 677 certifications and commits accompany $41{,}954$ nonzero protected base updates and $41{,}999$ accepted restorations from $44{,}303$ attempts; failed attempts are rejected. Maximum accepted endpoint error is below $10^{-6}$ on every dataset. Path-fraction shortfalls are at roundoff scale, and persistent-base progress exceeds the analytic reference expression on every accepted protected update examined.

\begin{table}[t]\centering\tablefont
\caption{AFM transaction ablation and execution validation}\label{tab:ed-ablation}
\textbf{Retention-plasticity summary}\par\smallskip
\begin{tabularx}{0.98\textwidth}{>{\raggedright\arraybackslash}X rrr}\toprule
Method & $R$ & $P$ & $R\times P$\\\midrule
Full AFM, $\eta=0.50$ & 0.9185 & \textbf{0.2458} & \textbf{0.2257}\\
Base-only, normalized budget & \textbf{0.9984} & 0.1906 & 0.1903\\
Base-only, fixed absolute budget & 0.9977 & 0.1908 & 0.1903\\
Single-timescale AFM & 0.9205 & 0.2452 & 0.2257\\\bottomrule
\end{tabularx}
\vspace{0.8em}
\textbf{Forgetting and predictive metrics}\par\smallskip
\begin{tabularx}{0.98\textwidth}{>{\raggedright\arraybackslash}X rrrr}\toprule
Method & Forgetting & BWT & Online & Test\\\midrule
Full AFM, $\eta=0.50$ & 0.0815 & -0.0658 & \textbf{0.2594} & \textbf{0.2362}\\
Base-only, normalized budget & \textbf{0.0016} & -0.0005 & 0.1885 & 0.2007\\
Base-only, fixed absolute budget & 0.0023 & -0.0004 & 0.1886 & 0.2016\\
Single-timescale AFM & 0.0795 & -0.0638 & 0.2589 & 0.2354\\\bottomrule
\end{tabularx}
\vspace{1em}
\begin{tabularx}{0.98\textwidth}{>{\raggedright\arraybackslash}X >{\raggedright\arraybackslash}p{0.30\textwidth}}\toprule
Execution-level quantity across 45 protected benchmark runs & Recorded value\\\midrule
Certifications and commits & 677\\
Nonzero protected base updates & 41,954\\
Finite restorations accepted / attempted & 41,999 / 44,303\\
Maximum accepted finite endpoint error & $<10^{-6}$ on every dataset\\
Accepted protected updates below analytic reference & 0\\
CLAD-C tighter finite-smoothness diagnostic exceedances & 11 / 22,932 accepted restoration events\\\bottomrule
\end{tabularx}
\vspace{1em}
\parbox{0.94\textwidth}{\tablefont The ablation uses $n=5$ frozen CORe50 seed trajectories at the predeclared $\eta=0.50$ coordinate. Base-only variants attain near-perfect retention by suppressing acquisition; full AFM increases plasticity, online accuracy and held-out final-test accuracy. The single-timescale variant has nearly identical aggregate performance on this finite stream. Failed finite-restoration attempts are rejected.}
\end{table}

Let $s=\widehat\lambda s^0$ be the realized base-step length, $q=\Pi g$, and $L$ the run's declared smoothness input. The execution audit distinguishes
\[
B_{\rm analytic}=\widehat\lambda\kappa
\frac{\max(1-\alpha L\widehat\lambda/2,0)}{1+\alpha L/2},\qquad
B_{\rm tight}=\frac{\max(s\norm q-Ls^2/2,0)}{\Delta_0}.
\]
The logged analytic margin is $\rho_{\rm persistent}-B_{\rm analytic}$; the tighter diagnostic uses measured $\Delta_0$ directly. The tighter expression uses the declared empirical smoothness input; it is not an outward-certified lower bound in these neural runs. A ratio exceedance is $B_{\rm tight}-\rho_{\rm persistent}$; its absolute counterpart is the numerator minus the observed safe-base decrease. On CLAD-C it exceeds the realized decrease in 11 of $22{,}932$ accepted restoration events, with maximum ratio discrepancy $3.8\times10^{-4}$ and maximum absolute loss-decrease discrepancy $9.1\times10^{-6}$. The analytic reference margin remains positive in all 11 events. No corresponding negative tight-diagnostic events occur on CORe50 or CLEAR-10.

The routing audit also matches the declared information conditions. No protected CORe50 or CLEAR-10 run reports the calibration obstruction, whereas all protected CLAD-C runs do. The implementation therefore withholds the positive route-identification conclusion on CLAD-C. No evaluator-side segment information is introduced.

\subsubsection{Runtime bottleneck analysis}

The component-level decomposition in
Table~\ref{tab:afm-runtime-profile} identifies protection geometry, finite
counterfactual completion, and endpoint verification as the principal
runtime categories. CLAD-C provides the clearest operation-level
decomposition because it also has the largest overall runtime overhead.
Table~\ref{tab:afm-runtime-profile-cladc} therefore separates invocation
frequency from per-call cost for this protocol.

\begin{table}[t]
\centering
\caption{Invocation frequency and host-observed per-call cost of the principal
AFM operations on CLAD-C at $\eta=0.1$. Values are medians across the five
frozen seeds. Host timing uses the low-perturbation profiling mode without
per-region CUDA synchronization; CUDA-event timing is recorded separately.
The final column reports the median share of profiled host-observed elapsed
time.}
\label{tab:afm-runtime-profile-cladc}
\tablefont
\resizebox{\textwidth}{!}{
\begin{tabular}{lrrr}
\toprule
Operation
& Median calls
& Median host-observed time/call
& Host-timing share \\
\midrule
Protected-projector construction
& 1,711
& 2861.552 ms
& \textbf{37.89\%} \\

Compact-cardinal shield construction
& 1,526
& 1611.751 ms
& \textbf{20.58\%} \\

Safe-base backtracking
& 1,665
& 613.933 ms
& 8.60\% \\

Finite-address preparation
& 1,665
& 601.503 ms
& 8.08\% \\

Protected prestate evaluation
& 1,711
& 599.454 ms
& 8.06\% \\

Protected poststate evaluation
& 1,664
& 599.695 ms
& 8.05\% \\

Frontier evaluation
& 25,665
& 16.670 ms
& 3.46\% \\

Candidate fitting
& 37
& 4846.976 ms
& 1.86\% \\

Sketch update
& 142
& 356.499 ms
& 0.47\% \\

Post-restoration re-evaluation
& 1,526
& 28.977 ms
& 0.39\% \\

Comparator endpoint/backtracking
& 1,665
& 15.096 ms
& 0.24\% \\

Ordinary backward pass
& 1,711
& 4.122 ms
& 0.07\% \\

Ordinary forward pass
& 1,711
& 3.930 ms
& 0.07\% \\
\bottomrule
\end{tabular}}
\end{table}

Protected-projector construction is the largest individual component of the
CLAD-C host-timing profile, accounting for $37.89\%$ of measured
host-observed time. It combines a median of 1,711 invocations with an
approximately $2.862$\,s host-observed duration per invocation.
Compact-cardinal shield construction is the second largest component,
accounting for $20.58\%$ across 1,526 median invocations at approximately
$1.612$\,s per call. Safe-base backtracking, finite-address preparation,
and protected prestate and poststate evaluations each account for
approximately $8$-$9\%$ of the host-timing profile.

The invocation analysis distinguishes operations that are individually
expensive in the host-timing profile from operations that dominate aggregate
execution. Candidate fitting has the largest median host-observed duration
per invocation among the listed operations, approximately $4.847$\,s, but
occurs only 37 times and therefore contributes $1.86\%$. Frontier evaluation
occurs 25,665 times but has a median host-observed duration of only
$16.67$\,ms per call and contributes $3.46\%$. Protected-projector
construction is therefore the largest host-visible implementation bottleneck
because substantial per-call duration is combined with repeated invocation.

These host-timing measurements identify protected-projector construction,
compact-cardinal shield construction, finite-address preparation, and
protected-state evaluation as the principal host-visible targets for
implementation optimization. Such optimization can focus on reducing
repeated construction and reevaluation costs without changing the same-state
comparator, persistent-assimilation rule, or finite endpoint constraints that
define the AFM transaction. The separately recorded CUDA-event timings provide
the corresponding device-side diagnostic.

\paragraph{Profiling validity.}
The profiling instrumentation does not modify the AFM learning rule, frozen
hyperparameters, prepared streams, or operating-point selection. The default
profiler uses low-perturbation host timing and deliberately avoids
per-region CUDA synchronization; CUDA-event elapsed times are recorded
separately. This design reduces synchronization-induced perturbation while
requiring the host-region percentages to be interpreted as a decomposition
of host-observed execution rather than synchronized device execution.
CORe50 and CLEAR-10 reproduce the corresponding unprofiled metrics exactly.
CLAD-C uses the same AFM configuration and byte-identical prepared training
and evaluation streams, but its frozen execution mode is nondeterministic.
Independent unprofiled CLAD-C executions with the same nominal seed and AFM
settings likewise produce distinct numerical trajectories. The profiling
results therefore characterize execution of the frozen AFM algorithm rather
than reproduction of one particular numerical realization.

\subsubsection{Controlled mechanism tests}

The controlled experiments exercise recurrence, semantic conflict, observable route change, capacity pressure, and observational indistinguishability. Table~\ref{tab:afm-controlled-mechanisms} summarizes the results. The within-route observable-shift condition produces a route split on every seed. Semantic conflict produces reopening but no unsupported split, and the identical-observation condition remains unresolved, consistent with the absence of learner-observable distinguishing information.

\begin{table}[t]
\centering
\caption{Controlled AFM mechanism tests on three seeds. These experiments test recurrence, conflict, route refinement, capacity pressure, and observational indistinguishability rather than external predictive performance.}
\label{tab:afm-controlled-mechanisms}
\tablefont
\renewcommand{\arraystretch}{1.08}
\begin{tabularx}{\textwidth}{p{0.22\textwidth}c p{0.28\textwidth} >{\raggedright\arraybackslash}X}
\toprule
Scenario & Seeds passed & Structural events & Observed behavior \\
\midrule
Favorable recurrence & 3/3 & commits 6/6/6; protected steps 20/13/10 & recurrent protected learning completed \\
Semantic conflict & 3/3 & reopenings 2/3/1; splits 0/0/0 & conflict handled without an unsupported route split \\
Observable shift & 3/3 & commits 3/3/3; protected steps 41/27/27 & observable-shift control completed \\
Within-route observable shift & 3/3 & splits 1/1/1 & positive route splitting exercised on every seed \\
Capacity pressure & 3/3 & commits 11/11/11; protected steps 30/18/21 & bounded-capacity control completed \\
Identical-observation impossibility & 3/3 & commits 3/3/3; protected steps 12/8/11 & indistinguishable observations did not induce a false semantic resolution \\
\bottomrule
\end{tabularx}
\end{table}

The transfer-conflict diagnostic did not instantiate the required theorem precondition: no transfer attempt occurred on any of its three seeds, so the designed zero-dimensional feasible subspace was not exercised. It is therefore non-informative for the transfer theorem and is reported only to delimit the empirical coverage of the mechanism tests.

\section{Discussion}
\label{sec:discussion}
Functional compatibility is an experimentally controllable property of incoming learning. Holding the neural state, unrestricted finite opportunity and retention allowance fixed makes its effect on persistent learning identifiable within the admitted local regime. The same-state comparator supplies the relevant denominator: the accepted update of the declared learner without protection, rather than a nominal gradient step. The theorem separates the requested charge fraction $\eta_t$ from compatibility $\kappa_t$, and the experiments show that learning rules differ in how effectively they convert that geometry into persistent progress. This perspective complements replay, regularization, protected-subspace and plasticity-preservation approaches \citep{schwarz2018progresscompress,dohare2024plasticity,stork2026interference,abbes2026gradient}.

Persistent assimilation and finite deployed completion are distinct achievements. The residual can reproduce the unrestricted current endpoint while restoring protected observations, but it does not install the unrestricted parameter trajectory. Outside its support, the protected base determines future predictions. The transaction ablation tests the practical consequence of coupling these operations: the base-only variants preserve almost everything by suppressing acquisition, whereas full AFM acquires more and improves predictive performance. The design supports that transaction-level conclusion without separately identifying a gain from the normalized budget or residual alone.

Finite protection also differs from population retention. The CLAD-C Pedestrian result demonstrates that preserving the declared protected observations does not prevent interference on unseen members of the same class. Population claims require the sampling and transport certificates stated in the theory. Observable routing has a parallel information boundary: outcome evidence can justify reopening, while a new observable route requires distinguishable signature evidence. The calibration obstruction therefore withholds a positive routing conclusion on CLAD-C rather than admitting evaluator boundaries into the learner.

The finite-norm results delimit the causal effect. The persistent ratio is not a universal function of compatibility: retention allowance, proposal mechanism, unrestricted progress and curvature all matter. At fixed norm, unrestricted decrease becomes an intervention outcome, so an unstable ratio can coexist with a positive absolute persistent effect. The present candidate-bank construction does not establish the four-level continuum at ordinary full update magnitude. Likewise, the 65 negative near-zero empirical margins occur where finite curvature or step conditions were not certified (Table~\ref{tab:ed-zeroaudit}); they are empirical boundary observations rather than certified theorem counterexamples.

Computational cost remains substantial. Table~\ref{tab:afm-runtime-profile} and the CLAD-C profile identify protection geometry, repeated projector construction, compact-cardinal construction and endpoint verification as the dominant host-observed costs. Reducing redundant geometry and address preparation offers concrete optimization targets, provided the same-state comparator, protected transaction and persistent/deployed distinction are preserved. Bounded structural state and explicit precision do not imply low wall-clock cost.

The convex regret result requires an exact projection oracle; neural executions use empirical endpoint checks. Approximate arithmetic incurs the declared error charges, reopening retains its detector-history cost, and functional renewal requires nondegeneracy beyond gate motion. These conditions connect the theoretical guarantees to the evaluated transactions.

The controlled evidence identifies functional compatibility as a causal determinant of locally feasible persistent learning. Retention constraints and update mechanisms govern how much of that opportunity is realized, and finite nonlinear geometry governs how far the intervention extends. AFM provides a constructive way to measure these effects while accounting separately for what enters persistent parameters and what is restored on finite deployed evidence.

\section*{Acknowledgements}
The author thanks the ADAPT Centre for access to high-performance computing resources used for the experimental runs. The ADAPT Centre had no role in conceptualization, mathematical or methodological development, software development, study design, data analysis, interpretation or manuscript preparation.

\section*{Funding}
The author received no specific funding for this work.

\section*{Author contributions}
H.J. conceived the study, developed the theory and methodology, implemented the software, designed and ran the experiments, performed the formal and statistical analyses, validated the results, and wrote and revised the manuscript.

\section*{Competing interests}
The author declares no competing interests.

\section*{Data availability}
All datasets used in this study are publicly available from the original sources cited in this paper. No new primary dataset was generated. The public benchmark datasets are not redistributed with this submission. Numerical summaries underlying the reported figures and tables are provided throughout this paper. Saved checkpoints and row-level derived experimental outputs required for additional verification are available from the corresponding author on reasonable request. The released analysis and experiment code described under Code availability provides the procedures and configurations used to reproduce the derived results.

\section*{Code availability}
The AFM implementation, causal-compatibility and follow-up experiment code, configuration files, analysis scripts and reproduction instructions are publicly available at \url{https://github.com/hosseinjavidnia/afm/}.

\bibliographystyle{tmlr}
\bibliography{afm_references}

\end{document}